%% file: paper.tex
\documentclass{article}

\usepackage{arxiv}

\usepackage[utf8]{inputenc} 
\usepackage[T1]{fontenc}    
\usepackage{hyperref}       
\usepackage{url}            
\usepackage{booktabs}       
\usepackage{amsfonts}       
\usepackage{nicefrac}       
\usepackage{microtype}      
\usepackage{lipsum}		
\usepackage{graphicx}
\usepackage{natbib}
\usepackage{doi}

\renewcommand{\shorttitle}{Distributionally Robust Model-based Reinforcement Learning with Large State Spaces}

\hypersetup{
pdftitle={Distributionally Robust Model-based Reinforcement
Learning with Large State Spaces},
pdfsubject={q-bio.NC, q-bio.QM},
pdfauthor={Shyam Sundhar Ramesh, Pier Giuseppe Sessa, Yifan Hu, Andreas Krause, Ilija Bogunovic },
pdfkeywords={Reinforcement Learning, Distributional-robustness, Generative Model, Sample Complexity, Gaussian Process},
}
\title{Distributionally Robust Model-based Reinforcement Learning with Large State Spaces }

\author{%
Shyam~Sundhar~Ramesh \\ 
  University College London\\
  \texttt{shyam.ramesh.22@ucl.ac.uk} \\
   \And
   Pier~Giuseppe~Sessa  \\
   ETH Zurich \\
   \texttt{sessap@ethz.ch} \\
    \And
   Yifan~Hu  \\
   EPFL \\
   \texttt{ yifan.hu@epfl.ch} \\
   \AND
   Andreas~Krause \\
   ETH Zurich \\
   \texttt{krausea@ethz.ch} \\
   \And
   Ilija~Bogunovic \\
   University College London \\
   \texttt{i.bogunovic@ucl.ac.uk} \\
}

\usepackage[]{natbib}

\usepackage[utf8]{inputenc} 
\usepackage[T1]{fontenc}    
\usepackage{hyperref}       
\usepackage{url}            
\usepackage{booktabs}       
\usepackage{amsfonts}       
\usepackage{nicefrac}       
\usepackage{microtype}      
\usepackage{xcolor}         
\usepackage{thm-restate}

\usepackage{enumitem} 

\usepackage{graphicx}

\usepackage{booktabs} 
\usepackage{tikz}
\usepackage{tikz-qtree}
\usetikzlibrary{fit}
\usepackage{amssymb}
\usepackage{algorithmic}
\usepackage{algorithm}
\usepackage{mathtools}
\usepackage{amsmath}
\usepackage{amsthm}
\usepackage{subcaption}
\usepackage{xcolor}
\setcitestyle{numbers,close={]},open={[},citesep={,}}

\usepackage{hyperref}
\usepackage{cleveref}
\usepackage{lipsum}
\usepackage{wrapfig}

\usepackage{float}
\usepackage{enumitem} 

\usepackage{algorithm}
\usepackage{algorithmic}
\usepackage{tikz}
\usetikzlibrary{shapes.geometric, arrows}
\usepackage{tikz-qtree}
\usetikzlibrary{fit}
\usepackage{amssymb}
\usepackage{amsfonts}
\usepackage{amsmath}
\usepackage{latexsym}
\usepackage{epsfig}
\usepackage{amsmath}
\usepackage{bbm}
\usepackage[utf8]{inputenc}
\usepackage[english]{babel}

\usepackage{mathtools}
\usepackage{anyfontsize}
\usepackage{mathabx,graphicx}
\usepackage{hyperref}
\usepackage{cleveref}
\usepackage{lipsum}
\usepackage{subcaption}

\DeclareMathOperator*{\argmax}{arg\,max}
\DeclareMathOperator*{\argmin}{arg\,min}

\newcommand{\E}{\mathbb{E}}
\newcommand{\pr}{\mathbb{P}}

\newcommand{\A}{\mathcal{A}}

\newcommand{\rp}{\mathbb{R}}

\newcommand{\mcX}{\mathcal{X}}
\newcommand{\mcXbar}{\overline{\mathcal{X}}}

\newcommand{\V}{\mathcal{V}}

\newcommand{\cS}{\mathcal{S}}

\newcommand{\psa}{P_{f}(s,a)}
\newcommand{\pnsa}{P_{\hat{f}_{n}}(s,a)}

\newcommand{\hpis}{\hat{\pi}_{n}}

\newcommand{\algus}{\textrm{Maximum Variance Reduction (\textsc{MVR}) for learning model dynamics}}

\newcommand{\D}{\mathcal{D}}

\newtheorem{theorem}{Theorem}

\newtheorem{lemma}{Lemma}
\newtheorem{assumption}{Assumption}
\newtheorem{proposition}{Proposition}

\begin{document}
\maketitle
\begin{abstract}
\input{main/00_abstract.tex}

\end{abstract}


\input{main/01_introduction.tex}
\input{main/02_problem_setting.tex}

\input{main/03_algorithm.tex}

\input{main/04_sample_complexity}

\input{main/05_experimental_setup.tex}

\input{main/06_conclusions.tex}
\input{main/07_acknowledgements.tex}

\bibliographystyle{unsrtnat}
\bibliography{ref}


\newpage
\appendix

\input{supplementary/07_MVR_guarantees.tex}
\input{supplementary/08_MVR_guarantees.tex}
\input{supplementary/09_KL-proof.tex}
\input{supplementary/10_diff-func.tex}

\input{supplementary/11_diff-opt-kl}

\input{supplementary/12_error-bound-kl}

\input{supplementary/13_limit-bound-kl}
\input{supplementary/14_cover-kl-proof-no-pol}

\input{supplementary/15_other_uncertainty}
\input{supplementary/16_diff-opt-ch}
\input{supplementary/17_cover-ch-proof}

\input{supplementary/18_function_lipschitz_ch}
\input{supplementary/19_error_bound_ch}

\input{supplementary/tv-formulation-proof}
\input{supplementary/diff-opt-tv}
\input{supplementary/error-bound-tv}
\input{supplementary/cover_tv_proof}
\input{supplementary/function_lipschitz_tv}
\input{supplementary/20_additional_experiments}

\end{document}

%% file: main/00_abstract.tex
Three major challenges in reinforcement learning are the complex dynamical systems with large state spaces, the costly data acquisition processes, and the deviation of real-world dynamics from the training environment deployment. To overcome these issues, we study distributionally robust Markov decision processes with continuous state spaces under the widely used Kullback–Leibler, chi-square, and total variation uncertainty sets. We propose a model-based approach that utilizes Gaussian Processes and the maximum variance reduction algorithm to efficiently learn multi-output nominal transition dynamics, leveraging access to a generative model (i.e., simulator). We further demonstrate the statistical sample complexity of the proposed method for different uncertainty sets. These complexity bounds are independent of the number of states and extend beyond linear dynamics, ensuring the effectiveness of our approach in identifying near-optimal distributionally-robust policies. 
The proposed method can be further combined with other model-free distributionally robust reinforcement learning methods to obtain a near-optimal robust policy.  Experimental results demonstrate the robustness of our algorithm to distributional shifts and its superior performance in terms of the number of samples needed.\looseness=-1

%% file: main/01_introduction.tex
\section{Introduction}

The use of reinforcement learning (RL) algorithms is gaining momentum in various complex domains, including robotics, nuclear fusion, and molecular discovery. Data acquisition in such environments can be a challenging and resource-intensive process. Safety considerations may also limit the amount of data that can be collected through interactions with the environment. To address this issue, a commonly adopted approach is to train RL policies using a simulator (generative model) enabling RL agents to learn from a simulated environment.

Dealing with complex applications that involve large state spaces requires data-efficient learning, even when a simulator is available. However, achieving optimal policies using existing approaches often requires a significant amount of training data, making data-efficient learning an ongoing challenge. Additionally, when deploying a policy to a real-world system, it is crucial to ensure its performance remains reliable despite mismatches between the simulator and the real-world system. Such mismatches can arise from approximation errors, time-varying system parameters, or even due to adversarial influence. For example, in self-driving, it is infeasible to precisely model all possible variables, such as road conditions, brightness, and tire pressure, which can all vary over time. The resulting mismatch, known as the 'sim-to-real gap', can diminish the performance or impact the reliability of RL algorithms trained on a simulator model.

In this work, we examine the use of a generative model in \emph{distributionally-robust model-based reinforcement learning}. Our aim is to find a distributionally-robust policy that is near-optimal by actively querying the simulator with a state-action pair selected by the learning algorithm. To achieve this, we introduce the kernelized Maximum Variance Reduction (MVR) algorithm, which identifies a state-action pair with the highest uncertainty according to the model to learn the nominal model dynamics. The algorithm produces a nominal dynamics estimate that is utilized within the robust Markov Decision Process (MDP) framework, where an uncertainty set that includes all models close (according to, e.g., Kullback–Leibler divergence) to the learned one is considered. We provide a thorough characterization of statistical sample complexity rates by utilizing the learned model to generate a near-optimal robust policy.

\subsection{Related Work}
Reinforcement learning with a generative model is introduced in \citet{kearns2002sparse} wherein one assumes access to a simulator that outputs the next state given any state-action pair. \citet{kakade2003sample} elucidate various uses for this generative setting and analyze it in further detail. For the finite MDP case, such a generative setting has been subsequently studied in various works such as \citep{kakade2003sample,gheshlaghi2013minimax,li2020breaking} and, recently, by~\citet{agarwal2020model} who provide minimax optimality guarantees for the naive plug-in estimator based algorithm. For large state spaces, generative RL is typically combined with function approximation as studied, e.g., by~\citep{abbasi2019politex,shariff2020efficient,lattimore2020learning,li2022near}. Recently, \citet{mehta2021experimental} consider generative RL in continuous state-action spaces from an experimental perspective and showcase the relevance of this setting to the nuclear fusion dynamics research. In addition, \citet{li2022near} present an active exploration strategy that utilizes the least-squares value iteration. Their approach aims to identify a near-optimal policy across the entire state space, providing polynomial sample complexity guarantees that remain unaffected by the number of states. In contrast to these works, we use generative RL to discover \emph{distributionally robust} policies through the modeling of unknown transition dynamics.

In model-based reinforcement learning, the model learned from a simulator encounters two issues well discussed in the literature, namely, the model-bias (\citep{deisenroth2019pilco,clavera2018model}) and the simulation to reality (sim2real) gap (\citep{andrychowicz2020learning,peng2018sim,mankowitz2019robust, christiano2016transfer, rastogi2018sample, wulfmeier2017mutual}). To address this from the perspective of distributional robustness, previous works (\citep{zhou2021finite,panaganti2022sample,yang2021towards}) have considered distributional robustness aspects in the context of finite Markov decision processes (MDPs) using the robust MDP framework from \citet{iyengar2005robust,nilim2005robust}. Various other works utilize this robust MDP framework such as~\citep{xu2010distributionally,wiesemann2013robust,yu2015distributionally,mannor2016robust,badrinath2021robust,petrik2019beyond} for the planning problem, and provide asymptotic guarantees for tabular and linear function approximators~\citep{lim2013reinforcement,tamar2014scaling,roy2017reinforcement,wang2021online}.  Our work is closely related to the recent works on distributionally robust RL  \citep{zhou2021finite,panaganti2022sample,yang2021towards}. However, unlike ours, the sample complexity bounds established in these works rely on the number of states and actions, making them impractical for large or infinite state spaces.  In the model-free setting, distributionally robust RL with large state space (though, still assumed to be finite) was considered by~\citet{panaganti2022robust} in a function approximation setup.  They assume access to offline data from the nominal transition dynamics and provide computational sample complexity bounds in terms of the size of the hypothesis space that is used to represent the set of state-action value functions (Q-function). Other works such as \citep{pinto2017robust,derman2020bayesian,mankowitz2019robust,zhang2020robust} consider robustness aspects in deep reinforcement learning, but these approaches lack theoretical guarantees.
To the best of our knowledge, our work is the first one to address the distributionally robust RL problem in the generative model setting with a \emph{model-based} approach and \emph{large} state spaces. Moreover, we are the first to consider general \emph{non-linear} transition dynamics and derive provable sample complexity guarantees for such a setting.
\looseness=-1
 
Similar to previous works, we utilize the kernelized MDP framework from \citet{chowdhury2019online} to model transition dynamics with continuous states and actions by assuming that the transition function belongs to an associated Reproducing Kernel Hilbert Space (RKHS). Such continuous MDP formulations also appear in \citep{curi2020efficient,curi2021combining}, however, these works consider finite horizon MDPs while in our work we consider infinite horizon discounted MDPs. In particular, \citet{curi2021combining} propose an adversarially robust upper-confidence algorithm to optimize performance in the worst case. However, their algorithm provides robustness guarantees against adversarial perturbations to the transition dynamics. Our work differs from this perspective as we consider robustness w.r.t.~distributional shifts of the transition dynamics.

Finally, in the related kernelized \emph{bandit} setting, model-based distributionally robust algorithms are proposed in \cite{kirschner2020distributionally, bogunovic2018adversarially, nguyen2020distributionally}.  
\looseness=-1

\subsection{Main Contributions} 
We formalize a distributionally robust reinforcement learning setting with continuous state spaces and non-linear transition dynamics in \Cref{sec: Problem Formulation}. In the generative model setting, we propose (in \Cref{sec: prelim}) a model-based approach that utilizes Gaussian Process models and the Maximum Variance Reduction (\textsc{MVR}) principle to efficiently learn transition dynamics. We provide novel statistical sample complexity guarantees in \Cref{sec:sample_complexity} for the proposed method and widely used uncertainty sets. Our sample complexity bounds are independent of the number of states, ensuring the effectiveness of our approach in identifying near-optimal distributionally-robust policies for large state spaces. In \Cref{sec: experiments}, our experimental findings showcase the sample efficiency and robustness of our algorithm in the face of distributional shifts within popular RL-testing environments.\looseness=-1

%% file: main/02_problem_setting.tex
\section{Problem Setting}
\label{sec: Problem Formulation}

A discounted Markov Decision Process (MDP) is a tuple $(\cS,\A, P, r,\gamma)$, with $\cS$ denoting the state space, the action space $\A $, and the probabilistic transition dynamics $P:\cS\times \A\to \Delta(\cS)$. Here, $\Delta(\cS)$ denotes the set of all probability distributions over $\cS$. 
The reward function $r:\cS \times \A \to [0,1]$ characterizes the reward $r(s,a)$ the learner receives upon playing $a\in\A$ in $s\in\cS$, and $ \gamma\in [0,1]$ denotes the discount factor. 
The learner uses a policy $\pi: \cS\to \Delta(\A)$  to select $a\in\A$ upon observing the state $s\in\cS$.

We define the cumulative discounted reward as $\sum_{t=0}^{\infty}\gamma^{t}r(s_{t},a_{t})$ for known initial state $s_{0}$ and $s_{t}\sim P(s_{t-1},a_{t-1})$ for $t>0$ and $a_{t}\sim \pi(s_{t})$.
The value function $V_{\pi}$ and the state-action value function $ Q_{\pi}$ are given as follows:
\begin{equation}
    V_{\pi}(s)=\E_{P,\pi}\Big[\sum_{t=0}^{\infty}\gamma^{t}r(s_{t},a_{t})\Big\vert s_{0}=s\Big], \quad
    Q_{\pi}(s,a)=r(s,a)+\E_{P,\pi}\Big[\sum_{t=1}^{\infty}\gamma^{t}r(s_{t},a_{t})\Big], \nonumber
\end{equation}
where $a_{t}\sim \pi(s_{t})$ and $s_{t+1}\sim P(s_{t},a_{t})$. 
Finally, we define the optimal policy $\pi^{*}$ corresponding to dynamics $P$ which yields the optimal value function, i.e., 
$V_{\pi^{*}}(s) =\max_{\pi} V_{\pi}(s)$ for all $s\in \cS$.

We assume the standard generative (or random) access model, in which the learner can query transition data arbitrarily from a simulator, i.e., each query to the simulator $(s_t,a_t)$ outputs a sample  $s_{t+1} \in \mathbb{R}^d$ where $s_{t+1}\sim P(s_t,a_t)$. 
In particular, we consider the following frequently used transition dynamics model:
\begin{equation}\label{eq: transition dynamics}
    s_{t+1}=f(s_{t},a_{t})+\omega_{t},
\end{equation}
where $\omega_{t}\in \mathbb{R}^{d}$ represents independent additive transition noise and follows a Gaussian distribution with zero mean and covariance $\sigma^{2}I$. 

\textbf{Regularity assumptions:} We assume that $f$ is \emph{unknown} and continuous for tractability reasons which is a common assumption when dealing with continuous state spaces (e.g., \citep{chowdhury2019online,curi2020efficient, kakade2020information}). Further on, we assume that $f$ resides in the Reproducing Kernel Hilbert Space (RKHS). Considering the multi-output definition of $f$ and in line with the previous work (e.g., \cite{chowdhury2019online, curi2020efficient}), we define the modified state-action space $\overline{\mathcal{X}}$ (over which the RKHS is defined) as $\overline{\mathcal{X}}:=\cS\times\A \times [d]$, where the last dimension $i\in  \{1,2,\dots,d\}$ incorporates the index of the output state vector, i.e., $f(\cdot,\cdot)=(\tilde{f}(\cdot,\cdot,1),\dots, \tilde{f}(\cdot,\cdot,d))$ where $\tilde{f}:\overline{\mathcal{X}}\to \mathbb{R}$. In particular, we assume that $\tilde{f}$ belongs to a space of well-behaved functions, denoted by $\mathcal{H}$, induced by some continuous, positive definite kernel function $k:\mcXbar\times\mcXbar\to\mathbb{R}$ and equipped with an inner product $\langle \cdot,\cdot \rangle_{k}$. All functions belonging to an RKHS $\mathcal{H}$ satisfy the reproducing property defined w.r.t.~the inner product $\langle \cdot,\cdot \rangle_{k}: \langle \tilde{f},k(x,\cdot)\rangle=\tilde{f}(x)$ for $\tilde{f}\in \mathcal{H}$. We also make the following common assumptions:  (i) the kernel function $k$ is bounded $k(x,x')\leq 1$ for all $x,x'\in \mcXbar$ and $\mcXbar$ is a compact set ($\mcX\subset\mathbb{R}^{p}$), and (ii) every function $\tilde{f}\in \mathcal{H} $ has a bounded RKHS norm (induced by the inner product) $\|\tilde{f}\|_{k}\leq B$.\looseness=-1

We refer to the simulator environment determined by $f$ as the \emph{nominal model} $P_{f}$, while the true environment encountered by the agent in the real world might not be the same (e.g., due to a sim-to-real gap). Consequently, we utilize the robust MDP framework to tackle this by considering an uncertainty set comprising of all models close to the nominal one. 

\textbf{Robust Markov Decision Process (RMDP)}:  We consider the robust MDP setting that addresses the uncertainty in transition dynamics and considers a set of transition models called the \emph{uncertainty set}. We use $\mathcal{P}^f$ to denote the uncertainty set that satisfies the $(s,a)$--rectangularity condition \cite{iyengar2005robust} (as defined in \Cref{eq: uncertainty set}), an assumption that is commonly used in the related literature (e.g.,~\citep{panaganti2022sample,panaganti2022robust,zhou2021finite}).
Similar to MDPs, we specify RMDP by a tuple $(\cS,\A, \mathcal{P}^f, r,\gamma)$ where the uncertainty set $\mathcal{P}^{f}$ consists of all models close to a nominal model $P_{f}$ in terms of a distance measure $D$: 
\begin{equation}\label{eq: uncertainty set}
        \mathcal{P}^{f}_{s,a}=\{p\in \Delta(\cS):D(p||P_{f}(s,a))\leq \rho\}, \quad \text{and} \quad  \mathcal{P}^{f}=\bigotimes\limits_{(s,a)\in \cS\times\A}\mathcal{P}^{f}_{s,a}. 
\end{equation}

Here, $D$ denotes some distance measure between probability distributions, and $\rho>0$ defines the radius of the uncertainty set. In this work, we consider three probability distance measures, including 
Kullback–Leibler (KL) divergence such that
$
    \mathrm{KL}(P||Q)=\int \log(\frac{dP}{dQ})dP 
$, Chi-Square ($\chi^2$) distance such that
$
\mathrm{\chi^2}(P||Q) = \int (\frac{dP}{dQ}-1)^2dP 
$  for $P$ being absolutely continuous with respect to $Q$, and Total Variation (TV) distance such that
$
\mathrm{TV}(P||Q) = \tfrac{1}{2}\|P-Q\|_{1}
$.\looseness=-1

In the RMDP setting, the goal is to discover a policy that maximizes the cumulative discounted reward for the worst-case transition model within the \emph{given} uncertainty set. Concretely, the robust value function $V_{\pi,f}^{R}$ corresponding to a policy $\pi$ and the optimal  robust value are given as follows:
\begin{equation}\label{eq: robust objective}
V^{\text{R}}_{\pi,f}(s)=\inf_{P\in \mathcal{P}^{f}}\E_{P,\pi}\Big[\sum_{t=1}^{\infty}\gamma^{t}r(s_{t},a_{t})\Big\vert s_{0}=s\Big], \quad 
V^{\text{R}}_{\pi^{*},f}(s) = \max_{\pi} V^{\text{R}}_{\pi,f}(s) \quad \forall s \in \cS.
\end{equation}
In fact, \citet{iyengar2005robust} shows that for any $f$, there exists a deterministic robust policy $\pi^{*}$.
Using the definition of the robust value function, we also define the robust Bellman operator (\cite{iyengar2005robust}) in terms of the robust state-action value function $Q^{R}_{\pi,f}$ as follows: 
\begin{equation}\label{eq: robust bellmann q}
Q^{\text{R}}_{\pi,f}(s,a)=r(s,a)+\gamma\inf_{D(p||P_{f}(s,a))\leq \rho}\E_{s'\sim p}\Big[V^{\text{R}}_{\pi,f}(s')\Big].
\end{equation}

The goal of the learner is to discover a near-optimal robust policy while minimizing the total number of samples $N$, i.e., queries to the nominal model (simulator). Concretely, for a fixed precision $\epsilon > 0$, the goal is to output a policy $\hat{\pi}_{N}$ after collecting $N$ samples, such that $\|V^{\text{R}}_{\hat{\pi}_{N},f}-V^{\text{R}}_{\pi^{*},f}\|_{\infty}\leq \epsilon$.\looseness=-1

%% file: main/03_algorithm.tex
\section{Sampling Algorithm}\label{sec: prelim}

In this section, we outline our methodology for addressing the problem described in \Cref{sec: Problem Formulation}. We begin by discussing the Gaussian process (GP) model often used in algorithms to learn RKHS functions \cite{rasmussen2006gaussian, kanagawa2018gaussian}.\looseness=-1

\textbf{Multi-output Gaussian process:} 
Under the assumptions of~\Cref{sec: Problem Formulation}, modeling uncertainty and learning the transition model $f$ can be performed via the Gaussian process framework.
A Gaussian process $GP(\mu(\cdot),k(\cdot,\cdot))$ over the input domain $\overline{\mathcal{X}}$, is a collection of random variables $(\tilde{f}(x))_{x\in \mcXbar}$ whose every finite subset $(\tilde{f}(x_i))_{i=1}^{n}, n\in \mathbb{N}$, follows multivariate Gaussian distribution with mean $\E[\tilde{f}(x_i)]=\mu(x_i)$ and covariance $\E[(\tilde{f}(x_i)-\mu(x_i))(\tilde{f}(x_j)-\mu(x_j))]=k(x_i,x_j)$ for every $1\leq i,j \leq n$. Standard algorithms implicitly use a zero-mean $GP(0,k(\cdot,\cdot))$ as the prior distribution over $\tilde{f}$, i.e, $\tilde{f}\sim GP(0,k(\cdot,\cdot))$, and assume that the noise variables are drawn independently across $t$ from $\mathcal{N}(0,\lambda)$ with $\lambda > 0$. Considering the multi-output definition of $f(\cdot,\cdot)=(\tilde{f}(\cdot,\cdot,1),\dots, \tilde{f}(\cdot,\cdot,d))$,
we build $d$ copies of the dataset such that $\mathcal{D}_{1:n,l}=\{(s_{i},a_{i},l),s_{i+1,l}\}_{i=1}^{n}$ each with $n$ transitions from a particular state-action pair $(s,a)$ to component $l$ of next state. For $x_i=(s_i,a_i)$ and $y_{i,l}=s_{i+1,l}$,  the posterior mean, covariance and variance for $\tilde{f}(x,l)$ are given by:
\begin{align}\label{eq: posterior mean-multi}
    \mu_{nd}(x,l)&=k_{nd}(x,l)(K_{nd}+I_{nd}\lambda)^{-1}y_{nd},\\
    k_{nd}((x,l),(x',l))&=k((x,l),(x',l))-k_{nd}(x,l)(K_{nd}+I_{nd}\lambda)^{-1}k_{nd}^{T}(x',l'),\nonumber\\ \label{eq: posterior variance-multi}
    \sigma_{nd}^{2}(x,l)&=k_{nd}((x,l),(x,l)).
\end{align}
Here $K_{nd}$ denotes the covariance matrix of dimensions $nd\times nd$ whose entries are $k((x_i,l),(x_j,l'))$ with $1\leq i,j \leq n$ and $1\leq l,l'\leq d$. $k_{nd}(x,l)=[k((x,l),(x_i,l'))]_{1\leq i\leq n,1\leq l'\leq  d}$ denotes the covariance vector and $y_{nd}=[y_{i,l}]_{1\leq i\leq n,1\leq l\leq  d}$ denotes the output vector. 
 Correspondingly, the posterior mean and variance for $f$ would be 
\begin{align}
    \mu_n(s,a)&=(\mu_{nd}(s,a,1),\cdots,\mu_{nd}(s,a,d)),\\
     \sigma_n(s,a)&=(\sigma_{nd}(s,a,1),\cdots,\sigma_{nd}(s,a,d)).
\end{align}

\textbf{Maximum variance reduction:} With certain assumptions on the loss function (squared loss) and noise distribution, the function estimation in RKHS is analogous to the Bayesian Gaussian process framework \cite{rasmussen2006gaussian}. When used with the same kernel function, this allows the construction of mean and variance estimates of $\tilde{f}\in\mathcal{H}$ using Gaussian processes (\cref{eq: posterior mean-multi} and \cref{eq: posterior variance-multi}). Based on these, one can construct shrinking statistical confidence bounds (used in our analysis in~\Cref{sec:mo_conf}) that hold with probability at least $1-\delta$, i.e., the following holds $|\tilde{f}(x)-\mu_{n-1}(x)|\leq \beta_n(\delta)\sigma_{n-1}(x)$ for every $n\geq 1$ and $x\in \mcXbar$. Here $\{\beta_i\}_{i=1}^n$ stands for the sequence of parameters that are suitably set (see \Cref{lemma: uniform vakili bound}) to ensure the validity of the confidence bounds.\looseness=-1

We use the maximum variance reduction (\textsc{MVR}) algorithm (\Cref{alg: mgpbo}) to learn about the nominal model $f$. MVR works on the principle of reducing the maximum uncertainty measured by the posterior standard deviation of a GP model calculated by using previously collected data. At each iteration, MVR queries a state-action pair that has the highest uncertainty according to the model and uses the obtained observation to update the GP posterior. The algorithm outputs nominal dynamics estimate $\hat{f}_n$ corresponding to the final GP posterior mean $\mu_n$.

\begin{algorithm}[t!]
    \caption{$\algus$}
    \begin{algorithmic}[1]
        \STATE \textbf{Require:} Simulator $f$, kernel $k$, domain $\cS\times\A$.
        \STATE  Set $\mu_0(s,a) = 0$, $\sigma_0(s,a) = 1$ for all $(s,a) \in \cS\times\A$
        \FOR{$ i=1,\dots, n$}
            \STATE  $ (s_{i},a_{i})=\argmax_{(s,a)\in \cS\times \A}\|\sigma_{i-1}(s,a)\|_{2}$
            \STATE Observe $s_{i+1}=f(s_{i},a_{i})+\omega_{i}$ \hspace{2mm} (i.e., sample $s_{i+1}$ from nominal $P_f(s_i, a_i)$)
            \STATE Update to $\mu_{i}$ and $\sigma_{i}$ by using $(s_i, a_i,s_{i+1})$ according to \cref{eq: posterior mean-multi} and \cref{eq: posterior variance-multi}
        \ENDFOR
         \RETURN Output the dynamics estimate $\hat{f}_n(\cdot,\cdot) = \mu_{n}(\cdot, \cdot)$ 
    \end{algorithmic}
    \label{alg: mgpbo}
    
\end{algorithm}

To characterize the precision of the learned model, we use the maximum information gain (\cite{srinivas2009gaussian}) 
\begin{equation}\label{eq: max info for gaus}
    \Gamma_n(\mcXbar) = \max_{x_1,\dots,x_n \in \mcXbar}0.5 \log \mathrm{det} (I_n + \lambda^{-1}K_{n}),
\end{equation}
a kernel-dependent quantity that is frequently used in GP optimization.    
For many commonly used kernels, $\Gamma_n$ is sublinear in $n$, which implies that the predictive uncertainties are shrinking sufficiently fast, and thus $\hat{f}_n$ is capable of generalizing well across the entire domain. This is formalized in the following lemma.\looseness=-1

\begin{restatable}{lemma}{modelerror}
    \label{lemma:contribution} 
        For $\beta_{n}(\delta)$ set as in \Cref{lemma:confidence_single_output_discretized} and $\mathcal{I}_{d}$ denoting $\{1,2,\cdots,d\}$, the MVR algorithm (\Cref{alg: mgpbo}) outputs the dynamics estimate $\hat{f}_n(\cdot,\cdot) = \mu_{n}(\cdot, \cdot)$ such that the following holds uniformly for all $(s,a)\in \cS\times \A$ with probability at least $1-\delta$, \looseness=-1
        \begin{equation}
             \|\mu_{n}(s,a)-f(s,a)\|_{2}\leq \mathcal{O}\Big(\frac{\beta_{n}(\delta)2ed}{\sqrt{n}}\sqrt{ \Gamma_{nd}(\cS\times \A\times \mathcal{I}_{d})}\Big).
        \end{equation} 
 \end{restatable}
The preceding lemma asserts that we can effectively estimate the unknown dynamics by utilizing the pure exploration procedure and that the error in the model reduces as we increase the number of samples. In the subsequent section, we leverage this finding to establish the minimum number of samples needed to obtain a distributionally robust policy that is close to optimal.

%% file: main/04_sample_complexity.tex
\section{Sample Complexity}
\label{sec:sample_complexity}
This section discusses the statistical sample complexity of the proposed \textsc{MVR} algorithm in distributionally robust MDPs. Specifically, given the optimal robust policies $\hat{\pi}_{N}$ 
and $\pi^{*}$ corresponding to the learned nominal dynamics $\hat{f}_{N}$ by the \textsc{MVR} algorithm with $N$ iterations and the true nominal dynamics $f$, respectively, we show the number of samples needed by the \textsc{MVR} algorithm to ensure that the following holds:
\begin{equation}\label{eq: Bound Objective}
    |V^{\text{R}}_{\hat{\pi}_N,f}(s)-V^{\text{R}}_{\pi^{*},f}(s)| \leq \epsilon, \forall~s\in\mathcal{S}.
\end{equation}
Note that several model-free methods \citep{panaganti2022robust,derman2018soft,mankowitz2019robust} have studied how to learn an optimal robust policy under KL, TV, and $\chi^2$ uncertainty set given trajectory samples generated from a transition dynamics. Thus, one can easily incorporate the \textsc{MVR} algorithm with these model-free algorithms to find an optimal $\hat \pi_N$ using samples generated by $\hat{f}_N$. One major benefit of our approach is that we do not need access to samples from the more costly simulator $f$ when training the robust policy. Consequently, we focus on the statistical sample complexity of the \textsc{MVR} algorithm rather than designing algorithms to find $\hat \pi_N$.\looseness=-1

\begin{theorem}(Sample Complexity of \textsc{MVR} under KL uncertainty set)\label{thm: kl sample} Consider a robust MDP with nominal transition dynamics $f$ satisfying the regularity assumptions from \Cref{sec: Problem Formulation} and with uncertainty set defined as in \Cref{eq: uncertainty set} w.r.t.~KL divergence. For $\pi^{*}$ denoting the robust optimal policy w.r.t.~nominal transition dynamics $f$ and  $\hat{\pi}_{N}$ denoting the robust optimal policy w.r.t.~learned nominal transition dynamics $\hat{f}_N$ via \textsc{MVR} (\Cref{alg: mgpbo}), and $\delta\in(0, 1)$, $\epsilon\in (0,\frac{1}{1-\gamma})$, it holds that $\max_{s}|V^{\text{R}}_{\hat{\pi}_N,f}(s)-V^{\text{R}}_{\pi^{*},f}(s)|\leq \epsilon$ with probability
at least $1 - \delta$ for any $N$ such that 
\begin{equation}
  N = \mathcal{O}\Big( e^{\frac{2-\gamma}{(1-\gamma)\alpha_\mathrm{kl}}}\frac{ \gamma^{2}\beta_N^{2}(\delta)d^{2}\Gamma_{Nd}}{(1-\gamma)^{4}\rho^{2}\epsilon^{2}}\Big).\\
\end{equation} 
\end{theorem}

\Cref{thm: kl sample} shows the number of samples required from the nominal transition dynamics $f$ (simulator) to construct a robust optimal policy reliably with high probability. 
The complexity bounds depend on the  maximum information gain $\Gamma_{Nd}$ (\Cref{eq: max info for gaus}), which is sublinear in $N$ for many commonly used kernels (\citet{srinivas2009gaussian}). Furthermore, in our analysis, we use the confidence bounds from \cite{vakili2021optimal} with $\beta_N^{2}(\delta)$ which only exhibits a logarithmic dependence on $N$. An additional $d$ factor that denotes the dimension of the state space $\cS$ in the obtained bound comes from utilizing the multi-output (of dimension $d$) GP framework to model the transition dynamics, which also appears in the regret bounds of similar works (\citep{chowdhury2019online,curi2020efficient,curi2021combining}). Finally, the term $\alpha_\mathrm{kl} \in (0,\tfrac{1}{2(1-\gamma)\rho})$ is 
a problem-dependent parameter that is independent of $N$, which similarly appears in the guarantees of~\citep{panaganti2022sample}.
\looseness=-1

We can compare our guarantees with the existing sample-complexity results in model-based distributionally robust RL which, however, only consider finite state-action spaces~\citep{panaganti2022sample,zhou2021finite, yang2021towards}. In particular, when considering KL uncertainty sets, \citet{panaganti2022sample} obtain sample complexity of order $\mathcal{O}\Big( e^{\frac{\alpha_\mathrm{kl}+2}{\alpha_\mathrm{kl}(1-\gamma)}}\frac{ \gamma^{2} |\cS|^2 |\A|}{(1-\gamma)^{4}\rho^{2}\epsilon^{2}}\Big)$ up to logarithmic factors. Notably, the latter complexity bound explicitly depends on the cardinality of the state and action spaces $|\cS|$ and $|\A|$, thus scaling badly when $\cS$ and $\A$ are large or continuous.
Instead, the guarantee of \Cref{thm: kl sample} depends on the state-action space \emph{only} through $\Gamma_{Nd}$ which remains bounded even when these are continuous. This allows us to successfully extend the distributionally robust framework to continuous state spaces. Other terms in the bound of \Cref{thm: kl sample} such as $\gamma$ (the discount factor), $\rho$ (radius of the uncertainty set) have similar dependencies. 
Crucially, the dependence on the precision parameter $\epsilon$ remains the same when compared to the guarantees provided for finite state-action setting.\looseness=-1

We relegate the proof of~\Cref{thm: kl sample} to Appendix~\ref{sec: KL-proof} but outline its main steps below:

\textbf{Step (i):} The first step is to bound the approximation error of policy $\hat{\pi}_{n}$ (i.e., the left-hand side of \Cref{eq: Bound Objective}) by the sum of two error terms: $|V^{\text{R}}_{\hat{\pi}_N,f}(s)-V^{\text{R}}_{\hat{\pi}_N,\hat{f}_{N}}(s)|$ and $|V^{\text{R}}_{\hat{\pi}_N,\hat{f}_{N}}(s)-V^{\text{R}}_{\pi^{*},f}|$. Utilizing the robust Bellman \Cref{eq: robust bellmann q}, bounding such errors boils down to bounding differences of the form: 
\begin{equation}
\label{eq:bounding_P_f_and_P_f_n}
    \max_s\Big|\inf\limits_{\mathrm{KL}(p||P_{f}(s,\hat{\pi}_N(s)))\leq \rho}\E_{s'\sim p}\Big[V^{\text{R}}_{\hat{\pi}_N,f}(s')\Big]-\inf\limits_{\mathrm{KL}(p||P_{\hat{f}_{n}}(s,\hat{\pi}_N(s)))\leq \rho}\E_{s'\sim p}\Big[V^{\text{R}}_{\hat{\pi}_N,f}(s')\Big]\Big|.    
\end{equation}
where $P_{f}(s,a)$ denotes the Gaussian transition distribution with mean $f(s,a)$ and covariance $\sigma^2 I$.

\textbf{Step (ii):}  The major challenge of bounding \Cref{eq:bounding_P_f_and_P_f_n} lies in the inner infinite-dimensional minimization problems over distributions. To overcome this, we can reformulate such problems into single-dimensional ones using duality~\citep{hu2013kullback,zhou2021finite,panaganti2022sample} according to the following lemma.
\begin{lemma}\label{lemma:kl reform}(Variant of \citet{hu2013kullback})
    For random variable $X$ and function $V$ satisfying that $V(X)$ has a finite Moment Generating function, it holds for all $\rho>0$:
\begin{equation}\label{eq: transformation orig}
    \inf_{P:\mathrm{KL}(P||P_{0})\leq \rho}\E_{X\sim P}[V(X)]=\sup_{\alpha\geq 0}\{-\alpha \log(\E_{X\sim P_{0}}[e^{\frac{-V(X)}{\alpha}}])-\alpha \rho\}.
\end{equation}
\end{lemma}

Let $H(V,P):=\sup_{\alpha\geq 0}\{-\alpha \log(\E_{X\sim P}[e^{\frac{-V(X)}{\alpha}}])-\alpha \rho\}$. Thus, applying~\Cref{lemma:kl reform}, we rewrite \Cref{eq:bounding_P_f_and_P_f_n} as the difference of two single-dimensional convex optimization problems with expectations over $P_f$ and $P_{\hat f_N}$, respectively:

\begin{align}
&   \max_s \Big|H(V^{\text{R}}_{\hat{\pi}_N,f},P_{f}(s,\hat{\pi}_N(s))) 
    -    H(V^{\text{R}}_{\hat{\pi}_N,f},P_{\hat{f}_{N}}(s,\hat{\pi}_N(s)))      \Big| \nonumber\\
\quad \leq &     \nonumber\max_{V(\cdot)\in\mathcal{V}}\max_{s,a} \Big|H(V,P_{f}(s,a)) -H(V,P_{\hat{f}_{N}}(s,a))   
    \Big| \nonumber\\
\quad \leq   & \max_{V(\cdot)\in\mathcal{V}}\max_{s,a}\max_{\alpha\in[\underline{\alpha},\overline{\alpha}]}c\left|\E_{s'\sim P_{{f}(s,a)}}[e^{\frac{-V(s')}{\alpha}}]-\E_{s'\sim P_{\hat{f}_{N}(s,a)}}[e^{\frac{-V(s')}{\alpha}}]\right|,\label{eq:step_two}
\end{align}
where $c,~\underline{\alpha}, ~\overline{\alpha}>0$ are constants, $\mathcal{V}$ denotes the value functional space, and the last inequality holds due to certain structural properties of the single-dimensional optimization problem in the RHS of~\Cref{eq: transformation orig}.\looseness=-1

\textbf{Step (iii):} Finally, we bound \Cref{eq:step_two} using the difference between the estimated model $\hat f_N$ and the true $f$, which is characterized by Lemma \ref{lemma:contribution}, in \Cref{sec: KL-proof}. Moreover, to address the outer maximum over all value functions, states, and actions, we incorporate a covering number argument.\looseness=-1 

\textbf{Other uncertainty sets:}~~ We further obtain the statistical sample complexities for $\chi^2$ distance and TV distance uncertainty sets. We note that the analysis follows similar steps as the ones of~\Cref{thm: kl sample}. The major difference lies in incorporating and handling the dual forms of $\chi^2$/TV uncertainty sets in our analysis which differ from the one of \Cref{lemma:kl reform}. For $\chi^2$ uncertainty set, we utilize the dual formulation that appears in  ~\citep{duchi2018learning}, while for TV uncertainty sets we follow the approach of ~\citep{yang2021towards}. 
As before, we can upper bound \Cref{eq:bounding_P_f_and_P_f_n} via covering number arguments and the distance between the nominal transition dynamics $f$ and the learned transition dynamics $\hat f_N$ by using \Cref{lemma:contribution}. Below, we outline the statistical sample complexity in the case of $\chi^2$ and TV uncertainty sets in \Cref{thm: chi sample,thm: tv sample}, respectively.\looseness=-1

\begin{proposition}(Sample Complexity of \textsc{MVR} under $\chi^2$ uncertainty set)\label{thm: chi sample} 
Under the setup of \Cref{thm: kl sample} with uncertainty set defined w.r.t.~$\chi^2$ distance, it holds that $\max_{s}|V^{\text{R}}_{\hat{\pi}_N,f}(s)-V^{\text{R}}_{\pi^{*},f}(s)|\leq \epsilon$ with probability
at least $1 - \delta$ for any $N$ such that 

\begin{equation}
  N = \mathcal{O}\Big(\Big(\frac{1+2\rho}{\sqrt{1+2\rho}-1}\Big)^{4}\frac{ \gamma^{4}\beta_N^2(\delta)d^{2}\Gamma_{Nd}}{(1-\gamma)^{8}\epsilon^{4}}\Big).\\
\end{equation} 
\end{proposition}
\vspace{0.5em}
\begin{proposition}(Sample Complexity of \textsc{MVR} under TV uncertainty set)\label{thm: tv sample} Under the setup of \Cref{thm: kl sample} with uncertainty set defined w.r.t.~TV distance, it holds that $\max_{s}|V^{\text{R}}_{\hat{\pi}_N,f}(s)-V^{\text{R}}_{\pi^{*},f}(s)|\leq \epsilon$ with probability
at least $1 - \delta$ for any $N$ such that 
\begin{equation}
  N = \mathcal{O}\Big(\frac{(2+\rho)^{2}}{\rho^{2}}\frac{ \gamma^{2}\beta^{2}_N(\delta)d^{2}\Gamma_{Nd}}{(1-\gamma)^{4}\epsilon^{2}}\Big).\\\end{equation} 
\end{proposition}

We relegate the proofs of~\Cref{thm: chi sample,thm: tv sample} to \Cref{sec: chi-square uncertainty,sec: tv distance proof}. In comparison to the exponential dependence on $\frac{1}{1-\gamma}$ for KL uncertainty set in \Cref{thm: kl sample}, we note that for both $\chi^2$/TV uncertainty sets, we obtain \emph{polynomial} dependence on $\frac{1}{1-\gamma}$. 
In the context of the TV uncertainty set, the dependency on $\epsilon$ in \Cref{thm: tv sample} remains consistent with the finite state case (\citep{panaganti2022sample}). However, in the $\chi^2$ case, the bound presented in \Cref{thm: chi sample} exhibits a worse dependence on $\epsilon$ compared to the result derived in \citep{panaganti2022sample}. This difference arises because we refrain from utilizing the same dual reformulation lemmas from \citep{iyengar2005robust}, as they are applicable exclusively to finite state-action settings. Improving these rates is an interesting direction for future work.

%% file: main/05_experimental_setup.tex
\vspace{1pt}
\section{Experiments}
\vspace{1pt}
\label{sec: experiments}
The aim of our experiments is to show the effectiveness of the proposed distributionally-robust model-based approach. In particular, our goal is to evaluate the robustness of our policies against different perturbations of the environment's parameters, and compare them with existing non-robust methods. Moreover, we compare our approach with model-free methods (robust and non-robust) which typically require a significantly larger number of interactions with the nominal environment.

\paragraph{Environments:} We consider the OpenAI's gym~\cite{openai_gym} environments of swing-up Pendulum, Cartpole and Reacher, respectively. Pendulum has a 2-dimensional state space and scalar actions (\citep{mehta2021experimental}). For Cartpole, we consider a scalar continuous action space as done in~\citet{mehta2021experimental}, while states are 4-dimensional. Reacher, instead, consists of a 2DOF robot arm with 8-dimensional states. For each environment we test our approach against various perturbations as outlined below.

\textbf{Module 1: Learning the model.} 
To learn the nominal environment, we utilize a setup similar to that of \citet{mehta2021experimental}, but instead of considering the $"\textrm{EIG}_{\tau^*}"$ acquisition function which minimizes entropy of the optimal trajectory $\tau^*$ using model-predictive control, we use the proposed Max Variance Reduction (\textsc{MVR}) method (\Cref{alg: mgpbo}). Similar to \citep{mehta2021experimental}, we use a Gaussian process (GP) prior to model the transition dynamics $f(s,a)$ (alternate models such as Neural Ensembles or Bayesian neural networks can be used to model the transition dynamics as done in, e.g., \citet{curi2020efficient,curi2021combining}).  As in continuous control problems the subsequent states are fairly close, we use our multi-output GP to model the difference $f(s_t,a_t)-s_{t+1}$.

\textbf{Module 2: Computing a robust policy.}
Given a learned model $\hat{f}_n$, we compute the associated robust policy $\hat{\pi}_{n}$ using the Robust Fitted Q-Iteration (\textsc{RFQI}) algorithm recently introduced in \cite{panaganti2022robust} (this effectively approximates our robust optimization oracle). \textsc{RFQI} computes a robust policy from offline data by alternated maximization of a dual-variable function and a Q-function. 

We generate such offline data by using a $\epsilon$-greedy non-robust policy (using Soft Actor-Critic~\cite{haarnoja2018soft} or Model Predictive Control~\cite{camacho2013model,chua2018deep}) which we train \emph{on the learned model $\hat{f}_n$} from Module~1 and let interact with it for $10^6$ steps.
Note that this is crucially different from the vanilla \textsc{RFQI}~\cite{panaganti2022robust} where the true nominal environment was used both for training such policy and for generating offline data. Indeed, this would require a significantly larger number of environment interactions.\looseness=-1

\begin{figure}[t]
    \centering
    \begin{subfigure}[b]{0.32\textwidth}
        \includegraphics[width=\textwidth]{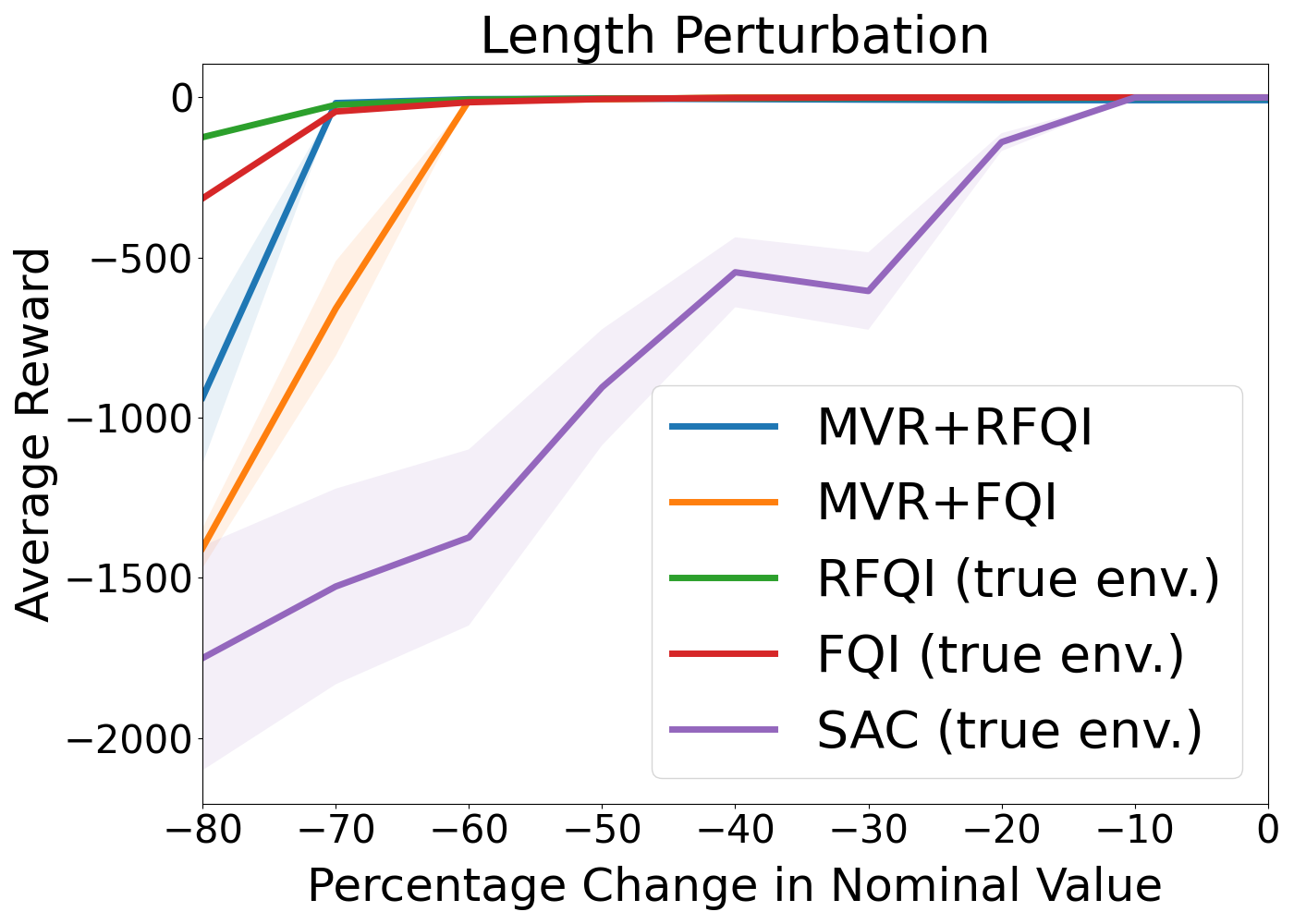}
        \caption{Pendulum}
     \end{subfigure}
     \hspace{1mm}
       \begin{subfigure}[b]{0.32\textwidth}
        \includegraphics[width=\textwidth]{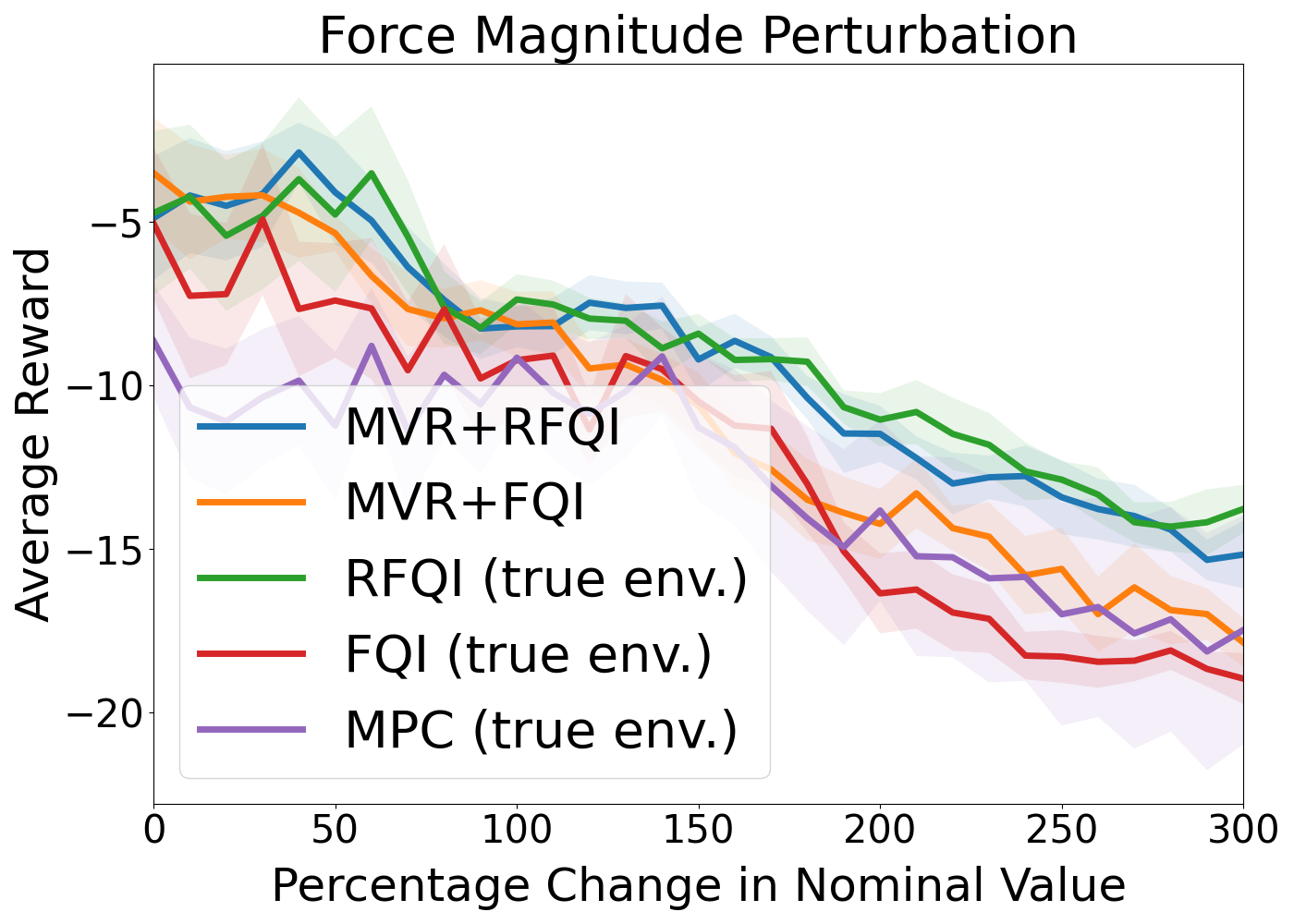}
        \caption{Cartpole}
     \end{subfigure}
     \hspace{1mm}
        \begin{subfigure}[b]{0.32\textwidth}
        \includegraphics[width=\textwidth]{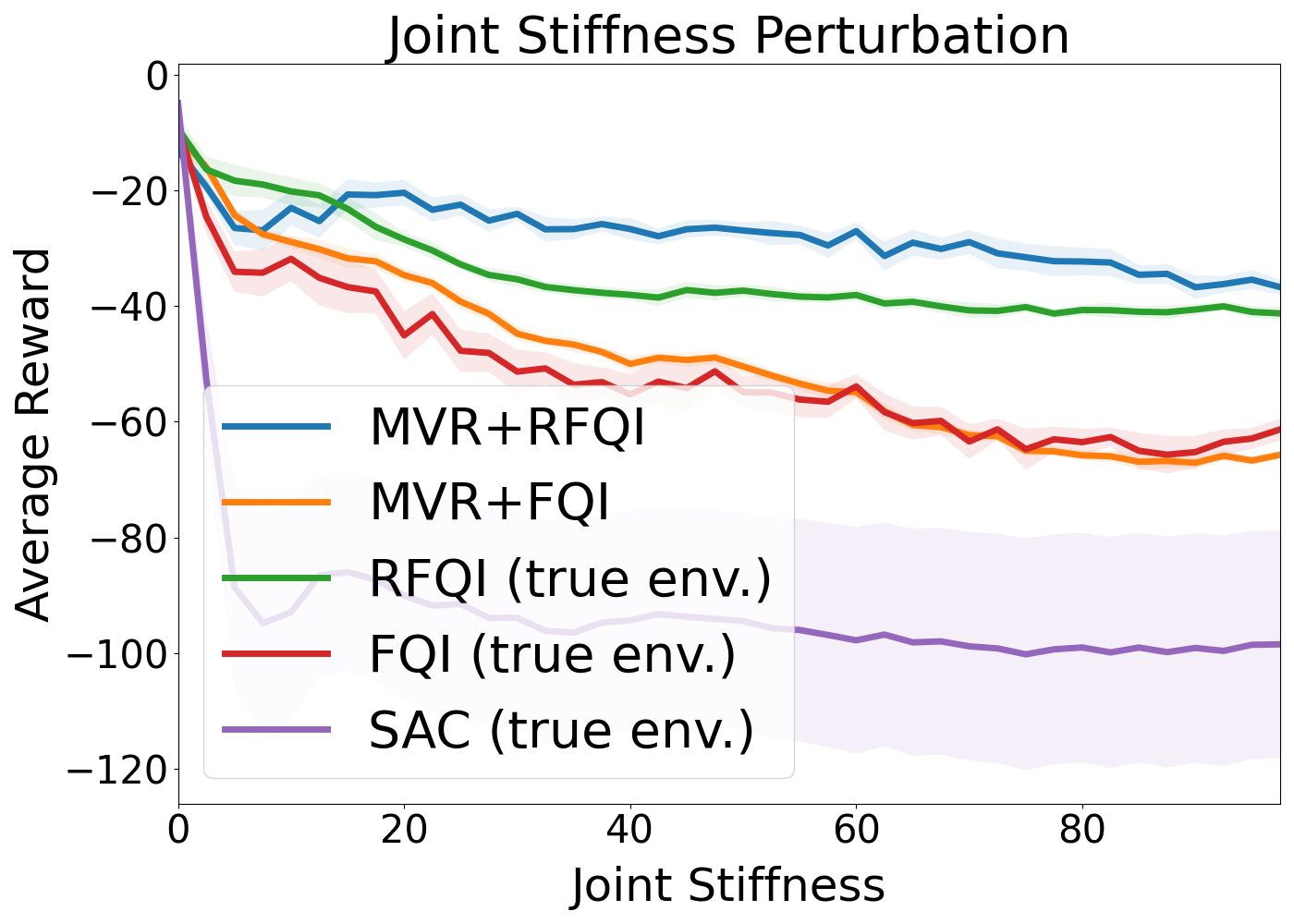}
        \caption{Reacher}
     \end{subfigure}
    \vspace{1mm}
     \caption{Average performance (over 20 episodes) on the considered environments, as a function of different perturbations: length perturbation for Pendulum, force magnitude perturbation for Cartpole, and perturbed joint stiffness for Reacher.
     Unlike our \textsc{MVR+RFQI} and non-robust \textsc{MVR+FQI}, the other baselines are model-free and require access to the true nominal environment for training. 
     The proposed approach \textsc{MVR+RFQI} achieves comparable performance to the model-free \textsc{RFQI} albeit requiring significantly fewer environment interactions (see Table~\ref{tab:num_samples}). Moreover, as the perturbation magnitude increases, \textsc{MVR+RFQI} outperforms the other non-robust baselines.\looseness=-1}
    \label{fig:main_results}
   
\end{figure}

\begin{table}[t]
    \centering
    \begin{tabular}{r c c c c c c}
    \toprule
         &  \textsc{MVR+RFQI} (ours) & \textsc{MVR+FQI} & \textsc{SAC} & \textsc{MPC} & \textsc{RFQI} & \textsc{FQI}   \\
         \midrule
        Pendulum & $60$ & $60$  & $10^4$ & - & $10^6+10^4$ & $10^6+10^4$ \\
         Cartpole & $150$ & $150$ & - & $2250$/step & $10^5\cdot 2250$ & $10^5\cdot 2250 $  \\
        Reacher & $2000$ & $2000$ &   $10^6$  & - & $10^6+10^6$& $10^6+10^6$\\
        \bottomrule\\
    \end{tabular}
    \vspace{1mm}
    \caption{Number of interactions with the nominal environment to obtain the results of Figure~\ref{fig:main_results}. For \textsc{MPC}, a total of $2250$ interactions are required at each step for planning multiple rollouts and selecting the best action. Both \textsc{RFQI} and \textsc{FQI} utilize $10^6$ offline data generated by \textsc{SAC} or \textsc{MPC}.}
    \label{tab:num_samples}
\end{table}

\textbf{Baselines:} 
We compare our approach, which we denote as \textsc{MVR+RFQI}, with the following baselines:\looseness=-1
\begin{itemize}
    \item \textsc{MVR+FQI}: This is a natural non-robust baseline that consists of computing a non-robust policy via the Fitted Q-Iteration (\textsc{FQI}) algorithm~\cite{ernst2005tree} on the same offline data used by \textsc{MVR+RFQI},
    \item Soft Actor-Critic (\textsc{SAC})~\cite{haarnoja2018soft}, or Model Predictive Control (\textsc{MPC})~\cite{camacho2013model,chua2018deep}, as model-free methods which compute non-robust policies \emph{interacting with the nominal environment} (in case of \textsc{MPC}, the latter is used for planning),
    \item \textsc{RFQI}~\cite{panaganti2022robust}, which also requires the nominal environment and uses $10^6$ offline data collected by \textsc{SAC} or \textsc{MPC} to train a robust policy,
    \item \textsc{FQI}~\cite{ernst2005tree}, which trains a non-robust policy from the same data.
\end{itemize}

\paragraph{Training:} Model-free methods are trained directly on the nominal environments. In particular, for Pendulum and Reacher we train \textsc{SAC} until convergence for $10^4$ and $10^6$ steps, respectively. On the continuous actions Cartpole, instead, we run \textsc{MPC} following the implementation of~\cite{pinneri2020sample,mehta2021experimental} which requires a total of $2250$ planning interactions to select the optimal action at each step. Depending on the environment, we utilize \textsc{SAC} or \textsc{MPC} mixed with an $\epsilon$-greedy rule to collect $10^6/10^5$ offline data. These are used to train the offline methods \textsc{RFQI} and \textsc{FQI} as done in~\cite{panaganti2022robust}. 
For the model-based approaches, instead, we first run \textsc{MVR} for a sufficiently informative number of samples ($60$ for Pendulum, $150$ for Cartpole and $2000$ for Reacher) to obtain an estimated model $\hat{f}_n$. Then, we use \textsc{SAC} (trained against model $\hat{f}_n$) or \textsc{MPC} to collect $10^6/10^5$ offline data on such estimated environment. These data are then used to train \textsc{MVR+RFQI} and \textsc{MVR+FQI}.
We provide further implementation details and hyperparameters in~\Cref{app:additional_exps}.

\paragraph{Evaluation:} For each environment, we evaluate the computed policy against different perturbation types and magnitudes. For Cartpole, we perturb the magnitude of the actuation force.
Its nominal value is $10$ and we perturb up to $300\%$. Also, we consider perturbations to gravity in the range of (-100\%,100\%) with the nominal value being $9.82$. For the Pendulum, we consider perturbations to the length of the pendulum and action perturbations (where a random action is chosen with $\epsilon$ probability). Finally, in the case of Reacher we consider perturbations to the joint's stiffness (from 0 to 100) coupled with perturbations of the joint's equilibrium position. Further details on the chosen perturbations and hyperparameters used are provided in~\Cref{app:additional_exps}.

In Figure~\ref{fig:main_results} we plot the average performance (over 20 episodes) of the different baselines subject to different perturbation types and magnitudes for each environment. Results for other perturbations are relegated to~\Cref{app:additional_exps}. Moreover, in Table~\ref{tab:num_samples} we report the total number of interactions with the nominal environment required to compute the evaluated policies. We remark that \textsc{MVR+RFQI} and \textsc{MVR+FQI} interact with the environment only to learn a good Gaussian Process model via the \textsc{MVR} approach. Instead, the other model-free methods utilize the nominal environment throughout the whole training and, in case of \textsc{RFQI} and \textsc{FQI}, even to generate offline data. 
Notably, the policy computed by \textsc{MVR+RFQI} displays comparable performance to its model-free counterpart \textsc{RFQI} which, as shown in Table~\ref{tab:num_samples}, requires a significantly larger number of environment interactions. This illustrates the sample-efficiency of the \textsc{MVR} approach in acquiring informative data and yielding good model estimates. 
Moreover, as the perturbation magnitude increases, \textsc{MVR+RFQI} generally achieves higher performance compared to \textsc{MVR+FQI} and the other non-robust methods, demonstrating the robustness of the computed policies. 
Additionally, as similarly noted e.g. by~\cite{kumar2022should}, we observe the offline methods \textsc{MVR+FQI} and \textsc{FQI} to be generally more robust (although they are not explicitly computing robust policies) than \textsc{SAC} and \textsc{MPC}.

%% file: main/06_conclusions.tex
\section{Conclusions}
We investigated distributionally robust reinforcement learning in the context of continuous state spaces and non-linear transition dynamics. Specifically, we proposed a model-based approach within the generative model setting, utilizing maximum variance reduction to learn nominal transition dynamics effectively. Our results include novel statistical sample complexity guarantees for commonly used uncertainty sets, required for identifying near-optimal distributionally robust policies in large state spaces. Through experiments conducted in popular RL-testing environments, we demonstrated the sample efficiency and robustness of our algorithm in the presence of distributional shifts. An important avenue for future research is the extension of our algorithm to the online and offline reinforcement learning settings.\looseness=-1

%% file: main/07_acknowledgements.tex
\section{Acknowledgements}
PGS was gratefully supported by ELSA (European Lighthouse on Secure and Safe AI) funded by the European Union under grant agreement No. 101070617. YH was supported by NCCR Automation from Switzerland. 
The authors would like to thank Viraj Mehta, Zaiyan Xu, Zhengqing Zhou, Zhengyuan Zhou, Wenhao Yang and Liangyu Zhang for the useful discussion.

%% file: supplementary/07_MVR_guarantees.tex
\allowdisplaybreaks
\section{Theoretical Guarantees of Maximum Variance Reduction (MVR)}

We formally introduce the Gaussian process model in \Cref{sec:MOGPs}. In \Cref{sec:mo_conf}, we describe the confidence bound results from \cite{vakili2021optimal} and adapt them to the case of multi-output GP models. Finally, in \Cref{sec:scg}, we provide sample complexity guarantees for the MVR algorithm. 

We recall the introduced notation $\mathcal{X}=\cS\times\A$ and remark that we use both $(s_i,a_i)$ and $x_i$ interchangeably in this section.

\subsection{Gaussian Process Model}
\label{sec:MOGPs}
Gaussian process (GP) is a non-parametric model that is often used to express uncertainty over functions on any set (e.g., RKHS). They allow to tractably construct posterior distribution over functions in the set to estimate the unknown non-linear function $\tilde{f}:\mcX\to \mathbb{R}$ given data containing samples from function $\tilde{f}$. It follows the Bayesian methodology of calculating posterior given the prior and assumes that the function values at any finite subset of the domain $\mcX$ follow the multivariate Gaussian distribution. One specifies a GP by a prior mean function and a covariance function usually defined using a kernel $k(x,x')$ where $x,x'\in \mcX$.

Assuming that the samples of $\tilde{f}: \mcX \to \mathbb{R}$ are noisy measurements of the underlying true function $\tilde{f}$ with i.i.d.~Gaussian noise $\mathcal{N}(0,\lambda)$, the posterior mean and covariance function of the posterior distribution can be explicitly calculated.
In essence, for $\{x_{1},\dots,x_{N}\}\in \mcX$ and $y_{n}=\tilde{f}(x_{n})+\omega_{n}$, the posterior mean, covariance and variance are given by:
\begin{align}\label{eq: single-output posterior mean}
    \mu_{n}(x)&=k_{n}(x)(K_{n}+I_n\lambda)^{-1}y_{n},\\
    k_{n}(x,x')&=k(x,x')-k_{n}(x)(K_{n}+I_n\lambda)^{-1}k_{n}^{T}(x'),\nonumber\\ \label{eq: single-output posterior variance}
    \sigma_{n}^{2}(x)&=k_{n}(x,x).
\end{align}
Here $K_{n}$ denotes the covariance matrix whose entries are $[K_{n}]_{i,j}=k(x_{i},x_{j})$ with $x_{i},x_{j}\in \{x_{1},\cdots,x_{N}\}$ and $k_{n}(x)=[k(x,x_{1}),\dots,k(x,x_{N})]$ denotes the covariance vector whose entries are the covariance between $x$ and $x_{j}$ for all $x_{j}\in\{x_{1},\dots,x_{N}\}$. The $n\times n$ identity matrix is denoted as $I_n$.\looseness=-1 

We consider multi-output GPs to model the unknown function $f$ that outputs states in $\mathbb{R}^{d}$. (see \Cref{sec: prelim}). Similar to \Cref{eq: single-output posterior mean} and \Cref{eq: single-output posterior variance}, we get analogous expressions for the multi-output case in \Cref{eq: posterior mean-multi} and \Cref{eq: posterior variance-multi}.

\subsection{Non-adaptive Multi-output Confidence Bounds}
\label{sec:mo_conf}
Our \Cref{alg: mgpbo} uses the maximum variance reduction rule to learn about the transition dynamics. As seen in our analysis (see \Cref{thm: kl sample supp}), we are interested in constructing confidence intervals for $f$ only at the end of $n$ iterations (i.e., after taking $n$ samples), and hence, we do not require anytime confidence bounds (e.g., as in \cite{srinivas2009gaussian}). Moreover, in our algorithm, the current decision $(s_i, a_i)$ does not depend on the previous noise realizations. By focusing on the single-output case first, the following confidence lemma from \cite{vakili2021optimal}, can be used to construct confidence intervals with $\beta(\delta)$ independent of $n$ which holds w.h.p. for a fixed $x\in\mcX$:
\begin{lemma}\label{lemma: vakili beta}
Given $n$ noisy observations of $\tilde{f}:\mcX\to \mathbb{R}$ with $\|\tilde{f}\|_{k}\leq B$ where noise $\{\omega_1,\cdots,\omega_n\}$ is independent of inputs $\{x_1,\cdots x_n\}$, for $\beta(\delta)=B+\frac{\sigma}{\lambda}\sqrt{2\log(2/\delta)}$, and $\mu_n$, $\sigma_n$ as defined in \Cref{eq: single-output posterior mean} and \Cref{eq: single-output posterior variance}, the following holds for a fixed $x\in \mcX$ with probability at least $1-\delta$,
\begin{equation}
    |\tilde{f}(x)-\mu_{n}(x)|\leq \beta(\delta)\sigma_{n}(x). \nonumber
\end{equation}
\end{lemma}

To extend this result over the entire input set $x\in \mcX$, the authors in \cite{vakili2021optimal} use a discretization assumption which ensures that there exists a discretization $\D_{n}$ such that $\tilde{f}(x)-\tilde{f}([x]_{n})\leq \frac{1}{\sqrt{n}}$, where $[x]_{n}=\argmin_{x'\in \D_n}\|x-x'\|_2$ and $|\D_n|\leq CB^{d}n^{d/2}$ for $C$ being independent of $n$ and $B$ (RKHS norm bound). Consequently, they obtain the following lemma providing uniform confidence bounds:\looseness=-1 
\begin{lemma}(\cite[Theorem-3]{vakili2021optimal})
\label{lemma:confidence_single_output_discretized}
Given $n$ noisy observations of $\tilde{f}:\mcX \to \mathbb{R}$, $\mcX \subset \mathbb{R}^d$ satisfying $\|\tilde{f}\|_{k}\leq B$ where noise $\{\omega_1,\cdots,\omega_n\}$ is independent of inputs $\{x_1,\cdots x_n\}\subset\mcX$ and when there exists discretization $\D_n$ of $\mcX$ with $|\D_n|\leq CB^{d}n^{d/2}$, for $\beta(\delta)=B+\frac{\sigma}{\lambda}\sqrt{2\log(2/\delta)}$ and $\beta_n(\delta)=2B+\beta(\frac{\delta}{3C(B+\sqrt{n}\beta(2\delta/3n))^{d}n^{d/2}})$, $\mu_n$, $\sigma_n$ as defined in \Cref{eq: single-output posterior mean} and \Cref{eq: single-output posterior variance}, the following holds for all $x\in \D_n$ with probability at least $1-\delta$,
\begin{equation}
    |\tilde{f}(x)-\mu_{n}(x)|\leq \beta_n(\delta)\sigma_{n}(x). \nonumber
\end{equation}
\end{lemma}

To extend this result to multiple dimensions as required in our work, we take the same discretization assumption as in \cite{vakili2021optimal}. But considering the multi-output definition of $f$, we define the modified state-action space $\overline{\mathcal{X}}$. 
This is in line with \citet{chowdhury2019online}, which also has a similar multi-output setting. We define the modified state-action space as $\overline{\mathcal{X}}:=\cS\times\A\times \{1,2,\cdots,d\}$ where the last dimension $i\in  \{1,2,\cdots,d\}$ incorporates the index of the output vector, in the sense that $f(\cdot,\cdot)=(\tilde{f}(\cdot,\cdot,1),\cdots,\tilde{f}(\cdot,\cdot,d))$ where $\tilde{f}:\overline{\mathcal{X}}\to \mathbb{R}$. We then detail the discretization assumption as in \cite{vakili2021optimal} w.r.t.~$\tilde{f}$ (see also \Cref{sec: Problem Formulation} for more details).

\begin{assumption}\label{assump: Discretization}
For every $n\in \mathbb{N}$ and $\tilde{f}\in \mathcal{H}_{k}(\cS\times \A\times \mathcal{I})$ there exists a discretization $\D_{n}(\cS\times\A)$ of $\cS\times \A$ such that $\tilde{f}(s,a,i)-\tilde{f}([s,a]_{n},i)\leq \frac{1}{\sqrt{n}}$, where $[s,a]_{n}=\argmin_{(s',a')\in \D_n(\cS\times \A)}\|(s,a)-(s',a')\|_2$, $i\in \mathcal{I}$, and $|\D_n(\cS\times \A)|\leq CB^{p}n^{p/2}$ ( $|\D_n(\cS\times \A \times \mathcal{I})|\leq CB^{p}n^{p/2}d$) for $C$ being independent of $n$ and $B$, and $\cS\times \A\subset \mathbb{R}^{p}$. 
\end{assumption}
\Cref{assump: Discretization} allows us to provide bounds for $\|f(s,a)-\mu_{n}(s,a)\|_{2}$ for all $(s,a)\in \cS$ using \Cref{lemma:confidence_single_output_discretized}. Note that \Cref{assump: Discretization} does not discretize the modified state-action space  ($\overline{\mathcal{X}}=\cS\times\A\times \{1,2,\cdots,d\}$) but instead discretizes $\cS\times \A$ for each $i\in \mathcal{I}$. Hence, $|\D_n(\cS\times \A \times \mathcal{I})|\leq CB^{p}n^{p/2}d$, and $\beta_n(\delta)$ will change accordingly.
We describe the following lemma detailing the same.

\begin{lemma}\label{lemma: uniform vakili bound}Under \Cref{assump: Discretization} with $\beta_{n}(\delta)$ as in \Cref{lemma:confidence_single_output_discretized} and training a Gaussian process model on observations up to iteration $n$ ($\{s_{1},\cdots,s_{n}\}$) and their corresponding inputs  ($\{(s_{0},a_{0}),\cdots,(s_{n-1},a_{n-1})\}$),  it holds with probability at least $1-\delta$,\looseness=-1
\begin{equation}
    \|f(s,a)-\mu_{n}(s,a)\|_{2}\leq \beta_{n}(\delta)\sqrt{d}\|\sigma_{n}([s,a]_{n})\|_{2}+\frac{2d}{\sqrt{n}}, \nonumber
\end{equation}
 uniformly for all $(s,a)\in \cS\times \A$ and $[s,a]_{n}=\argmin_{(s',a')\in \D_n(\cS\times \A)}\|(s,a)-(s',a')\|_2$.
\end{lemma}

 \begin{proof}
 For any $(s,a)\in \cS\times \A$,
 \begin{align}
     &\|f(s,a)-\mu_{n}(s,a)\|_{2} \nonumber\\&=\sqrt{\sum_{i=1}^{d}(\tilde{f}(s,a,i)-\mu_{n}(s,a,i))^{2}} \\
     &= \sqrt{\sum_{i=1}^{d}|\tilde{f}(s,a,i)-\tilde{f}([s,a]_{n},i)+\tilde{f}([s,a]_{n},i)-\mu_{n}([s,a]_{n},i)+\mu_{n}([s,a]_{n},i)-\mu_{n}(s,a,i)|^{2}} \nonumber\\
     \label{eq: uniform vakili-2}
     &\leq \sum_{i=1}^{d} \Big(|\tilde{f}(s,a,i)-\tilde{f}([s,a]_{n},i)|+|\tilde{f}([s,a]_{n},i)-\mu_{n}([s,a]_{n},i)|+|\mu_{n}([s,a]_{n},i)-\mu_{n}(s,a,i)|\Big) \\
     \label{eq: uniform vakili-3}
     &\leq \Bigg(\sum_{i=1}^{d}(|\tilde{f}([s,a]_{n},i)-\mu_{n}([s,a]_{n},i)|)\Bigg)+\frac{2d}{\sqrt{n}}\\
     \label{eq: uniform vakili-4}
     &\leq \beta_{n}(\delta)\Bigg(\sum_{i=1}^{d}(\sigma_{n}([s,a]_{n},i))\Bigg)+\frac{2d}{\sqrt{n}}\\
     \label{eq: uniform vakili-5}
     &\leq \beta_{n}(\delta)\sqrt{d}\sqrt{\sum_{i=1}^{d}(\sigma_{n}([s,a]_{n},i))^{2}}+\frac{2d}{\sqrt{n}}\\
     \label{eq: uniform vakili-6}
     &\leq \beta_{n}(\delta)\sqrt{d}\|\sigma_{n}([s,a]_{n})\|_{2}+\frac{2d}{\sqrt{n}}.
 \end{align}
 In \Cref{eq: uniform vakili-2}, \Cref{eq: uniform vakili-5} we use $\|x\|_{2}\leq \|x\|_{1}\leq \sqrt{d}\|x\|_{2}$. And \Cref{eq: uniform vakili-3} and \Cref{eq: uniform vakili-4} follow from \Cref{assump: Discretization} (since $\tilde{f},\mu_n\in \mathcal{H}_k(\cS\times\A\times \mathcal{I})$) and \Cref{lemma:confidence_single_output_discretized}, respectively.

 \end{proof}

%% file: supplementary/08_MVR_guarantees.tex
\subsection{Sample Complexity Guarantees}
\label{sec:scg}

Our objective is to obtain a uniform upper bound on the model precision $\|\mu_{n}(s,a)-f(s,a)\|_{2}$ for all state-action pairs (s,a) while accounting for the errors induced by discretization. Here, $\mu_n(\cdot,\cdot)$ is obtained from \Cref{alg: mgpbo}.  We achieve this by using \Cref{lemma: uniform vakili bound} to obtain a bound in terms of maximum information gain (\Cref{eq: max info for gaus}). 

\modelerror*

\begin{proof}
From \Cref{lemma: uniform vakili bound}, it holds that with probability at least $1-\delta$ uniformly for all $(s,a)\in\cS\times\A$:
\begin{align}
    \|\mu_{n}(s,a)-f(s,a)\|_{2}&\leq\beta_{n}(\delta)\sqrt{d}\|\sigma_{n}([s,a]_{n})\|_{2}+\frac{2d}{\sqrt{n}} \nonumber\\
    &\leq \beta_{n}(\delta)\sqrt{d}\max_{(s,a)\in\cS\times\A}\|\sigma_{n}(s,a)\|_{2}+\frac{2d}{\sqrt{n}} \nonumber\\
    &\leq \beta_{n}(\delta)\sqrt{d}\|\sigma_{n}(s_n,a_n)\|_{2}+\frac{2d}{\sqrt{n}} \nonumber\\
    &\leq \frac{2d}{\sqrt{n}} +  \frac{\beta_{n}(\delta)}{n}\sqrt{d}\sum_{j=1}^{n}\|\sigma_{j}(s_n,a_n)\|_{2} \nonumber\\\label{eq:mvr decision rule}
    &\leq \frac{2d}{\sqrt{n}}  + \frac{\beta_{n}(\delta)}{n}\sqrt{d}\sum_{j=1}^{n}\|\sigma_{j}(s_j,a_j)\|_{2}\\
    &\leq \frac{\beta_{n}(\delta)}{\sqrt{n}}\sqrt{d}\sqrt{\sum_{j=1}^{n}\|\sigma_{j}(s_j,a_j)\|^{2}_{2}}+\frac{2d}{\sqrt{n}} \nonumber\\
    \label{eq: pre learning error bound}
    &\leq \frac{\beta_{n}(\delta)2e d}{\sqrt{n}}\sqrt{ \Gamma_{nd}(\cS\times \A\times \mathcal{I}_{d})}+\frac{2d}{\sqrt{n}}\\
    \label{eq: learning error bound}
    &= \mathcal{O}\Big(\frac{\beta_{n}(\delta)2ed}{\sqrt{n}}\sqrt{ \Gamma_{nd}(\cS\times \A\times \mathcal{I}_{d})}\Big).
\end{align}
Here, \Cref{eq:mvr decision rule} follows from the decision rule in line-4 of \Cref{alg: mgpbo} and \Cref{eq: pre learning error bound} is obtained using standard bound for the sum of variances in the case of multi-output GPs from \citet[Lemma-7]{curi2021combining} and \citet[Lemma-11]{chowdhury2019online}. 
\end{proof}

%% file: supplementary/09_KL-proof.tex
\newpage
\section{Sample Complexity Bounds for KL Uncertainty Sets}\label{sec: KL-proof}
\begin{theorem}(Sample Complexity of \textsc{MVR} under KL uncertainty set)\label{thm: kl sample supp} Consider a robust MDP with nominal transition dynamics $f$ satisfying the regularity assumptions from \Cref{sec: Problem Formulation} and with uncertainty set defined as in \Cref{eq: uncertainty set} w.r.t.~KL divergence. For $\pi^{*}$ denoting the robust optimal policy w.r.t.~nominal transition dynamics $f$ and  $\hat{\pi}_{N}$ denoting the robust optimal policy w.r.t.~learned nominal transition dynamics $\hat{f}_N$ via \textsc{MVR} (\Cref{alg: mgpbo}), and $\delta\in(0, 1)$, $\epsilon\in (0,\frac{1}{1-\gamma})$, it holds that $\max_{s}|V^{\text{R}}_{\hat{\pi}_N,f}(s)-V^{\text{R}}_{\pi^{*},f}(s)|\leq \epsilon$ with probability
at least $1 - \delta$ for any $N$ such that 
\begin{equation}
  N = \mathcal{O}\Big( e^{\frac{2-\gamma}{(1-\gamma)\alpha_\mathrm{kl}}}\frac{ \gamma^{2}\beta_N^{2}(\delta)d^{2}\Gamma_{Nd}}{(1-\gamma)^{4}\rho^{2}\epsilon^{2}}\Big).\\
\end{equation} 
\end{theorem}
\begin{proof}
\textbf{Step (i):}
As detailed in the proof outline of \Cref{sec:sample_complexity}, in order to bound $|V^{\text{R}}_{\hpis,f}(s)-V^{\text{R}}_{\pi^* ,f}(s)|$, we begin by adding and subtracting $V^{\text{R}}_{\hpis,\hat{f}_{n}}(s)$ which is the robust value function w.r.t.~the nominal transition dynamics $\hat{f}_{n}$ and its corresponding optimal policy $\hpis$. Then, we split the difference into two terms as follows:
\begin{equation}\label{eq: bound separation}
   |V^{\text{R}}_{\hpis,f}(s)-V^{\text{R}}_{\pi^*,f}(s)|= \underbrace{|V^{\text{R}}_{\hpis,f}(s)-V^{\text{R}}_{\hpis,\hat{f}_{n}}(s)|}_{(i)}+\underbrace{|V^{\text{R}}_{\hpis,\hat{f}_{n}}(s)-V^{\text{R}}_{\pi^{*},f}(s)|}_{(ii)}.
\end{equation}

In order to not disturb the flow of the proof we bound (i) and (ii) separately \Cref{lemma: (i) bound lemma} and \Cref{lemma: (ii) bound lemma} respectively. From \Cref{lemma: (i) bound lemma}, we obtain that 
\begin{align}\label{eq: (i) reusage kl}
    (i)&\leq\max_{s}\Big|V^{\text{R}}_{\hpis,f}(s)-V^{\text{R}}_{\hpis,\hat{f}_{n}}(s)\Big|\nonumber\\&\leq \frac{\gamma}{1-\gamma}\max_{s}\Big|\inf_{\mathrm{KL}(p||P_{f}(s,\hpis(s)))\leq \rho}\E_{s'\sim p}\Big[V^{\text{R}}_{\hpis,f}(s')\Big]-\inf_{\mathrm{KL}(p||P_{\hat{f}_{n}}(s,\hpis(s)))\leq \rho}\E_{s'\sim p}\Big[V^{\text{R}}_{\hpis,f}(s')\Big]\Big|.
\end{align}
And from \Cref{lemma: (ii) bound lemma}, we obtain that 
\begin{align}\label{eq: (ii) reusage kl}
    (ii)&\leq\max_{s}\Big|V^{\text{R}}_{\hpis,\hat{f}_n}(s)-V^{\text{R}}_{\pi^{*},f}(s)\Big|\nonumber \\&\leq \frac{\gamma}{1-\gamma}\max_{s}\Big|\inf_{\mathrm{KL}(p||P_{\hat{f}_{n}}(s,\hpis(s)))\leq \rho}\E_{s'\sim p}\Big[V^{\text{R}}_{\pi^{*},f}(s')\Big]-\inf_{\mathrm{KL}(p||P_{f}(s,\hpis(s)))\leq \rho}\E_{s'\sim p}\Big[V^{\text{R}}_{\pi^{*},f}(s')\Big]\Big|.
\end{align}
Note that both these terms in \Cref{eq: (i) reusage kl,eq: (ii) reusage kl} are of the form mentioned in the \textbf{Step (i)} of \Cref{sec:sample_complexity}.

\textbf{Step (ii):} Next, corresponding to \textbf{Step (ii)} of the proof outline in \Cref{sec:sample_complexity}, we use \Cref{lemma:kl reform} to bound \Cref{eq: (i) reusage kl,eq: (ii) reusage kl}. Denote $M:= \frac{1}{1-\gamma}\geq \max_{s}V^{\text{R}}_{\pi}(s)$ for convenience. 
Using \Cref{eq: (i) reusage kl} and \Cref{lemma: diff-opt} (internally using \Cref{lemma:kl reform}), conditioned on the event of \Cref{lemma: diff-opt} holding true, it holds that

\begin{align}
    (i)&\leq\max_{s}\Big|V^{\text{R}}_{\hpis,f}(s)-V^{\text{R}}_{\hpis,\hat{f}_{n}}(s)\Big|\nonumber\\&\leq\frac{1}{1-\gamma}\max_{s}\Big|\gamma\inf_{\mathrm{KL}(p||P_{f}(s,\hpis(s)))\leq \rho}\E_{s'\sim p}\Big[V^{\text{R}}_{\hpis,f}(s')\Big]-\gamma\inf_{\mathrm{KL}(p||P_{\hat{f}_{n}}(s,\hpis(s)))\leq \rho}\E_{s'\sim p}\Big[V^{\text{R}}_{\hpis,f}(s')\Big]\Big|  \nonumber\\
    \label{eq: diff-opt lemma usage}
    &\leq \max_{s,a}\Bigg(2\gamma\tfrac{M^2}{\rho} e^{\frac{M}{\underline{\alpha}}}\max_{\alpha\in[\underline{\alpha},\frac{M}{\rho}]}\Big|\E_{s'\sim P_{\hat{f}_{n}}(s,a)}[e^{\frac{-V^{\text{R}}_{\hpis,f}(s')}{\alpha}}]-\E_{s'\sim P_{f}(s,a)}[e^{\frac{-V^{\text{R}}_{\hpis,f}(s')}{\alpha}}]\Big|\Bigg).\\\label{eq: diff-opt lemma usage overall-1}
    &\leq \max_{V(\cdot)\in\mathcal{V}}\max_{s,a}\Bigg(2\gamma\tfrac{M^2}{\rho} e^{\frac{M}{\underline{\alpha}}}\max_{\alpha\in[\underline{\alpha},\frac{M}{\rho}]}\Big|\E_{s'\sim P_{\hat{f}_{n}}(s,a)}[e^{\frac{-V(s')}{\alpha}}]-\E_{s'\sim P_{f}(s,a)}[e^{\frac{-V(s')}{\alpha}}]\Big|\Bigg).
\end{align}
We can bound (ii) similarly.
\begin{align}
    (ii)&\leq\max_{s}\Big|V^{\text{R}}_{\hpis,\hat{f}_n}(s)-V^{\text{R}}_{\pi^{*},f}(s)\Big|\\\label{eq: diff-opt lemma usage overall-2}
    &\leq \max_{V(\cdot)\in\mathcal{V}}\max_{s,a}\Bigg(2\gamma\tfrac{M^2}{\rho} e^{\frac{M}{\underline{\alpha}}}\max_{\alpha\in[\underline{\alpha},\frac{M}{\rho}]}\Big|\E_{s'\sim P_{\hat{f}_{n}}(s,a)}[e^{\frac{-V(s')}{\alpha}}]-\E_{s'\sim P_{f}(s,a)}[e^{\frac{-V(s')}{\alpha}}]\Big|\Bigg).
\end{align}
\textbf{Step (iii):} Next, we want to utilize the learning error bound (\Cref{eq: learning error bound}) that bounds the difference between the means of true nominal transition dynamics $P_{f}$ and learned nominal transition dynamics $P_{\hat{f}_n}$ to bound \Cref{eq: diff-opt lemma usage overall-1,eq: diff-opt lemma usage overall-2}.

We begin by bounding the difference $\Big|\E_{s'\sim P_{\hat{f}_{n}}(s,a)}[e^{\frac{-V(s')}{\alpha}}]-\E_{s'\sim P_{f}(s,a)}[e^{\frac{-V(s')}{\alpha}}]\Big|$, by the difference in means of $P_{f}$ and $P_{\hat{f}_n}$ in \Cref{lemma: error bound}. Since \Cref{eq: diff-opt lemma usage overall-1} has a $\max$ over all value functions, we introduce a covering number argument in \Cref{lemma: kl-cover-no-pol} to reform it to a max over the functions in the $\zeta-$covering set.  We then use \Cref{lemma: error bound}  to obtain bounds in terms of maximum information gain $\Gamma_{Nd}$ (\Cref{eq: max info for gaus}) and $\zeta$. Further details regarding the covering number argument are deferred to \Cref{lemma: kl-cover-no-pol}. Then, we apply the result of \Cref{lemma: kl-cover-no-pol} with $\zeta=1$ (defined in \Cref{lemma: kl-cover-no-pol}) on \Cref{eq: diff-opt lemma usage overall-1}. Then, it holds that
\begin{align}\label{eq: (i) final bound}
  (i) \leq\max_{s}\Big|V^{\text{R}}_{\hpis,f}(s)-V^{\text{R}}_{\hpis,\hat{f}_{n}}(s)\Big|&
= \mathcal{O}\Bigg(2\tfrac{M^2}{\rho} e^{\frac{M}{\alpha_{kl}}}e^{\frac{1}{\alpha_{kl}}}\frac{ \beta_n(\delta)\sqrt{2ed^{2}\Gamma_{nd}}}{\sigma\sqrt{n}}\Bigg)  ,
\end{align}
where $\alpha_{kl}$ is a problem-dependent constant denoting the minimum value of $\underline{\alpha}$ defined in \Cref{lemma: diff-opt}. A similar constant also appears in the sample complexity bounds provided in \cite{panaganti2022sample,zhou2021finite}. 
Note that $\beta_n$, which appears in \Cref{lemma: vakili beta}, has a logarithmic dependence on $n$. Similarly, from \Cref{eq: diff-opt lemma usage overall-2} and \Cref{lemma: error bound,lemma: kl-cover-no-pol}, we obtain 
\begin{align}\label{eq: (ii) final bound}
    (ii)\leq\max_{s}\Big|V^{\text{R}}_{\hpis,\hat{f}_n}(s)-V^{\text{R}}_{\pi^{*},f}(s)\Big|
    &=\mathcal{O}\Big(2\gamma\tfrac{M^{2}}{\rho} e^{\frac{M}{\alpha_{kl}}}e^{\frac{1}{\alpha_{kl}}}\frac{ \beta_n(\delta)\sqrt{2ed^{2}\Gamma_{nd}}}{\sigma\sqrt{n}}\Big).  
\end{align}
Note that we want to bound $V^{\text{R}}_{\hpis,f}(s)-V^{\text{R}}_{\pi^{*},f}(s)=(i)+(ii)$ over all $s\in \cS$. Using $ \max_{s}\Big|V^{\text{R}}_{\hpis,f}(s)-V^{\text{R}}_{\pi^{*},f}(s)\Big|\leq  \max_{s}\Big|V^{\text{R}}_{\hpis,\hat{f}_n}(s)-V^{\text{R}}_{\pi^{*},f}(s)\Big|+  \max_{s}\Big|V^{\text{R}}_{\hpis,\hat{f}_n}(s)-V^{\text{R}}_{\pi^{*},f}(s)\Big|$ and substituting $M$ by $1/(1-\gamma)$, we obtain from \Cref{eq: (i) final bound} and \Cref{eq: (ii) final bound} 
\begin{align}
    \max_{s}\Big|V^{\text{R}}_{\hpis,f}(s)-V^{\text{R}}_{\pi^{*},f}(s)\Big| 
     &=\mathcal{O}\Big(\gamma e^{\frac{1}{(1-\gamma)\alpha_{kl}}}e^{\frac{1}{\alpha_{kl}}}\frac{ \beta_n(\delta)d\sqrt{2e\Gamma_{nd}}}{(1-\gamma)^{2}\rho\sigma\sqrt{n}}\Big). \nonumber
\end{align}
Finally,  to ensure that $\max_{s}|V^{\text{R}}_{\hpis,f}(s)-V^{\text{R}}_{\pi^{*},f}(s)|\leq \epsilon$ , it suffices to have
\begin{align}
     \max_{s}\Big|V^{\text{R}}_{\hpis,f}(s)-V^{\text{R}}_{\pi^{*},f}(s)\Big| 
     =\mathcal{O}\Big(\gamma e^{\frac{1}{(1-\gamma)\alpha_{kl}}}e^{\frac{1}{\alpha_{kl}}}\frac{ \beta_n(\delta)d\sqrt{2e\Gamma_{nd}}}{(1-\gamma)^{2}\rho\sigma\sqrt{n}}\Big)
     = \epsilon.\nonumber
\end{align}

By inverting the previously obtained result, we arrive at  
\begin{equation}
     n = \mathcal{O}\Big( e^{\frac{2}{(1-\gamma)\alpha_{kl}}}e^{\frac{2}{\alpha_{kl}}}\frac{ \gamma^{2}\beta_n^{2}(\delta)d^{2}\Gamma_{nd}}{(1-\gamma)^{4}\rho^{2}\epsilon^{2}}\Big).  \nonumber
\end{equation}

\end{proof}

\input{supplementary/21_theorem-step-1_lemmas}

%% file: supplementary/21_theorem-step-1_lemmas.tex
\begin{lemma}\label{lemma: (i) bound lemma} (Simplification using robust Bellman equation)
Denote $(i):=\Big|V^{\text{R}}_{\hpis,f}(s)-V^{\text{R}}_{\hpis,\hat{f}_{n}}(s)|$ 
  for  $V^{\text{R}}_{\hpis,f}$ being the robust value function of policy $\hpis$ w.r.t.~true nominal transition dynamics $f$ and $V^{\text{R}}_{\hpis,\hat{f}_{n}}$ being the robust value function of policy $\hpis$ w.r.t.~learned nominal transition dynamics $f$. Then the following holds,
\begin{align}
    (i)&=\Big|V^{\text{R}}_{\hpis,f}(s)-V^{\text{R}}_{\hpis,\hat{f}_{n}}(s)\Big|\nonumber\\
&\leq\max_{s}\Big|V^{\text{R}}_{\hpis,f}(s)-V^{\text{R}}_{\hpis,\hat{f}_{n}}(s)\Big|\nonumber\\&\leq \frac{\gamma}{1-\gamma}\max_{s}\Big|\inf_{\mathrm{D}(p||P_{f}(s,\hpis(s)))\leq \rho}\E_{s'\sim p}\Big[V^{\text{R}}_{\hpis,f}(s')\Big]-\inf_{\mathrm{D}(p||P_{\hat{f}_{n}}(s,\hpis(s)))\leq \rho}\E_{s'\sim p}\Big[V^{\text{R}}_{\hpis,f}(s')\Big]\Big|.
\end{align}    
\end{lemma}
\begin{proof}
    Since both the quantities are w.r.t.~the same policy, using the definition of the robust $Q$-function and the robust Bellman equation (see \Cref{eq: robust bellmann q}), we obtain:
\begin{align}\label{eq: (i) def}
    (i)
    &=|V^{\text{R}}_{\hpis,f}(s)-V^{\text{R}}_{\hpis,\hat{f}_{n}}(s)|\\
    &=|Q^{\text{R}}_{\hpis,f}(s,\hpis(s))-Q^{\text{R}}_{\hpis,\hat{f}_{n}}(s,\hpis(s))|  \nonumber\\
    \begin{split}
    &=|r(s,\hpis(s))-r(s,\hpis(s))\\&\quad+\gamma\inf_{\mathrm{D}(p||P_{f}(s,\hpis(s)))\leq \rho}\E_{s'\sim p}\Big[V^{\text{R}}_{\hpis,f}(s')\Big]-
\gamma\inf_{\mathrm{D}(p||P_{\hat{f}_{n}}(s,\hpis(s)))\leq \rho}\E_{s'\sim p}\Big[V^{\text{R}}_{\hpis,\hat{f}_{n}}(s')\Big]|   
    \end{split}\nonumber\\\label{eq: rob-q-func-kl}
    &=|\gamma\inf_{\mathrm{D}(p||P_{f}(s,\hpis(s)))\leq \rho}\E_{s'\sim p}\Big[V^{\text{R}}_{\hpis,f}(s')\Big]-
\gamma\inf_{\mathrm{D}(p||P_{\hat{f}_{n}}(s,\hpis(s)))\leq \rho}\E_{s'\sim p}\Big[V^{\text{R}}_{\hpis,\hat{f}_{n}}(s')\Big]|
\end{align}

Adding and subtracting $\gamma\inf\limits_{\mathrm{D}(p||P_{\hat{f}_{n}}(s,\hpis(s)))\leq \rho}\E_{s'\sim p}\Big[V^{\text{R}}_{\hpis,f}(s')\Big]$ to \Cref{eq: rob-q-func-kl}, we obtain the following two terms:
\begin{equation}
    (i_a)=|\gamma\inf_{\mathrm{D}(p||P_{f}(s,\hpis(s)))\leq \rho}\E_{s'\sim p}\Big[V^{\text{R}}_{\hpis,f}(s')\Big]-\gamma\inf_{\mathrm{D}(p||P_{\hat{f}_{n}}(s,\hpis(s)))\leq \rho}\E_{s'\sim p}\Big[V^{\text{R}}_{\hpis,f}(s')\Big]|,  \nonumber
\end{equation}
\begin{equation}
    (i_b)=|\gamma\inf_{\mathrm{D}(p||P_{\hat{f}_{n}}(s,\hpis(s)))\leq \rho}\E_{s'\sim p}\Big[V^{\text{R}}_{\hpis,f}(s')\Big]-\gamma\inf_{\mathrm{D}(p||P_{\hat{f}_{n}}(s,\hpis(s)))\leq \rho}\E_{s'\sim p}\Big[V^{\text{R}}_{\hpis,\hat{f}_{n}}(s')\Big]|.  \nonumber
\end{equation}

Now, we use \Cref{lemma: diff-func} to bound $(i_b)$. We have: 
\begin{align}
    (i_b)&=|\gamma\inf_{\mathrm{D}(p||P_{\hat{f}_{n}}(s,\hpis(s)))\leq \rho}\E_{s'\sim p}\Big[V^{\text{R}}_{\hpis,f}(s')\Big]-\gamma\inf_{\mathrm{D}(p||P_{\hat{f}_{n}}(s,\hpis(s)))\leq \rho}\E_{s'\sim p}\Big[V^{\text{R}}_{\hpis,\hat{f}_{n}}(s')\Big]|  \nonumber\\\label{eq: i_b bound}
    &\stackrel{\Cref{lemma: diff-func}}{\leq} \gamma\max_{s}\Big|V^{\text{R}}_{\hpis,f}(s)-V^{\text{R}}_{\hpis,\hat{f}_{n}}(s)\Big|\qquad (\Cref{lemma: diff-func}).
\end{align}

Plugging \Cref{eq: i_b bound} into \Cref{eq: (i) def} and using the fact that $(i)=(i_a)+(i_b)$, we have
\begin{align}
      \label{eq: redefinition of i}
     (i)&=|V^{\text{R}}_{\hpis,f}(s)-V^{\text{R}}_{\hpis,\hat{f}_{n}}(s)| \\
    &\leq (i_a)+\gamma\max_{s}\Big|V^{\text{R}}_{\hpis,f}(s)-V^{\text{R}}_{\hpis,\hat{f}_{n}}(s)\Big|  \nonumber\\
    \label{eq: ib bound usage}
    \begin{split}
         &= |\gamma\inf_{\mathrm{D}(p||P_{f}(s,\hpis(s)))\leq \rho}\E_{s'\sim p}\Big[V^{\text{R}}_{\hpis,f}(s')\Big]-\gamma\inf_{\mathrm{D}(p||P_{\hat{f}_{n}}(s,\hpis(s)))\leq \rho}\E_{s'\sim p}\Big[V^{\text{R}}_{\hpis,f}(s')\Big]|\\&\quad + \gamma\max_{s}\Big|V^{\text{R}}_{\hpis,f}(s)-V^{\text{R}}_{\hpis,\hat{f}_{n}}(s)\Big|.
    \end{split}
\end{align}
Taking maximum over states in \Cref{eq: redefinition of i} and \Cref{eq: ib bound usage} we have
\begin{align}
    &\max_{s}\Big|V^{\text{R}}_{\hpis,f}(s)-V^{\text{R}}_{\hpis,\hat{f}_{n}}(s)\Big|\nonumber\\
    \begin{split}
         &\leq \max_{s}\Big|\gamma\inf_{\mathrm{D}(p||P_{f}(s,\hpis(s)))\leq \rho}\E_{s'\sim p}\Big[V^{\text{R}}_{\hpis,f}(s')\Big]-\gamma\inf_{\mathrm{D}(p||P_{\hat{f}_{n}}(s,\hpis(s)))\leq \rho}\E_{s'\sim p}\Big[V^{\text{R}}_{\hpis,f}(s')\Big]\Big|\\&\quad + \gamma\max_{s}\Big|V^{\text{R}}_{\hpis,f}(s)-V^{\text{R}}_{\hpis,\hat{f}_{n}}(s)\Big|.
    \end{split}\nonumber
\end{align}
Moving $\gamma\max_{s}\Big|V^{\text{R}}_{\hpis,f}(s)-V^{\text{R}}_{\hpis,\hat{f}_{n}}(s)\Big|$ to the LHS and dividing $(1-\gamma)$ on both sides, it holds that
\begin{align}
    (i)&\leq\max_{s}\Big|V^{\text{R}}_{\hpis,f}(s)-V^{\text{R}}_{\hpis,\hat{f}_{n}}(s)\Big|\nonumber\\
   \label{eq: first-layer-bound-kl-1}&\leq \frac{\gamma}{1-\gamma}\max_{s}\Big|\inf_{\mathrm{D}(p||P_{f}(s,\hpis(s)))\leq \rho}\E_{s'\sim p}\Big[V^{\text{R}}_{\hpis,f}(s')\Big]-\inf_{\mathrm{D}(p||P_{\hat{f}_{n}}(s,\hpis(s)))\leq \rho}\E_{s'\sim p}\Big[V^{\text{R}}_{\hpis,f}(s')\Big]\Big|.
\end{align}

\end{proof}
\begin{lemma}\label{lemma: (ii) bound lemma}(Simplification using robust Bellman equation)
Denote $(ii):=\Big|V^{\text{R}}_{\hpis,\hat{f}_n}(s)-V^{\text{R}}_{\pi^{*},f}(s)\Big|$ for  $V^{\text{R}}_{\hpis,\hat{f}_n}$ being the robust value function of policy $\hpis$ w.r.t.~learned nominal transition dynamics $\hat{f}_n$ and $V^{\text{R}}_{\pi^{*},f}$ being the robust value function of policy $\pi^{*}$ w.r.t.~true nominal transition dynamics $f$. Then the following holds,
   \begin{align}
(ii)&=\Big|V^{\text{R}}_{\hpis,\hat{f}_n}(s)-V^{\text{R}}_{\pi^{*},f}(s)\Big|\nonumber \\&\leq\max_{s}\Big|V^{\text{R}}_{\hpis,\hat{f}_n}(s)-V^{\text{R}}_{\pi^{*},f}(s)\Big|\nonumber \\&\leq \frac{\gamma}{1-\gamma}\max_{s}\Big|\inf_{\mathrm{D}(p||P_{\hat{f}_{n}}(s,\hpis(s)))\leq \rho}\E_{s'\sim p}\Big[V^{\text{R}}_{\pi^{*},f}(s')\Big]-\inf_{\mathrm{D}(p||P_{f}(s,\hpis(s)))\leq \rho}\E_{s'\sim p}\Big[V^{\text{R}}_{\pi^{*},f}(s')\Big]\Big|.
\end{align}
\end{lemma}
\begin{proof}
 We first note that $Q^{\text{R}}_{\pi^{*},f}(s,\hpis(s))\leq Q^{\text{R}}_{\pi^{*},f}(s,\pi^{*}(s))$ as $\pi^{*}$ is the robust optimal policy for the nominal transition dynamics  $f$ (see \Cref{eq: robust objective}). As a result, we have
\begin{align}\label{eq: (ii) def}
    (ii)&=|V^{\text{R}}_{\hpis,\hat{f}_{n}}(s)-V^{\text{R}}_{\pi^{*},f}(s)|  \\
    &=|Q^{\text{R}}_{\hpis,\hat{f}_{n}}(s,\hpis(s))-Q^{\text{R}}_{\pi^{*},f}(s,\pi^{*}(s))|  \nonumber\\
    &\leq |Q^{\text{R}}_{\hpis,\hat{f}_{n}}(s,\hpis(s))-Q^{\text{R}}_{\pi^{*},f}(s,\hpis(s))|  \nonumber\\
    \begin{split}
         &= |r(s,\hpis(s))-r(s,\hpis(s))\\ &\quad +\gamma\inf_{\mathrm{D}(p||P_{\hat{f}_{n}}(s,\hpis(s)))\leq \rho}\E_{s'\sim p}\Big[V^{\text{R}}_{\hpis,\hat{f}_{n}}(s')\Big]-
\gamma\inf_{\mathrm{D}(p||P_{f}(s,\hpis(s)))\leq \rho}\E_{s'\sim p}\Big[V^{\text{R}}_{\pi^{*},f}(s')\Big]|.    
    \end{split} \nonumber\\\label{eq: rob-q-func-kl-2}
    &=|\gamma\inf_{\mathrm{D}(p||P_{\hat{f}_{n}}(s,\hpis(s)))\leq \rho}\E_{s'\sim p}\Big[V^{\text{R}}_{\hpis,\hat{f}_{n}}(s')\Big]-
\gamma\inf_{\mathrm{D}(p||P_{f}(s,\hpis(s)))\leq \rho}\E_{s'\sim p}\Big[V^{\text{R}}_{\pi^{*},f}(s')\Big]|
\end{align}

Adding and subtracting $\gamma\inf\limits_{\mathrm{D}(p||P_{\hat{f}_{n}}(s,\hpis(s)))\leq \rho}\E_{s'\sim p}\Big[V^{\text{R}}_{\pi^{*},f}(s')\Big]$ to \Cref{eq: rob-q-func-kl-2}, we obtain the following two terms: 
\begin{equation}
    (ii_a)=|\gamma\inf_{\mathrm{D}(p||P_{\hat{f}_{n}}(s,\hpis(s)))\leq \rho}\E_{s'\sim p}\Big[V^{\text{R}}_{\hpis,\hat{f}_{n}}(s')\Big]-\gamma\inf_{\mathrm{D}(p||P_{\hat{f}_{n}}(s,\hpis(s)))\leq \rho}\E_{s'\sim p}\Big[V^{\text{R}}_{\pi^{*},f}(s')\Big]|,  \nonumber
\end{equation}
\begin{equation}
    (ii_b)=|\gamma\inf_{\mathrm{D}(p||P_{\hat{f}_{n}}(s,\hpis(s)))\leq \rho}\E_{s'\sim p}\Big[V^{\text{R}}_{\pi^{*},f}(s')\Big]-\gamma\inf_{\mathrm{D}(p||P_{f}(s,\hpis(s)))\leq \rho}\E_{s'\sim p}\Big[V^{\text{R}}_{\pi^{*},f}(s')\Big]|.  \nonumber
\end{equation}

Now, we use \Cref{lemma: diff-func} to bound $(ii_a)$ . We have: 
\begin{align}
    (ii_a)&=|\gamma\inf_{\mathrm{D}(p||P_{\hat{f}_{n}}(s,\hpis(s)))\leq \rho}\E_{s'\sim p}\Big[V^{\text{R}}_{\hpis,\hat{f}_{n}}(s')\Big]-\gamma\inf_{\mathrm{D}(p||P_{\hat{f}_{n}}(s,\hpis(s)))\leq \rho}\E_{s'\sim p}\Big[V^{\text{R}}_{\pi^{*},f}(s')\Big]|  \nonumber\\
    \label{eq: ii_a bound}
    &\leq \gamma\max_{s}\Big|V^{\text{R}}_{\pi^{*},f}(s)-V^{\text{R}}_{\hpis,\hat{f}_{n}}(s)\Big|.
\end{align}

Plugging \Cref{eq: ii_a bound} into \Cref{eq: (ii) def} and using the fact that $(ii)=(ii_a)+(ii_b)$, we have
\begin{align}
      \label{eq: redefinition of ii}
     (ii)&=|V^{\text{R}}_{\pi^{*},f}(s)-V^{\text{R}}_{\hpis,\hat{f}_{n}}(s) |\\
    &\leq (ii_b)+\max_{s}\Big|V^{\text{R}}_{\pi^{*},f}(s)-V^{\text{R}}_{\hpis,\hat{f}_{n}}(s)\Big|  \nonumber\\
    \label{eq: iia bound usage}
    \begin{split}
         &= |\gamma\inf_{\mathrm{D}(p||P_{\hat{f}_{n}}(s,\hpis(s)))\leq \rho}\E_{s'\sim p}\Big[V^{\text{R}}_{\pi^{*},f}(s')\Big]-\gamma\inf_{\mathrm{D}(p||P_{f}(s,\hpis(s)))\leq \rho}\E_{s'\sim p}\Big[V^{\text{R}}_{\pi^{*},f}(s')\Big]|\\&\quad + \gamma\max_{s}\Big|V^{\text{R}}_{\pi^{*},f}(s)-V^{\text{R}}_{\hpis,\hat{f}_{n}}(s)\Big|.
    \end{split}
\end{align}
Taking maximum over states in \Cref{eq: redefinition of ii} and \Cref{eq: iia bound usage} and following similar steps as in \Cref{eq: first-layer-bound-kl-1}, we have
\begin{align}
    (ii)&\leq\max_{s}\Big|V^{\text{R}}_{\pi^{*},f}(s)-V^{\text{R}}_{\hpis,\hat{f}_{n}}(s)\Big|\nonumber\\
    \begin{split}
         &\leq \max_{s}\Big|\gamma\inf_{\mathrm{D}(p||P_{f}(s,\hpis(s)))\leq \rho}\E_{s'\sim p}\Big[V^{\text{R}}_{\hpis,f}(s')\Big]-\gamma\inf_{\mathrm{D}(p||P_{\hat{f}_{n}}(s,\hpis(s)))\leq \rho}\E_{s'\sim p}\Big[V^{\text{R}}_{\hpis,f}(s')\Big]\Big|\\&\quad + \gamma\max_{s}\Big|V^{\text{R}}_{\pi^{*},f}(s)-V^{\text{R}}_{\hpis,\hat{f}_{n}}(s)\Big|
    \end{split}\nonumber\\
    \label{eq: first-layer-bound-kl-2}
    &\leq \frac{\gamma}{1-\gamma}\max_{s}\Big|\inf_{\mathrm{D}(p||P_{\hat{f}_{n}}(s,\hpis(s)))\leq \rho}\E_{s'\sim p}\Big[V^{\text{R}}_{\pi^{*},f}(s')\Big]-\inf_{\mathrm{D}(p||P_{f}(s,\hpis(s)))\leq \rho}\E_{s'\sim p}\Big[V^{\text{R}}_{\pi^{*},f}(s')\Big]\Big|.
\end{align}

\end{proof}

%% file: supplementary/10_diff-func.tex
\begin{lemma}\label{lemma: diff-func}
(from \citet[Lemma 1]{panaganti2022sample}) Let $V_{1}$ and $V_{2}$ be two value functions from $ \cS\to [0,1/(1-\gamma)] $. Let $D$ be any distance measure between probability distributions (e.g., KL-divergence, $\chi^{2}-$ divergence, or variation distance defined in \Cref{eq: uncertainty set}). The following inequality (1-Lipschitz w.r.t. $V$) holds true\looseness=-1 
\begin{multline}
\Big|\inf_{D(p||P_{\tilde{f}}(s,a))\leq \rho}\E_{s'\sim p}\Big[V_{1}(s')\Big]-\inf_{D(p||P_{\tilde{f}}(s,a))\leq \rho}\E_{s'\sim p}\Big[V_{2}(s')\Big]\Big|\leq \max_{s'}|V_{2}(s')-V_{1}(s')|.\nonumber
\end{multline}
\end{lemma}
\begin{proof}
We want to bound 
\begin{equation}
\Big|\inf_{D(p||P_{\tilde{f}}(s,a))\leq \rho}\E_{s'\sim p}\Big[V_{1}(s')\Big]-\inf_{D(p||P_{\tilde{f}}(s,a))\leq \rho}\E_{s'\sim p}\Big[V_{2}(s')\Big]\Big|.\nonumber
\end{equation}
Notice that
\begin{align*}
    &\inf_{D(p||P_{\tilde{f}}(s,a))\leq \rho}\E_{s'\sim p}\Big[V_{1}(s')\Big]-\inf_{D(p||P_{\tilde{f}}(s,a))\leq \rho}\E_{s'\sim p}\Big[V_{2}(s')\Big]\\
    &= \inf_{D(p||P_{\tilde{f}}(s,a))\leq \rho}\sup_{D(p'|P_{\tilde{f}}(s,a))\leq \rho}\E_{s'\sim p}\Big[V_{1}(s')\Big]-\E_{s'\sim p'}\Big[V_{2}(s')\Big]\\
    &\geq  \inf_{D(p||P_{\tilde{f}}(s,a))\leq \rho}\E_{s'\sim p}\Big[V_{1}(s')\Big]-\E_{s'\sim p}\Big[V_{2}(s')\Big]\\
    &= \inf_{D(p||P_{\tilde{f}}(s,a))\leq \rho}\E_{s'\sim p}\Big[V_{1}(s')-V_{2}(s')\Big],
\end{align*}
where the inequality follows from the property of supremum. By the definition of  $\inf$, for any $\epsilon>0$, there exists some distribution $q$ s.t. $D(q|P_{\tilde{f}}(s,a))\leq \rho$ satisfying
\begin{equation}
    \E_{s'\sim q}\Big[V_{1}(s')-V_{2}(s')\Big]-\epsilon\leq \inf_{D(p||P_{\tilde{f}}(s,a)))\leq \rho}\E_{s'\sim p}\Big[V_{1}(s')-V_{2}(s')\Big].\nonumber
\end{equation}
Then, we have 
\begin{align}
     &\inf_{D(p||P_{\tilde{f}}(s,a))\leq \rho}\E_{s'\sim p}\Big[V_{2}(s')\Big]-\inf_{D(p||P_{\tilde{f}}(s,a))\leq \rho}\E_{s'\sim p}\Big[V_{1}(s')\Big] \nonumber\\
     &\leq -\inf_{D(p||P_{\tilde{f}}(s,a))\leq \rho}\E_{s'\sim p}\Big[V_{1}(s')-V_{2}(s')\Big] \nonumber\\
    &\leq -\E_{s'\sim q}\Big[V_{1}(s')-V_{2}(s')\Big]+\epsilon \nonumber\\
     &\leq \E_{s'\sim q}\Big[V_{2}(s')-V_{1}(s')\Big]+\epsilon \nonumber\\
    \label{eq: diff-func eqn}
    &\leq \max_{s'}|V_{2}(s')-V_{1}(s')|+\epsilon.
\end{align} 
Let $\epsilon\to 0$, we obtain one side of the desired bound.

One can similarly bound $\inf_{D(p||P_{\tilde{f}}(s,a))\leq \rho}\E_{s'\sim p}\Big[V_{1}(s')\Big]-\inf_{D(p||P_{\tilde{f}}(s,a))\leq \rho}\E_{s'\sim p}\Big[V_{2}(s')\Big]$ by just interchanging $V_{1}$ and $V_{2}$ everywhere. Combining this argument with \Cref{eq: diff-func eqn}, we obtain \begin{multline}
\Big|\inf_{D(p||P_{\tilde{f}}(s,a))\leq \rho}\E_{s'\sim p}\Big[V_{1}(s')\Big]-\inf_{D(p||P_{\tilde{f}}(s,a))\leq \rho}\E_{s'\sim p}\Big[V_{2}(s')\Big]\Big|\leq \max_{s'}|V_{2}(s')-V_{1}(s')|.\nonumber
\end{multline}
\end{proof}

%% file: supplementary/11_diff-opt-kl.tex
\begin{lemma}\label{lemma: diff-opt} (Simplification using \Cref{lemma:kl reform} reformulation)
For any value function $V(\cdot): \cS\to [0,1/(1-\gamma)] $, define the event \textbf{E} as follows:
\begin{multline}
    \max_{s} \left|\inf_{KL(p||P_{\hat{f}_{n}}(s,\hat{\pi}_{n}(s)))\leq \rho}\E_{s'\sim p}\Big[V(s')\Big]-\inf_{KL(p||P_{f}(s,\hat{\pi}_{n}(s)))\leq \rho}\E_{s'\sim p}\Big[V(s')\Big]\right|\leq\\ \max_{s,a}2\tfrac{M}{\rho} e^{\frac{M}{\underline{\alpha}}}\max_{\alpha\in[\underline{\alpha},\overline{\alpha}]}\left|\E_{s'\sim P_{\hat{f}_{n}(s,a)}}[e^{\frac{-V(s')}{\alpha}}]-\E_{s'\sim P_{f}(s,a)}[e^{\frac{-V(s')}{\alpha}}]\right|.  \nonumber
\end{multline}
Then, for any $n>\{\max_{s,a}N'(\rho,\psa),\max_{s,a}N''(\rho,\psa)\}$
where $N'(\rho,\psa)=\mathcal{O}\Big(\frac{\beta^{2}_n(\delta)2ed^{2}\Gamma_{nd}}{(\frac{\kappa-e^{-\rho}}{2})^{2}}\Big)$ and $N''(\rho,\psa)=\mathcal{O}\Big(\frac{4 M^{2}e^{\frac{2M}{\underline{\alpha}}}\beta^{2}_n(\delta)2ed^{2}\Gamma_{nd}}{(\rho\tau)^{2}}\Big)$ with $\overline{\alpha}=\frac{M}{\rho}$, $M=\frac{1}{1-\gamma}$, $\kappa$ defined in \Cref{eq: kappa equation}, $\tau$ defined in \Cref{eq: tau-equation}, and $\underline{\alpha}=\alpha^{*}/2$ defined in \Cref{eq: alpha transformation-1}, the event \textbf{E} holds true with probability at least $1-\delta$.
\end{lemma}
\begin{proof} (A similar proof as in \citet[Lemma-4]{zhou2021finite}). First note that,
\begin{multline}\label{eq: stosa}
    \max_{s}\left|\inf_{KL(p||P_{\hat{f}_{n}}(s,\hat{\pi}_{n}(s)))\leq \rho}\E_{s'\sim p}\Big[V(s')\Big]-\inf_{KL(p||P_{f}(s,\hat{\pi}_{n}(s)))\leq \rho}\E_{s'\sim p}\Big[V(s')\Big]\right|\leq\\ \max_{s,a}\left|\inf_{KL(p||P_{\hat{f}_{n}}(s,a))\leq \rho}\E_{s'\sim p}\Big[V(s')\Big]-\inf_{KL(p||P_{f}(s,a))\leq \rho}\E_{s'\sim p}\Big[V(s')\Big]\right|.
\end{multline}
Recall \cite[Theorem-1]{hu2013kullback} for distributionally robust optimization with a random variable $X$ and a random function $H$. One can rewrite an infinite-dimensional optimization problem as a scalar optimization problem:
\begin{equation}\label{eq: transformation}
    \sup_{P:KL(p||P_{0})\leq \rho}\E_{X\sim P}[H(X)]=\inf_{\alpha\geq 0}\{\alpha \log(\E_{X\sim P_{0}}[e^{\frac{H(X)}{\alpha}}])+\alpha \rho\}.
\end{equation}
For now, we focus on bounding $\left|\inf\limits_{KL(p||P_{\hat{f}_{n}}(s,a))\leq \rho}\E_{s'\sim p}\Big[V(s')\Big]-\inf\limits_{KL(p||P_{f}(s,a))\leq \rho}\E_{s'\sim p}\Big[V(s')\Big]\right|$ for one particular $(s,a)$. For brevity, we write $P_{f}(s,a)$ and $P_{\hat{f}_{n}}(s,a)$ as $P_{f}$ and $P_{\hat{f}_{n}}$, respectively. By \Cref{eq: transformation}, we have
\begin{equation}\label{eq: transformation-1}
    \inf_{P:KL(p||P_{f})\leq \rho}\E_{s'\sim P}[V(s')]=\max_{\alpha\geq 0}\{-\alpha \log(\E_{s'\sim P_{f}}[e^{\frac{-V(s')}{\alpha}}])-\alpha \rho\},
\end{equation}
\begin{equation}\label{eq: transformation-2}
    \inf_{P:KL(p||\hat{P}_{\hat{f}_{n}})\leq \rho}\E_{s'\sim P}[V(s')]=\max_{\alpha\geq 0}\{-\alpha \log(\E_{s'\sim P_{\hat{f}_{n}}}[e^{\frac{-V(s')}{\alpha}}])-\alpha \rho\}.
\end{equation}

For the finite state-action space setting, \citet[Lemma-4]{zhou2021finite} characterizes the property of the optimal $\alpha^{*}$. Following a similar proof strategy, we denote  
\begin{equation}\label{eq: alpha transformation-1}
    \alpha^{*}=\argmax_{\alpha\geq 0}\{-\alpha \log(\E_{s'\sim P_{f}}[e^{\frac{-V(s')}{\alpha}}])-\alpha \rho\},
\end{equation}
and 
\begin{equation}\label{eq: alpha transformation-2}
  \hat{\alpha}_{n}^{*}=\argmax_{\alpha\geq 0}\{-\alpha \log(\E_{s'\sim P_{\hat{f}_{n}}}[e^{\frac{-V(s')}{\alpha}}])-\alpha \rho\}.
\end{equation}

To ensure that $\max_{\alpha\geq 0}\{-\alpha \log(\E_{s'\sim P_{f}}[e^{\frac{-V(s')}{\alpha}}])-\alpha \rho\}-\max_{\alpha\geq 0}\{-\alpha \log(\E_{s'\sim P_{\hat{f}_{n}}}[e^{\frac{-V(s')}{\alpha}}])-\alpha \rho\}$ is small enough, we need to show that $ \alpha^{*}$ and $\hat{\alpha}_{n}^{*}$ are close enough.

For this, one considers two different cases, $\alpha^{*}=0$ and $\alpha^{*}>0$.

\textbf{Case-1}:

In Case-1, we investigate the conditions for $\hat{\alpha}^{*}_n=0$ given that $\alpha^*=0$. 
According to \citep[Proposition-2]{hu2013kullback}, for $\alpha^{*}=0$ to occur, the random variable $Y:=V(s')$ where $s'\sim \mathcal{N}(f(s,a),\sigma^{2} I)$ must satisfy three conditions namely, (i) $Y$ must be bounded, 
 (ii) $Y$ must have finite mass at its essential infimum, and (iii) the finite mass at essential infimum should be greater than $e^{-\rho}$. So we want to verify whether these conditions hold true for $\hat{Y}_{n}:=V(s')$ where $s'\sim \mathcal{N}(\hat{f}_{n}(s,a),\sigma^{2} I)$ when
$Y$ satisfies these conditions.

We restate definition of the essential infimum for a real-valued random variable $Y$, denoted as $\mathrm{ESI}(Y)$.
\begin{equation}\label{eq: Essential infimum}
    \mathrm{ESI}(Y)=\sup\{t\in \mathbb{R}: \pr\{Y<t\}=0\}.
\end{equation}
We first show that $Y=V(s')$ where $s'\sim \mathcal{N}(f(s,a),\sigma^{2} I)$ and $\hat{Y}_{n}=V(s')$ where $s'\sim \mathcal{N}(\hat{f}_{n}(s,a),\sigma^{2} I)$ have the same essential infimum. By the definition of $\mathrm{ESI}(Y)$, for any $\epsilon > 0$, it holds that 
\begin{equation}\label{eq: Essential infimum expand}
\pr\{\mathrm{ESI}(Y)\leq Y<\mathrm{ESI}(Y)+\epsilon\}> 0,\quad \pr\{ Y<\mathrm{ESI}(Y)\}= 0.
\end{equation}
It implies for $Y=V(s')$ and $s'\sim \mathcal{N}(f(s,a),\sigma^{2} I)$ that
\begin{align}\label{eq: probgreater-1}
    \pr_{ s'\sim \mathcal{N}(f(s,a),\sigma^{2} I)}\{s'\in \mathbb{R}^{d}:\mathrm{ESI}(Y)\leq Y=V(s')<\mathrm{ESI}(Y)+\epsilon\}&>0 ,\\\label{eq: probsmaller-1}
    \pr_{ s'\sim \mathcal{N}(f(s,a),\sigma^{2} I)}\{s'\in \mathbb{R}^{d}:Y=V(s')<\mathrm{ESI}(Y)\}&=0.
\end{align}
It further implies that,  the set $\{s'\in \mathbb{R}^{d}:\mathrm{ESI}(Y)\leq V(s')<\mathrm{ESI}(Y)+\epsilon\}$ must have a Lebesgue measure greater than 0 and $\{s'\in \mathbb{R}^{d}:V(s')<\mathrm{ESI}(Y)\}$ must have a Lebesgue measure equal to 0 since $s'\sim \mathcal{N}(f(s,a),\sigma^{2} I)$ is a continuous distribution.

Due to this fact that the set $\{s'\in \mathbb{R}^{d}:\mathrm{ESI}(Y)\leq V(s')<\mathrm{ESI}(Y)+\epsilon\}$ has a Lebesgue measure greater than zero and noting that $\mathcal{N}(\hat{f}_{n}(s,a),\sigma^{2} I)$ is also a continuous distribution with the same support  as of $\mathcal{N}(f(s,a),\sigma^{2} I)$ (i.e., the probability density function of $\mathcal{N}(\hat{f}_{n}(s,a),\sigma^{2} I)$ is positive whenever probability density function of $\mathcal{N}(f(s,a),\sigma^{2} I)$ is positive), it holds that
\begin{align}\label{eq: lebesgue-prop-1}
    \pr_{ s'\sim \mathcal{N}(\hat{f}_{n}(s,a),\sigma^{2} I)}\{s'\in \mathbb{R}^{d}:\mathrm{ESI}(Y)\leq \hat{Y}_{n}=V(s')<\mathrm{ESI}(Y)+\epsilon\}&>0.
\end{align}
A similar argument follows for 
\begin{equation}\label{eq: lebesgue-prop-2}
     \pr_{ s'\sim \mathcal{N}(\hat{f}_{n}(s,a),\sigma^{2} I)}\{s'\in \mathbb{R}^{d}:\hat{Y}_{n}=V(s')<\mathrm{ESI}(Y)\}=0.
\end{equation}

In essence, \Cref{eq: lebesgue-prop-1,eq: lebesgue-prop-2} imply,
$$\pr\{\mathrm{ESI}(Y)\leq \hat{Y}_{n}<\mathrm{ESI}(Y)+\epsilon\}= 0,\quad \pr\{ \hat{Y}_{n}<\mathrm{ESI}(Y)\}> 0.$$
Hence,  from the definition of $\mathrm{ESI}(\cdot)$ in \Cref{eq: Essential infimum,eq: Essential infimum expand}, we have $\mathrm{ESI}(Y)=\mathrm{ESI}(\hat{Y}_{n})$.

As a result, for $\alpha^{*}=0$ to occur and for $Y=V(s')(s'\sim \mathcal{N}(f(s,a),\sigma^{2} I))$ to have finite mass at the essential infimum (condition-(ii)), i.e,  
$\pr\{Y=\mathrm{ESI}(Y)\}> 0,$ it requires that
\begin{align}
    \pr_{ s'\sim\mathcal{N}(f(s,a),\sigma^{2} I)}\{s'\in \mathbb{R}^{d}:Y=V(s')=\mathrm{ESI}(Y)\}&>0.  \nonumber
\end{align}
This will further require that the set $\{s'\in \mathbb{R}^{d}:Y=V(s')=\mathrm{ESI}(Y)\}$ must have a Lebesgue measure greater than 0. Following a similar argument as to have obtained \Cref{eq: lebesgue-prop-1} (the probability density function of $\mathcal{N}(\hat{f}_{n}(s,a),\sigma^{2} I)$ is positive whenever probability density function of $\mathcal{N}(f(s,a),\sigma^{2} I)$ is positive), the set  $\{s'\in \mathbb{R}^{d}:Y=V(s')=\mathrm{ESI}(Y)\}$ having Lebesgue measure greater than 0, will imply
\begin{equation}\label{eq: alpha_n condition-1}
    \pr_{ s'\sim \mathcal{N}(\hat{f}_{n}(s,a),\sigma^{2} I)}\{s'\in \mathbb{R}^{d}:\hat{Y}_{n}=V(s')=\mathrm{ESI}(Y)\}>0, 
\end{equation}
and 
\begin{equation}\label{eq: alpha_n condition-2}
\pr\{\hat{Y}_{n}=\mathrm{ESI}(Y)\}> 0
\end{equation}
Since $ \mathrm{ESI}(Y)=\mathrm{ESI}(\hat{Y}_{n})$,  \Cref{eq: alpha_n condition-1,eq: alpha_n condition-2} imply
\begin{equation}
\pr\{\hat{Y}_{n}=\mathrm{ESI}(\hat{Y}_n)\}> 0,
\end{equation}
Hence, if $\pr\{Y=\mathrm{ESI}(Y)\}> 0$ holds true, it also holds that $\pr\{\hat{Y}_{n}=\mathrm{ESI}(\hat{Y}_n)\}> 0$. This implies that whenever Y has a finite mass at its essential infimum, $\hat{Y}_n$ also has finite mass at its essential infimum (condition-(ii) satisfied).

But, recall that according to \citep[Proposition-2]{hu2013kullback} for $\alpha^{*}=0$ to occur, the finite mass which $Y$ has at its essential infimum should also be greater than $e^{-\rho}$ (condition-(iii)).  Hence, one has to check if $Y$ satisfies
\begin{equation}\label{eq: kappa equation}
\pr_{ s'\sim \mathcal{N}(f(s,a),\sigma^{2} I)}\{s'\in \mathbb{R}^{d}:Y=V(s')=\mathrm{ESI}(Y)\}>e^{-\rho},    
\end{equation}
what is the condition that $Y_n$ satisfies
$$ \pr_{s'\sim \mathcal{N}(\hat{f}_{n}(s,a),\sigma^{2} I)}\{s'\in \mathbb{R}^{d}:\hat{Y}_n=V(s')=\mathrm{ESI}(\hat{Y}_n)\}>e^{-\rho},$$ so that $\hat{\alpha}^{*}_n=0$ whenever $\alpha^{*}=0$. 
Denote $\kappa:=\pr_{ s'\sim \mathcal{N}(f(s,a),\sigma^{2} I)}\{s'\in \mathbb{R}^{d}:Y=V(s')=\mathrm{ESI}(Y)\}$, $\kappa_n:=\pr_{s'\sim \mathcal{N}(\hat{f}_{n}(s,a),\sigma^{2} I)}\{s'\in \mathbb{R}^{d}:\hat{Y}_n=V(s')=\mathrm{ESI}(\hat{Y}_n)\}$, and $S_{min}:=\{s'\in \mathbb{R}^{d}:V(s')=\mathrm{ESI}(Y)=\mathrm{ESI}(\hat{Y}_n)\}$. If $\kappa>e^{-\rho}$ and $\kappa-\kappa_{n}\leq \frac{\kappa-e^{-\rho}}{2}$, then it will hold that $\kappa_n>e^{-\rho}$.  
\begin{align}
    |\kappa-\kappa_{n}|&=\Bigg|\int_{S_{min}}\frac{1}{\sqrt{(2\pi\sigma^{2})^{d}}}(e^{-\frac{\|s'-f(s,a)\|^{2}}{\sigma^{2}}}-e^{-\frac{\|s'-\hat{f}_{n}(s,a)\|^{2}}{\sigma^{2}}})dx\Bigg|  \nonumber\\
    &\leq\int_{S_{min}}\frac{1}{\sqrt{(2\pi\sigma^{2})^{d}}}\Bigg|e^{-\frac{\|s'-f(s,a)\|^{2}}{\sigma^{2}}}-e^{-\frac{\|s'-\hat{f}_{n}(s,a)\|^{2}}{\sigma^{2}}}\Bigg| dx \nonumber\\
    &\leq\int_{\mathbb{R}^{d}}\frac{1}{\sqrt{(2\pi\sigma^{2})^{d}}}\Bigg|e^{-\frac{\|s'-f(s,a)\|^{2}}{\sigma^{2}}}-e^{-\frac{\|s'-\hat{f}_{n}(s,a)\|^{2}}{\sigma^{2}}}\Bigg|dx  \nonumber\\
    &\leq \|f(s,a)-\hat{f}_{n}(s,a)\|_{2} \quad (\Cref{lemma: error bound})  \nonumber\\
    &\leq \mathcal{O}\Big(\frac{\beta_{n}(\delta)\sqrt{2ed^{2}\Gamma_{nd}}}{\sqrt{n}}\Big),  \nonumber
\end{align}
 
We need $\mathcal{O}\Big(\frac{\beta_{n}(\delta)\sqrt{2ed^{2}\Gamma_{nd}}}{\sqrt{n}}\Big)\leq \frac{\kappa-e^{-\rho}}{2}$, 
which in turn requires $n=\mathcal{O}\Big(\frac{\beta^{2}_n(\delta)2ed^{2}\Gamma_{nd}}{(\frac{\kappa-e^{-\rho}}{2})^{2}}\Big)=N'(\rho,\psa)$. Hence, for $n>\max_{s,a}N'(\rho,\psa)$ with probability at least $1-\delta$, it holds that
$$\kappa_{n}>e^{-\rho},$$ for all $(s,a)\in \cS\times\A$ whenever $\kappa>e^{-\rho}$, implying $\alpha_{n}^{*}=0$ whenever $\alpha^{*}=0$.\\

\textbf{Case-2}:
Consider the case of $\alpha^{*}>0$. The idea is to bound both $\alpha^{*}$ and $\hat{\alpha}^{*}_{n}$
by a set $[\underline{\alpha}$,$\overline{\alpha}]$ and bound $\max_{\alpha\geq 0}\{(-\alpha \log(\E_{s'\sim P_{f}(s,\pi(s))}[e^{\frac{-V(s')}{\alpha}}])-\alpha \rho)-(-\alpha \log(\E_{s'\sim P_{f(s,\pi'(s))}}[e^{\frac{-V(s')}{\alpha}}])-\alpha \rho)\}$ for $\alpha$ taking values within set  $[\underline{\alpha}$, $\overline{\alpha}]$.
We first provide a upper bound for $\alpha^{*}$ as $\frac{M}{\rho}$ where $M=\frac{1}{1-\gamma}$ denoting the maximum value of $V(s')$. \begin{align}
    \max_{\alpha\geq 0}\{-\alpha \log(\E_{s'\sim P_{f}}[e^{\frac{-V(s')}{\alpha}}])-\alpha \rho\}&\geq \lim_{\alpha\to 0}[-\alpha \log(\E_{s'\sim P_{f}}[e^{\frac{-V(s')}{\alpha}}])-\alpha \rho]  \nonumber\\
    &=\mathrm{ESI}(V(s')|_{s'\sim P_{f}})\quad (\Cref{lemma: limit bound})  \nonumber\\
    \label{eq: alpha-bound}
    &\geq 0.
\end{align}
Since $\max_s V(s)\leq M$, we have
\begin{equation}
    -\alpha \log(\E_{s'\sim P_{f}}[e^{\frac{-V(s')}{\alpha}}])-\alpha \rho \leq -\alpha \log(e^{\frac{-M}{\alpha}})-\alpha \rho = M-\alpha \rho.  \nonumber
\end{equation}
It implies for $\alpha >\frac{M}{\rho}$ that  
\begin{equation}\label{eq: alpha-bound-2}
    -\alpha \log(\E_{s'\sim P_{f}}[e^{\frac{-V(s')}{\alpha}}])-\alpha \rho<0.
\end{equation}

By \Cref{eq: alpha-bound}, since $\max_{\alpha\geq 0}\{-\alpha \log(\E_{s'\sim P_{f}}[e^{\frac{-V(s')}{\alpha}}])-\alpha \rho\}\geq 0$, $\argmax_{\alpha\geq 0}\{-\alpha \log(\E_{s'\sim P_{f}}[e^{\frac{-V(s')}{\alpha}}])-\alpha \rho\}$ cannot be greater than  $\frac{M}{\rho}$ due to \Cref{eq: alpha-bound-2} holding for all $\alpha >\frac{M}{\rho}$. Hence, we have $\alpha^{*}\leq\frac{M}{\rho}$. A similar argument holds for $\hat{\alpha}_{n}^{*}$ and it holds that $\hat{\alpha}_{n}^{*}\leq\frac{M}{\rho}$.

Denote $\underline{\alpha}:=\alpha^{*}/2$, $\overline{\alpha}:=\frac{M}{\rho}$, and
\begin{equation}\label{eq: tau-equation}
\tau:=\min\Big\{\underline{\alpha} \log(\E_{s'\sim P_{f}}[e^{\frac{-V(s')}{\underline{\alpha}}}])+\underline{\alpha} \rho,\overline{\alpha} \log(\E_{s'\sim P_{f}}[e^{\frac{-V(s')}{\overline{\alpha}}}])+\overline{\alpha} \rho\Big\}-\alpha^{*} \log\Big(\E_{s'\sim P_{f}}[e^{\frac{-V(s')}{\alpha^{*}}}]\Big)-\alpha^{*} \rho.  \nonumber
\end{equation}
We first show that,
\begin{align}\label{eq: log exp bound}
     \Big|\log(\frac{\E_{s'\sim P_{\hat{f}_{n}}}[e^{\frac{-V(s')}{\alpha}}]}{\E_{s'\sim P_{f}}[e^{-\frac{V(s')}{\alpha}}]}) \Big| \leq e^{\frac{M}{\alpha}}|\E_{s'\sim P_{\hat{f}_{n}}}[e^{\frac{-V(s')}{\alpha}}]-\E_{s'\sim P_{f}}[e^{\frac{-V(s')}{\alpha}}] |.
\end{align}
Consider 2 cases:  $\E_{s'\sim P_{\hat{f}_{n}}}[e^{\frac{-V(s')}{\alpha}}]\geq\E_{s'\sim P_{f}}[e^{\frac{-V(s')}{\alpha}}]$ and $\E_{s'\sim P_{f}}[e^{\frac{-V(s')}{\alpha}}]>\E_{s'\sim P_{\hat{f}_{n}}}[e^{\frac{-V(s')}{\alpha}}]$

\textbf{Case-1:} $\E_{s'\sim P_{\hat{f}_{n}}}[e^{\frac{-V(s')}{\alpha}}]\geq\E_{s'\sim P_{f}}[e^{\frac{-V(s')}{\alpha}}]$:
\begin{align}
     |\log(\frac{\E_{s'\sim P_{\hat{f}_{n}}}[e^{\frac{-V(s')}{\alpha}}]}{\E_{s'\sim P_{f}}[e^{-\frac{V(s')}{\alpha}}]}) | &= \log(\frac{\E_{s'\sim P_{\hat{f}_{n}}}[e^{\frac{-V(s')}{\alpha}}]}{\E_{s'\sim P_{f}}[e^{-\frac{V(s')}{\alpha}}]})\nonumber\\
     &=\log(1+\frac{\E_{s'\sim P_{\hat{f}_{n}}}[e^{\frac{-V(s')}{\alpha}}]-\E_{s'\sim P_{f}}[e^{\frac{-V(s')}{\alpha}}]}{\E_{s'\sim P_{f}}[e^{-\frac{V(s')}{\alpha}}]})\nonumber\\
     &\leq \frac{\E_{s'\sim P_{\hat{f}_{n}}}[e^{\frac{-V(s')}{\alpha}}]-\E_{s'\sim P_{f}}[e^{\frac{-V(s')}{\alpha}}]}{\E_{s'\sim P_{f}}[e^{-\frac{V(s')}{\alpha}}]}\nonumber\\
     &\leq e^{\frac{M}{\alpha}}(\E_{s'\sim P_{\hat{f}_{n}}}[e^{\frac{-V(s')}{\alpha}}]-\E_{s'\sim P_{f}}[e^{\frac{-V(s')}{\alpha}}] ).\nonumber
\end{align}

\textbf{Case-2:} $\E_{s'\sim P_{\hat{f}_{n}}}[e^{\frac{-V(s')}{\alpha}}]<\E_{s'\sim P_{f}}[e^{\frac{-V(s')}{\alpha}}]$:
\begin{align}
     |\log(\frac{\E_{s'\sim P_{\hat{f}_{n}}}[e^{\frac{-V(s')}{\alpha}}]}{\E_{s'\sim P_{f}}[e^{-\frac{V(s')}{\alpha}}]}) | &= \log(\frac{\E_{s'\sim P_{f}}[e^{\frac{-V(s')}{\alpha}}]}{\E_{s'\sim P_{\hat{f}_{n}}}[e^{-\frac{V(s')}{\alpha}}]})\nonumber\\
     &=\log(1+\frac{\E_{s'\sim P_{f}}[e^{\frac{-V(s')}{\alpha}}]-\E_{s'\sim P_{\hat{f}_{n}}}[e^{\frac{-V(s')}{\alpha}}]}{\E_{s'\sim P_{\hat{f}_{n}}}[e^{-\frac{V(s')}{\alpha}}]})\nonumber\\
     &\leq \frac{\E_{s'\sim P_{f}}[e^{\frac{-V(s')}{\alpha}}]-\E_{s'\sim P_{\hat{f}_{n}}}[e^{\frac{-V(s')}{\alpha}}]}{\E_{s'\sim P_{\hat{f}_{n}}}[e^{-\frac{V(s')}{\alpha}}]}\nonumber\\
     &\leq e^{\frac{M}{\alpha}}(\E_{s'\sim P_{f}}[e^{\frac{-V(s')}{\alpha}}]-\E_{s'\sim P_{\hat{f}_{n}}}[e^{\frac{-V(s')}{\alpha}}] ).\nonumber
\end{align}
Hence, \Cref{eq: log exp bound} holds. Then, for $\alpha \in [\underline{\alpha},\overline{\alpha}]$, we have 
\begin{align}
    &|(\alpha\log(\E_{s'\sim P_{\hat{f}_{n}}}[e^{\frac{-V(s')}{\alpha}}])+\alpha \rho)-(\alpha \log(\E_{s'\sim P_{f}}[e^{\frac{-V(s')}{\alpha}}])+\alpha \rho)|\\&=\alpha |\log(1+\frac{\E_{s'\sim P_{\hat{f}_{n}}}[e^{\frac{-V(s')}{\alpha}}]-\E_{s'\sim P_{f}}[e^{\frac{-V(s')}{\alpha}}]}{\E_{s'\sim P_{f}}[e^{-\frac{V(s')}{\alpha}}]}) |  \nonumber\\
    &\stackrel{(i)}{\leq} \alpha e^{\frac{M}{\alpha}}|\E_{s'\sim P_{\hat{f}_{n}}}[e^{\frac{-V(s')}{\alpha}}]-\E_{s'\sim P_{f}}[e^{\frac{-V(s')}{\alpha}}]|  \nonumber\\
    &\stackrel{(ii)}{\leq} \alpha e^{\frac{M}{\alpha}}\|f(s,a)-\hat{f}_{n}(s,a)\|\quad (\Cref{lemma: error bound})  \nonumber\\
    \label{eq: alpha-difference-bound}
    &\stackrel{(iii)}{\leq} \mathcal{O}\Big(\alpha e^{\frac{M}{\alpha}}\beta_{n}(\delta)\sqrt{\frac{2ed^{2}\Gamma_{nd}}{n}}\Big)\quad (\text{from} \Cref{eq: learning error bound}) .
\end{align}
Here (i) holds from \Cref{eq: log exp bound}, (ii) from \Cref{lemma: error bound} and (iii) from \Cref{eq: learning error bound}.

We further show that $\hat{\alpha}_{n}^{*}\in[\underline{\alpha},\overline{\alpha}]$. The first step in achieving that is to restrict $n>N''(\rho,\psa)=\mathcal{O}\Big(4\frac{ M^{2}e^{\frac{2M}{\underline{\alpha}}}\beta^{2}_n(\delta)2ed^{2}\Gamma_{nd}}{(\rho\tau)^{2}}\Big)$. It implies that if  $\mathcal{O}\Big(\alpha e^{\frac{M}{\alpha}}\beta_{n}(\delta)\sqrt{\frac{2ed^{2}\Gamma_{nd}}{n}}\Big)<\tau/2$ and 

for $n>\max_{s,a}N''(\rho,\psa)$ from \Cref{eq: alpha-difference-bound} with probability at least $1-\delta$, for all $(s,a)\in \cS\times\A$, we have

\begin{equation}\label{eq: tau-bound}
\max_{\underline{\alpha},\alpha^{*},\overline{\alpha}} |(\alpha\log(\E_{s'\sim P_{\hat{f}_{n}}}[e^{\frac{-V(s')}{\alpha}}])+\alpha \rho)-(\alpha \log(\E_{s'\sim P_{f}}[e^{\frac{-V(s')}{\alpha}}])+\alpha \rho)|\leq \tau/2.
\end{equation}
It further implies that
\begin{align}
&\max_{\alpha\in[\underline{\alpha},\overline{\alpha}]}\{(-\alpha\log(\E_{s'\sim P_{\hat{f}_{n}}}[e^{\frac{-V(s')}{\alpha}}])-\alpha \rho)\} \nonumber
\\&\stackrel{(i)}{\geq} -\alpha^{*}\log(\E_{s'\sim P_{\hat{f}_{n}}}[e^{\frac{-V(s')}{\alpha^{*}}}])-\alpha^{*} \rho \nonumber \\
    &\stackrel{(ii)}{\geq} -\alpha^{*}\log(\E_{s'\sim P_{f}}[e^{\frac{-V(s')}{\alpha^{*}}}])-\alpha^{*} \rho-\tau/2 \nonumber \\
    &\stackrel{(iii)}{\geq} \max\{-\underline{\alpha} \log(\E_{s'\sim P_{f}}[e^{\frac{-V(s')}{\underline{\alpha}}}])-\underline{\alpha} \rho,-\overline{\alpha} \log(\E_{s'\sim P_{f}}[e^{\frac{-V(s')}{\overline{\alpha}}}])-\overline{\alpha} \rho\}+\tau/2 \nonumber\\
    \label{eq: alpha_n-bound}
    &\stackrel{(iv)}{\geq} \max\{-\underline{\alpha} \log(\E_{s'\sim P_{\hat{f}_{n}}}[e^{\frac{-V(s')}{\underline{\alpha}}}])-\underline{\alpha} \rho,-\overline{\alpha} \log(\E_{s'\sim P_{\hat{f}_{n}}}[e^{\frac{-V(s')}{\overline{\alpha}}}])-\overline{\alpha} \rho\}.
\end{align}
where (i) follows from the fact that $\alpha^{*}\in [\underline{\alpha},\overline{\alpha}]$, (ii) follows from \Cref{eq: tau-bound}, (iii) follows from the definition of $\tau$ in \Cref{eq: tau-equation} and (iv) again follows from \Cref{eq: tau-bound}.

Thus $\hat{\alpha}_{n}^{*}\in[\underline{\alpha},\overline{\alpha}]$ follows from \Cref{eq: alpha_n-bound} and concavity of $-\alpha \log(\E_{s'\sim P_{f}}[e^{\frac{-V(s')}{\alpha}}])-\alpha \rho$ w.r.t. $\alpha$. Note that $\alpha^{*}$ also belongs in this set. We bound $
    \max_{\alpha\geq 0}\{-\alpha \log(\E_{s'\sim P_{f}}[e^{\frac{-V(s')}{\alpha}}])-\alpha \rho\}-\max_{\alpha\geq 0}\{-\alpha \log(\E_{s'\sim P_{\hat{f}_{n}}}[e^{\frac{-V(s')}{\alpha}}])-\alpha \rho\}$ only between $\alpha\in[\underline{\alpha},\overline{\alpha}]$ instead of all $\alpha >0$. As a result, it holds that 
\begin{align}
    &\Bigg| \max_{\alpha\in[\underline{\alpha},\overline{\alpha}]}\{-\alpha \log(\E_{s'\sim P_{f}}[e^{\frac{-V(s')}{\alpha}}])-\alpha \rho\}- \max_{\alpha\in[\underline{\alpha},\overline{\alpha}]}\{-\alpha \log(\E_{s'\sim P_{\hat{f}_{n}}}[e^{\frac{-V(s')}{\alpha}}])-\alpha \rho\}\Bigg|\\ &\leq  \max_{\alpha\in[\underline{\alpha},\overline{\alpha}]}\Bigg|\{-\alpha \log(\E_{s'\sim P_{f}}[e^{\frac{-V(s')}{\alpha}}])-\alpha \rho\}-\{-\alpha \log(\E_{s'\sim P_{\hat{f}_{n}}}[e^{\frac{-V(s')}{\alpha}}])-\alpha \rho\}\Bigg|. \nonumber
   \\&= \max_{\alpha\in[\underline{\alpha},\overline{\alpha}]}\alpha |\log(1+\frac{\E_{s'\sim P_{\hat{f}_{n}}}[e^{\frac{-V(s')}{\alpha}}]-\E_{s'\sim P_{f}}[e^{\frac{-V(s')}{\alpha}}]}{\E_{s'\sim P_{f}}[e^{-\frac{V(s')}{\alpha}}]}) |  \nonumber\\
    &\leq \max_{\alpha\in[\underline{\alpha},\overline{\alpha}]}2\alpha e^{\frac{M}{\alpha}}|\E_{s'\sim P_{\hat{f}_{n}}}[e^{\frac{-V(s')}{\alpha}}]-\E_{s'\sim P_{f}}[e^{\frac{-V(s')}{\alpha}}]|. \nonumber\\
     &\leq 2\tfrac{M}{\rho} e^{\frac{M}{\underline{\alpha}}}\max_{\alpha\in[\underline{\alpha},\overline{\alpha}]}|\E_{s'\sim P_{\hat{f}_{n}}}[e^{\frac{-V(s')}{\alpha}}]-\E_{s'\sim P_{f}}[e^{\frac{-V(s')}{\alpha}}]|, \nonumber
\end{align}
where the first inequality follows from \Cref{eq: log exp bound} and second inequality follows from the bounds of $\alpha$. Taking a maximum over all $(s,a)$ gets the desired result.

\end{proof}

%% file: supplementary/12_error-bound-kl.tex
\begin{lemma}\label{lemma: error bound}(Bound by difference between estimated model $\hat{f}_n$ and true $f$)
For any value function $V(s'): \cS\to [0,1/(1-\gamma)] $ and any $\alpha>0$, it holds that
\begin{equation*}
    |\E_{s'\sim P_{\hat{f}_{n}}(s,a)}[e^{\frac{-V(s')}{\alpha}}]-\E_{s'\sim P_{f}(s,a)}[e^{\frac{-V(s')}{\alpha}}]|\leq \sigma^{-1}\|f(s,a)-\hat{f}_{n}(s,a)\|,
\end{equation*} 
where $\pnsa=\mathcal{N}(\hat{f}_{n}(s,a),\sigma^{2} I)$ and $\psa=\mathcal{N}(f(s,a),\sigma^{2} I)$. 
\end{lemma}
\begin{proof}
\begin{align*}
    \left|\E_{s'\sim P_{\hat{f}_{n}}(s,a)}[e^{\frac{-V(s')}{\alpha}}]-\E_{s'\sim P_{f}(s,a)}[e^{\frac{-V(s')}{\alpha}}]\right|&= \left|\int_{\mathbb{R}^{d}}\frac{1}{\sqrt{(2\pi\sigma^{2})^{d}}}e^{\frac{-V(s')}{\alpha}}(e^{-\frac{\|x-f(s,a)\|^{2}}{2\sigma^{2}}}-e^{-\frac{\|x-\hat{f}_{n}(s,a)\|^{2}}{2\sigma^{2}}})\right|\\
    &\leq  \int_{\mathbb{R}^{d}}\frac{1}{\sqrt{(2\pi\sigma^{2})^{d}}}e^{\frac{-V(s')}{\alpha}}\left|e^{-\frac{\|x-f(s,a)\|^{2}}{2\sigma^{2}}}-e^{-\frac{\|x-\hat{f}_{n}(s,a)\|^{2}}{2\sigma^{2}}}\right|\\
    &\leq  \int_{\mathbb{R}^{d}}\frac{1}{\sqrt{(2\pi\sigma^{2})^{d}}}\left|e^{-\frac{\|x-f(s,a)\|^{2}}{2\sigma^{2}}}-e^{-\frac{\|x-\hat{f}_{n}(s,a)\|^{2}}{2\sigma^{2}}}\right|\\
    &\stackrel{\textrm{(i)}}{\leq} 2\cdot  \textrm{TV}(\pnsa,\psa)\\
    &\stackrel{\textrm{(ii)}}{\leq} 2 \sqrt{\textrm{KL}(\pnsa,\psa)/2}\\
    &\stackrel{\textrm{(iii)}}{\leq} 2 \sqrt{\|f(s,a)-\hat{f}_n(s,a)\|^{2}/4\sigma^{2}}\\
    &\leq  \|f(s,a)-\hat{f}_n(s,a)\|/\sigma,
\end{align*}
where (i) follows from the definition of Total Variation (TV) distance between any two multivariate Gaussians, (ii) uses the Pinsker's inequality, (iii) uses the formula for KL-divergence between multivariate Gaussian distributions.

\end{proof}

%% file: supplementary/13_limit-bound-kl.tex
\begin{lemma}\label{lemma: limit bound} (Proposition-2 in \cite{hu2013kullback})
For any function $V(\cdot): \cS\to [0,1/(1-\gamma)] $ and random variable $Y=V(s')$ for $s'\sim P_{f}(s,a)$, we have 
\begin{equation}
\lim_{\alpha\to 0}[-\alpha \log(\E_{s'\sim P_{f}(s,a)}[e^{\frac{-V(s')}{\alpha}}])-\alpha \rho]=\mathrm{ESI}(Y),     \nonumber
\end{equation}
where $    \mathrm{ESI}(Y)=\sup\{t\in \mathbb{R}: \pr\{Y<t\}=0\}$ (essential infimum).
\end{lemma}
\begin{proof}
Consider the case when $M>\mathrm{ESI}(Y)$. Let $\kappa_{M}=\pr(V(s')\leq M)=\int_{s'}\mathbbm{1}(V(s')\leq M)e^{-\frac{\|s'-f(s,a)\|^{2}}{\sigma^{2}}}$. It holds that
\begin{align}
    &-\alpha \log\Big(\E_{s'\sim P_{f}(s,a)}[e^{\frac{-V(s')}{\alpha}}]\Big)\\&= -\alpha \log\Big(\E_{s'\sim P_{f}(s,a)}[\mathbbm{1}(V(s')\leq M)e^{\frac{-V(s')}{\alpha}}+\mathbbm{1}(V(s')>M)e^{\frac{-V(s')}{\alpha}}]\Big)  \nonumber\\
    &\leq  -\alpha \log\Big(\E_{s'\sim P_{f}(s,a)}[\mathbbm{1}(V(s')\leq M)e^{\frac{-V(s')}{\alpha}}]\Big)  \nonumber\\
    &\leq  -\alpha \log\Big(\E_{s'\sim P_{f}(s,a)}[\mathbbm{1}(V(s')\leq M)e^{\frac{-M}{\alpha}}]\Big)  \nonumber\\
    &\leq  -\alpha \log\Big(\kappa_{M}e^{\frac{-M}{\alpha}}]\Big)  \nonumber\\
    \label{eq: limit bound}
    &=M-\alpha \log(\kappa_{M}).
\end{align}
Thus for any $M>\mathrm{ESI}(Y)$, we have 
$$\lim_{\alpha\to 0}[\{-\alpha \log(\E_{s'\sim P_{f}(s,a)}[e^{\frac{-V(s')}{\alpha}}])-\alpha \rho]\leq M.$$ Combining with the fact that $\lim_{\alpha\to 0}[\{-\alpha \log(\E_{s'\sim P_{f}(s,a)}[e^{\frac{-V(s')}{\alpha}}])-\alpha \rho]\geq \mathrm{ESI}(Y)$, we get the desired result.  
\end{proof}

%% file: supplementary/14_cover-kl-proof-no-pol.tex
\begin{lemma}{($\zeta-$cover construction)}\label{lemma: kl-cover-no-pol}
For $\mathcal{V}$ denoting the set of value functions from $\cS\to[0,1/(1-\gamma)]$, $\overline{\alpha}=M/\rho$, $\underline{\alpha}$ as defined in \Cref{lemma: diff-opt} we have with probability at least $1-\delta$,

\begin{multline}
    \max_{V\in \mathcal{V}}\max_{s,a}2\overline{\alpha} e^{\frac{M}{\underline{\alpha}}}\max_{\alpha\in[\underline{\alpha},\overline{\alpha}]}|\E_{s'\sim P_{\hat{f}_{n}}(s,a)}[e^{\frac{-V(s')}{\alpha}}]-\E_{s'\sim P_{f}(s,a)}[e^{\frac{-V(s')}{\alpha}}]|\\
     \leq\mathcal{O}\Big(2(\frac{M}{\rho}) e^{\frac{M}{\alpha_{kl}}}e^{\frac{\zeta}{\alpha_{kl}}}\frac{ \beta_n(\delta)\sqrt{2ed^{2}\Gamma_{nd}}}{\sqrt{n}}\Big).  \nonumber
\end{multline}

\end{lemma}
\begin{proof}
  
 Let $\mathcal{N}_{\mathcal{V}}(\zeta)$ be the $\zeta-$ cover of the set $\mathcal{V}$. By definition, there exists $V'\in \mathcal{N}_{\V}(\zeta)$ such that $\|V'-V\|\leq \zeta$ for every $V\in \V$.
\begin{align}
     &|\E_{s'\sim P_{\hat{f}_{n}}(s,a)}[e^{-V(s')/\alpha}]-\E_{s'\sim P_{f}(s,a)}[e^{-V(s')/\alpha}]| \nonumber\\&\leq |\int_{\mathbb{R}^{d}}\frac{1}{\sqrt{(2\pi\sigma^{2})^{d}}}e^{\frac{-V(s')}{\alpha}}(e^{-\frac{\|s'-f(s,a)\|^{2}}{\sigma^{2}}}-e^{-\frac{\|s'-\hat{f}_{n}(s,a)\|^{2}}{\sigma^{2}}})| \nonumber\\
    &\leq  \int_{\mathbb{R}^{d}}\frac{1}{\sqrt{(2\pi\sigma^{2})^{d}}}e^{\frac{-V(s')}{\alpha}}|e^{-\frac{\|s'-f(s,a)\|^{2}}{\sigma^{2}}}-e^{-\frac{\|s'-\hat{f}_{n}(s,a)\|^{2}}{\sigma^{2}}}| \nonumber\\
    &\leq  \int_{\mathbb{R}^{d}}\frac{1}{\sqrt{(2\pi\sigma^{2})^{d}}}e^{\frac{-V(s')+V'(s')}{\alpha}}e^{\frac{-V'(s')}{\alpha}}|e^{-\frac{\|s'-f(s,a)\|^{2}}{\sigma^{2}}}-e^{-\frac{\|s'-\hat{f}_{n}(s,a)\|^{2}}{\sigma^{2}}}| \nonumber\\
    &\stackrel{(i)}{\leq}  e^{\frac{\zeta}{\alpha_{kl}}}\int_{\mathbb{R}^{d}}\frac{1}{\sqrt{(2\pi\sigma^{2})^{d}}}e^{\frac{-V'(s')}{\alpha}}|e^{-\frac{\|s'-f(s,a)\|^{2}}{\sigma^{2}}}-e^{-\frac{\|s'-\hat{f}_{n}(s,a)\|^{2}}{\sigma^{2}}}| \nonumber\\\label{eq: V-cover bound kl}
    &\leq  \max_{V'\in \mathcal{N}_{\mathcal{V}}(\zeta)}\max_{s,a}\max_{\alpha\in[\alpha_{kl},\overline{\alpha}]}e^{\frac{\zeta}{\alpha_{kl}}}\int_{\mathbb{R}^{d}}\frac{1}{\sqrt{(2\pi\sigma^{2})^{d}}}e^{\frac{-V'(s')}{\alpha}}|e^{-\frac{\|s'-f(s,a)\|^{2}}{\sigma^{2}}}-e^{-\frac{\|s'-\hat{f}_{n}(s,a)\|^{2}}{\sigma^{2}}}|.
\end{align}
Here (i) is obtained using the fact that  $\|V'-V\|\leq \zeta$ and $\alpha_{kl}$ is the minimum value of $\underline{\alpha}$ as defined in \Cref{lemma: diff-opt}. Using \Cref{eq: V-cover bound kl}, we bound uniformly over all $V\in \mathcal{V}$, we have
\begin{align}
     &\max_{V\in \mathcal{V}}\max_{s,a}2\overline{\alpha} e^{\frac{M}{\alpha_{kl}}}\max_{\alpha\in[\alpha_{kl},\overline{\alpha}]}|\E_{s'\sim P_{\hat{f}_{n}}(s,a)}[e^{\frac{-V(s')}{\alpha}}]-\E_{s'\sim P_{f}(s,a)}[e^{\frac{-V(s')}{\alpha}}]|  \nonumber\\
     &\leq \max_{V'\in \mathcal{N}_{\mathcal{V}}(\zeta)}\max_{s,a}\max_{\alpha\in[\alpha_{kl},\overline{\alpha}]}2\overline{\alpha} e^{\frac{M}{\alpha_{kl}}}e^{\frac{\zeta}{\alpha_{kl}}}\int_{\mathbb{R}^{d}}\frac{1}{\sqrt{(2\pi\sigma^{2})^{d}}}e^{\frac{-V'(s')}{\alpha}}|e^{-\frac{\|s'-f(s,a)\|^{2}}{\sigma^{2}}}-e^{-\frac{\|s'-\hat{f}_{n}(s,a)\|^{2}}{\sigma^{2}}}| \nonumber\\\label{eq: kl-cover-final}
      &\leq \max_{V'\in \mathcal{N}_{\mathcal{V}}(\zeta)}\max_{s,a}\max_{\alpha\in[\alpha_{kl},\overline{\alpha}]}2\overline{\alpha} e^{\frac{M}{\alpha_{kl}}}e^{\frac{\zeta}{\alpha_{kl}}}\int_{\mathbb{R}^{d}}\frac{1}{\sqrt{(2\pi\sigma^{2})^{d}}}|e^{-\frac{\|s'-f(s,a)\|^{2}}{\sigma^{2}}}-e^{-\frac{\|s'-\hat{f}_{n}(s,a)\|^{2}}{\sigma^{2}}}|\\
       &\stackrel{(i)}{\leq} \max_{s,a}4\overline{\alpha}\sigma^{-1}e^{\frac{M}{\alpha_{kl}}}e^{\frac{\zeta}{\alpha_{kl}}}\|f(s,a)-\hat{f}_n(s,a)\|  \nonumber\\
         &\stackrel{(ii)}{\leq}\mathcal{O}\Big(2(\frac{M}{\rho}) e^{\frac{M}{\alpha_{kl}}}e^{\frac{\zeta}{\alpha_{kl}}}\frac{ \beta_n(\delta)\sqrt{2ed^{2}\Gamma_{nd}}}{\sigma\sqrt{n}}\Big)  \nonumber
\end{align}
 Here (i) follows from \Cref{lemma: error bound} and by the fact that none of the remaining terms inside $\max$ depend on $V'$ or $\alpha$. And (ii) follows from $\overline{\alpha}=\frac{M}{\rho}$ and \Cref{eq: learning error bound}.

\end{proof}

%% file: supplementary/15_other_uncertainty.tex
\section{Other Uncertainty Sets}
\label{appendix:other_uncertainty}
\subsection{$\chi^2$ Uncertainty Set}
\label{sec: chi-square uncertainty}
The f-divergence (\citet{ali1966general,csiszar1967information}) between probability measures $P$ and $P_{0}$ defined over $\mathcal{X}$ for a convex function $f:\rp\to\bar{\rp}_{+}=\rp_{+}\cup\{\infty\}$ satisfying $f(1)=0$ and $f(t)=\infty$ for any $t<0$ is defined as follows:
\begin{equation}\label{eq:f-divergence}
    D_{f}(P||P_{0})=\int f\Big(\frac{dP}{dP_{0}}\Big)dP_{0}.
\end{equation}
Specifically \citet{duchi2018learning} considers the Cressie-Read family of f-divergences (\citet{cressie1984multinomial}, see \Cref{sec: chi-square uncertainty}) which includes $\chi^{2}$ divergence ($k=2$), etc. This family of f-divergences can be parametrized by $k\in (-\infty,\infty)\backslash \{0,1\}$ with $
    f_{k}(t):=\frac{t^{k}-kt+k-1}{k(k-1)}$  

Using this, we state the reformulation result from \citet[Lemma-1]{duchi2018learning}.
\begin{lemma}\label{lemma: f-div reform}
For $k\in(1,\infty)$, $k_{*}=k/k-1$, any $\rho>0$ and $c_{k}(\rho)=(1+k(k-1)\rho)^{\frac{1}{k}}$ and $X\sim P_{0}$ where $P_{0}$ is any probability distribution over $\mathcal{X}$ with $H:\mathcal{X}\to\mathbb{R}$, we have 
\begin{equation}\label{eq: divergence transformation}
    \sup_{P:D_{f_{k}}(P||P_{0})\leq \rho}\E_{P}[H(X)]=\inf_{\eta\in \rp }\{c_{k}(\rho)(\E_{P_{0}}[(H(X)-\eta)_{+}^{k_{*}}])^{\frac{1}{k_{*}}}+\eta\}.
\end{equation}
\end{lemma}

\begin{theorem}(Sample Complexity under $\chi^2$ uncertainty set)\label{thm: chi sample supp} Consider a robust MDP (see \Cref{sec: Problem Formulation}) with nominal transition dynamics $f$ and uncertainty set defined as in \Cref{eq: uncertainty set} w.r.t.~$\chi^2$ divergence. For $\pi^{*}$ denoting the robust optimal policy w.r.t.~nominal transition dynamics $f$ and  $\pi_{N}^{*}$ denoting the robust optimal policy w.r.t.~learned nominal transition dynamics $\hat{f}_N$ via \Cref{alg: mgpbo}, and $\delta\in(0, 1)$, $\epsilon\in (0,\frac{1}{1-\gamma})$,it holds that $\max_{s}|V^{\text{R}}_{\pi_N^{*},f}(s)-V^{\text{R}}_{\pi^{*},f}(s)|\leq \epsilon$ with probability
at least $1 - \delta$ for any $N \geq  N_{\chi^2}$, where

\begin{equation}
    N_{\chi^2}= \mathcal{O}\Big(\Big(\frac{1+2\rho}{\sqrt{1+2\rho}-1}\Big)^{4}\frac{ \gamma^{4}\beta_n(\delta)^{2}d^{2}\gamma_{nd}}{\epsilon^{4}(1-\gamma)^{8}}\Big). 
\end{equation} 
\end{theorem}
\begin{proof}
\textbf{Step (i):}
As detailed in the proof outline of \Cref{sec:sample_complexity}, in order to bound $V^{\text{R}}_{\hpis,f}(s)-V^{\text{R}}_{\pi^* ,f}(s)$, we begin by adding and subtracting $V^{\text{R}}_{\hpis,\hat{f}_{n}}(s)$ which is the robust value function w.r.t.~the nominal transition dynamics $\hat{f}_{n}$ and its corresponding optimal policy $\hpis$. Then, we split the difference into two terms as follows:
\begin{equation}\label{eq: bound separation chi}
   V^{\text{R}}_{\hpis,f}(s)-V^{\text{R}}_{\pi^*,f}(s)= \underbrace{V^{\text{R}}_{\hpis,f}(s)-V^{\text{R}}_{\hpis,\hat{f}_{n}}(s)}_{(i)}+\underbrace{V^{\text{R}}_{\hpis,\hat{f}_{n}}(s)-V^{\text{R}}_{\pi^{*},f}(s)}_{(ii)}.
\end{equation}
In order to not disturb the flow of the proof we bound (i) and (ii) separately \Cref{lemma: (i) bound lemma} and \Cref{lemma: (ii) bound lemma} respectively. From \Cref{lemma: (i) bound lemma}, we obtain that 
\begin{align}\label{eq: (i) reusage ch}
    (i)&\leq\max_{s}\Big|V^{\text{R}}_{\hpis,f}(s)-V^{\text{R}}_{\hpis,\hat{f}_{n}}(s)\Big|\nonumber\\&\leq \frac{\gamma}{1-\gamma}\max_{s}\Big|\inf_{\chi^2(p||P_{f}(s,\hpis(s)))\leq \rho}\E_{s'\sim p}\Big[V^{\text{R}}_{\hpis,f}(s')\Big]-\inf_{\chi^2(p||P_{\hat{f}_{n}}(s,\hpis(s)))\leq \rho}\E_{s'\sim p}\Big[V^{\text{R}}_{\hpis,f}(s')\Big]\Big|.
\end{align}
And from \Cref{lemma: (ii) bound lemma}, we obtain that 
\begin{align}\label{eq: (ii) reusage ch}
    (ii)&\leq\max_{s}\Big|V^{\text{R}}_{\hpis,\hat{f}_n}(s)-V^{\text{R}}_{\pi^{*},f}(s)\Big|\nonumber \\&\leq \frac{\gamma}{1-\gamma}\max_{s}\Big|\inf_{\chi^2(p||P_{\hat{f}_{n}}(s,\hpis(s)))\leq \rho}\E_{s'\sim p}\Big[V^{\text{R}}_{\pi^{*},f}(s')\Big]-\inf_{\chi^2(p||P_{f}(s,\hpis(s)))\leq \rho}\E_{s'\sim p}\Big[V^{\text{R}}_{\pi^{*},f}(s')\Big]\Big|.
\end{align}
Note that both these terms in \Cref{eq: (i) reusage ch,eq: (ii) reusage ch} are of the form mentioned in the \textbf{Step (i)} of \Cref{sec:sample_complexity}.

\textbf{Step (ii):} Next, corresponding to \textbf{Step (ii)} of the proof outline in \Cref{sec:sample_complexity}, we use \Cref{lemma: f-div reform} to bound \Cref{eq: (i) reusage ch,eq: (ii) reusage ch}.
Denote $M:= \frac{1}{1-\gamma}\geq \max_{s}V^{\text{R}}_{\pi}(s)$ and $c_{2}(\rho):=\sqrt{1+2\rho}$ for convenience. 
Using \Cref{eq: (i) reusage ch} and \Cref{lemma: diff-opt-ch} (internally using \Cref{lemma: f-div reform}), it holds that

\begin{align}
    (i)&\leq\max_{s}\Big|V^{\text{R}}_{\hpis,f}(s)-V^{\text{R}}_{\hpis,\hat{f}_{n}}(s)\Big|\nonumber\\&\leq\frac{1}{1-\gamma}\max_{s}\Big|\gamma\inf_{\chi^2(p||P_{f}(s,\hpis(s)))\leq \rho}\E_{s'\sim p}\Big[V^{\text{R}}_{\hpis,f}(s')\Big]-\gamma\inf_{\chi^2(p||P_{\hat{f}_{n}}(s,\hpis(s)))\leq \rho}\E_{s'\sim p}\Big[V^{\text{R}}_{\hpis,f}(s')\Big]\Big| \nonumber \\
    \label{eq: diff-opt lemma usage chi}
    &\leq \max_{s,a}\Bigg(\frac{\gamma\sqrt{1+2\rho}}{1-\gamma}\sup_{\eta\in [0,\frac{c_{2}(\rho)M}{c_{2}(\rho)-1}]}\Big\{\Big|\E_{\psa}[(-V^{\text{R}}_{\hpis,f}(s')+\eta)_{+}^{2}]-\E_{\pnsa}[(-V^{\text{R}}_{\hpis,f}(s')+\eta)_{+}^{2}]\Big|^{\frac{1}{2}}\Big\}\Bigg).\\\label{eq: diff-opt lemma usage overall-ch-1}
    &\leq \max_{V(\cdot)\in\mathcal{V}}\max_{s,a}\Bigg(\frac{\gamma\sqrt{1+2\rho}}{1-\gamma}\sup_{\eta\in [0,\frac{c_{2}(\rho)M}{c_{2}(\rho)-1}]}\Big\{\Big|\E_{\psa}[(-V(s')+\eta)_{+}^{2}]-\E_{\pnsa}[(-V(s')+\eta)_{+}^{2}]\Big|^{\frac{1}{2}}\Big\}\Bigg).
\end{align}

We can bound (ii) similarly.
\begin{align}
    (ii)&\leq\max_{s}\Big|V^{\text{R}}_{\hpis,\hat{f}_n}(s)-V^{\text{R}}_{\pi^{*},f}(s)\Big|\\\label{eq: diff-opt lemma usage overall-ch-2}
    &\leq \max_{V(\cdot)\in\mathcal{V}}\max_{s,a}\Bigg(\frac{\gamma\sqrt{1+2\rho}}{1-\gamma}\sup_{\eta\in [0,\frac{c_{2}(\rho)M}{c_{2}(\rho)-1}]}\Big\{\Big|\E_{\psa}[(-V(s')+\eta)_{+}^{2}]-\E_{\pnsa}[(-V(s')+\eta)_{+}^{2}]\Big|^{\frac{1}{2}}\Big\}\Bigg).
\end{align}
\textbf{Step (iii):} Next, we want to utilize the learning error bound (\Cref{eq: learning error bound}) that bounds the difference between the means of true nominal transition dynamics $P_{f}$ and learned nominal transition dynamics $P_{\hat{f}_n}$ to bound \Cref{eq: diff-opt lemma usage overall-ch-1,eq: diff-opt lemma usage overall-ch-2}.

We begin by bounding the difference $\Big|\E_{\psa}[(-V(s')+\eta)_{+}^{2}]-\E_{\pnsa}[(-V(s')+\eta)_{+}^{2}]\Big|$, by the difference in means of $P_{f}$ and $P_{\hat{f}_n}$ in \Cref{lemma: error bound-ch}. Since \Cref{eq: diff-opt lemma usage overall-ch-1} has a $\max$ over all value functions, we introduce a covering number argument in \Cref{lemma: ch-cover-no-pol} to reform it to a $\max$ over the functions in the $\zeta-$covering set.  We then use \Cref{lemma: error bound-ch}  to obtain bounds in terms of maximum information gain $\Gamma_{Nd}$ (\Cref{eq: max info for gaus}) and $\zeta$. Further details regarding the covering number argument are deferred to \Cref{lemma: ch-cover-no-pol}. Then, we apply the result of \Cref{lemma: ch-cover-no-pol} with $\zeta=1$ (defined in \Cref{lemma: ch-cover-no-pol}) on \Cref{eq: diff-opt lemma usage overall-ch-1}. Then, it holds that
\begin{align}\label{eq: (i) final bound chi}
  (i) \leq \max_{s}\Big|V^{\text{R}}_{\hpis,f}(s)-V^{\text{R}}_{\hpis,\hat{f}_{n}}(s)\Big|&
= \mathcal{O}\Bigg(\Big(\frac{\gamma(c_{2}(\rho))^{2}M^{2}}{c_{2}(\rho)-1}\Big)\Big(\frac{ \beta_n(\delta)\sqrt{2ed^{2}\gamma_{nd}}}{\sigma\sqrt{n}}\Big)^{\frac{1}{2}}\Bigg).
\end{align}

Note that $\beta_n$, which appears in \Cref{lemma: vakili beta}, has a logarithmic dependence on $n$. Similarly, from \Cref{eq: diff-opt lemma usage overall-ch-2}, and \Cref{lemma: error bound-ch,lemma: ch-cover-no-pol}, we obtain 
\begin{align}\label{eq: (ii) final bound chi}
    (ii)\leq\max_{s}\Big|V^{\text{R}}_{\hpis,\hat{f}_n}(s)-V^{\text{R}}_{\pi^{*},f}(s)\Big|
    &=\mathcal{O}\Bigg(\Big(\frac{\gamma(c_{2}(\rho))^{2}M^{2}}{c_{2}(\rho)-1}\Big)\Big(\frac{ \beta_n(\delta)\sqrt{2ed^{2}\gamma_{nd}}}{\sigma\sqrt{n}}\Big)^{\frac{1}{2}}\Bigg).  
\end{align}
Note that we want to bound $V^{\text{R}}_{\hpis,f}(s)-V^{\text{R}}_{\pi^{*},f}(s)=(i)+(ii)$ over all $s\in \cS$. Using $ \max_{s}\Big|V^{\text{R}}_{\hpis,f}(s)-V^{\text{R}}_{\pi^{*},f}(s)\Big|\leq  \max_{s}\Big|V^{\text{R}}_{\hpis,\hat{f}_n}(s)-V^{\text{R}}_{\pi^{*},f}(s)\Big|+  \max_{s}\Big|V^{\text{R}}_{\hpis,\hat{f}_n}(s)-V^{\text{R}}_{\pi_n^{*},f}(s)\Big|$ and substituting $M$ by $1/(1-\gamma)$, we obtain from \Cref{eq: (i) final bound chi} and \Cref{eq: (ii) final bound chi}
\begin{align}
    \max_{s}\Big|V^{\text{R}}_{\hpis,f}(s)-V^{\text{R}}_{\pi^{*},f}(s)\Big| 
     &=\mathcal{O}\Bigg(\Big(\frac{\gamma(c_{2}(\rho))^{2}M^{2}}{c_{2}(\rho)-1}\Big)\Big(\frac{ \beta_n(\delta)\sqrt{2ed^{2}\gamma_{nd}}}{\sigma\sqrt{n}}\Big)^{\frac{1}{2}}\Bigg).  \nonumber
\end{align}

Finally,  to ensure that $\max_{s}|V^{\text{R}}_{\hpis,f}(s)-V^{\text{R}}_{\pi^{*},f}(s)|\leq \epsilon$ , it suffices to have
\begin{align}
     \max_{s}\Big|V^{\text{R}}_{\hpis,f}(s)-V^{\text{R}}_{\pi^{*},f}(s)\Big| 
     =\mathcal{O}\Bigg(\Big(\frac{\gamma(c_{2}(\rho))^{2}M^{2}}{c_{2}(\rho)-1}\Big)\Big(\frac{ \beta_n(\delta)\sqrt{2ed^{2}\gamma_{nd}}}{\sigma\sqrt{n}}\Big)^{\frac{1}{2}}\Bigg)  
     &= \epsilon\nonumber.
\end{align}
Moving $\sqrt{n}$ and $\epsilon$ to opposite sides and squaring both sides twice, we obtain
\begin{equation}
     n= \mathcal{O}\Big(\Big(\frac{1+2\rho}{\sqrt{1+2\rho}-1}\Big)^{4}\frac{ \gamma^{4}\beta_n(\delta)^{2}d^{2}\gamma_{nd}}{\sigma^2\epsilon^{4}(1-\gamma)^{8}}\Big).\nonumber\\
\end{equation}

\end{proof}

%% file: supplementary/16_diff-opt-ch.tex
\begin{lemma}\label{lemma: diff-opt-ch} (Simplification using \Cref{lemma: f-div reform} reformulation)
For any value function $V$ from $ \cS\to [0,1/(1-\gamma)] $, it holds that
\begin{multline}
    \max_{s}\Big|\inf_{\chi^2(p||P_{\hat{f}_{n}}(s,\hat{\pi}_{n}(s)))\leq \rho}\E_{s'\sim p}\Big[V(s')\Big]-\inf_{\chi^2(p||P_{f}(s,\hat{\pi}_{n}(s)))\leq \rho}\E_{s'\sim p}\Big[V(s')\Big]\Big|\leq\\ \max_{s,a}c_{2}(\rho)\sup_{\eta\in [0,\frac{c_{2}(\rho)M}{c_{2}(\rho)-1}]}\{|\E_{\psa}[(-V(s')+\eta)_{+}^{2}]-\E_{\pnsa}[(-V(s')+\eta)_{+}^{2}]|^{\frac{1}{2}}\},
\end{multline}
where $c_{2}(\rho)=\sqrt{1+2\rho}$ and $M=1/(1-\gamma)$.
\end{lemma}
\begin{proof}
First note that,
\begin{multline}\label{eq: stosa-ch}
    \max_{s}|\inf_{\chi^2(p||P_{\hat{f}_{n}}(s,\hat{\pi}_{n}(s)))\leq \rho}\E_{s'\sim p}\Big[V(s')\Big]-\inf_{\chi^2(p||P_{f}(s,\hat{\pi}_{n}(s)))\leq \rho}\E_{s'\sim p}\Big[V(s')\Big]|\leq\\ \max_{s,a}\Big|\inf_{\chi^2(p||P_{\hat{f}_{n}}(s,a))\leq \rho}\E_{s'\sim p}\Big[V(s')\Big]-\inf_{\chi^2(p||P_{f}(s,a))\leq \rho}\E_{s'\sim p}\Big[V(s')\Big]\Big|.
\end{multline}
Using \Cref{lemma: f-div reform} and focusing to bound right side of \Cref{eq: stosa-ch} for one particular $(s,a)$ state-action pair, we obtain  
\begin{multline}\label{eq: ch-sq reformulation}
    \Big|\inf_{\chi^2(p||P_{\hat{f}_{n}}(s,a))\leq \rho}\E_{s'\sim p}\Big[V(s')\Big]-\inf_{\chi^2(p||P_{f}(s,a))\leq \rho}\E_{s'\sim p}\Big[V(s')\Big]\Big|=\\ \Big|\sup_{\eta\in \rp}\{-c_{2}(\rho)(\E_{\psa}[(-V(s')-\eta)_{+}^{2}])^{\frac{1}{2}}-\eta\}-\sup_{\eta\in \rp}\{-c_{2}(\rho)(\E_{\pnsa}[(-V(s')-\eta)_{+}^{2}])^{\frac{1}{2}}-\eta\}\Big|\\
    \stackrel{(i)}{=}\Big|\sup_{\eta\in \rp}\{-c_{2}(\rho)(\E_{\psa}[(-V(s')+\eta)_{+}^{2}])^{\frac{1}{2}}+\eta\}-\sup_{\eta\in \rp}\{-c_{2}(\rho)(\E_{\pnsa}[(-V(s')+\eta)_{+}^{2}])^{\frac{1}{2}}+\eta\}\Big|,
\end{multline}
where (i) is obtained by replacing $\eta$ with $-\eta$. \\
Let $g_{\chi^2}(\eta,\psa):=\Big(-c_{2}(\rho)(\E_{\psa}[(-V(s')+\eta)_{+}^{2}])^{\frac{1}{2}}+\eta\Big)$. Note that $g_{\chi^2}(\eta,\psa)$ satisfies the following:
For $\eta\leq 0$ (implying $(-V(s')+\eta)\leq 0$ and $(-V(s')+\eta)_{+}= 0$),
\begin{equation}\label{eq: lower bound}
    g_{\chi^2}(\eta,\psa)=\eta \leq 0.
\end{equation}
And for $\eta=\frac{c_{2}(\rho)M}{c_{2}(\rho)-1}>0$,
\begin{align}
    g_{\chi}(\tfrac{c_{2}(\rho)M}{c_{2}(\rho)-1},\psa)&=-c_{2}(\rho)(\E_{\psa}[(-V(s')+\tfrac{c_{2}(\rho)M}{c_{2}(\rho)-1})_{+}^{2}])^{\frac{1}{2}}+\tfrac{c_{2}(\rho)M}{c_{2}(\rho)-1}\nonumber\\
    &\stackrel{(i)}{\leq} \tfrac{c_{2}(\rho)M}{c_{2}(\rho)-1}-c_{2}(\rho)(\E_{\psa}[(-M+\tfrac{c_{2}(\rho)M}{c_{2}(\rho)-1})_{+}^{2}])^{\frac{1}{2}}\nonumber\\
     &\leq \tfrac{c_{2}(\rho)M}{c_{2}(\rho)-1}-c_{2}(\rho)(\E_{\psa}[(\tfrac{M}{c_{2}(\rho)-1})_{+}^{2}])^{\frac{1}{2}}\nonumber\\
    &\leq \tfrac{c_{2}(\rho)M}{c_{2}(\rho)-1}-\tfrac{c_{2}(\rho)M}{c_{2}(\rho)-1}\nonumber\\
    \label{eq: upper bound}
    &=0,
\end{align}
where (i) follows from the fact that the random variable $V(s')$ is bounded by $M=1/1-\gamma$. A similar result can be shown for  $g_{\chi^2}(\eta,\pnsa)$ (or for any P). Along with the convexity of $\eta\to g_{\chi}(\eta,P)$ (\citet{duchi2018learning}), and $\inf_{\chi^2(p||P)\leq \rho}\E_{s'\sim p}\Big[V(s')\Big]\geq 0$, \Cref{eq: lower bound} and \Cref{eq: upper bound} imply that the sup is attained between $[0,\frac{c_{2}(\rho)M}{c_{2}(\rho)-1}]$ for both $\sup_{\eta\in\rp}{g_{\chi}(\eta,\psa)}$ and $\sup_{\eta\in\rp}{g_{\chi}(\eta,\pnsa)}$. Using this in \Cref{eq: ch-sq reformulation} we have,
\begin{align}
    &\Big|\sup_{\eta\in \rp}\{g_{\chi}(\eta,\psa)\}-\sup_{\eta\in \rp}\{g_{\chi}(\eta,\pnsa)\}\Big|\\&= \Big|\sup_{\eta\in [0,\frac{c_{2}(\rho)M}{c_{2}(\rho)-1}]}\{g(_{\chi}\eta,\psa
    )\}-\sup_{\eta\in [0,\frac{c_{2}(\rho)M}{c_{2}(\rho)-1}]}\{g_{\chi}(\eta,\pnsa)\}\}\Big|\\
 &\leq  \sup_{\eta\in [0,\frac{c_{2}(\rho)M}{c_{2}(\rho)-1}]}\{|g_{\chi}(\eta,\psa)-g_{\chi}(\eta,\pnsa)|\}\\
  &\leq  \sup_{\eta\in [0,\frac{c_{2}(\rho)M}{c_{2}(\rho)-1}]}\{|c_{2}(\rho)(\E_{\psa}[(-V(s')+\eta)_{+}^{2}])^{\frac{1}{2}}-c_{2}(\rho)\E_{\pnsa}[(-V(s')+\eta)_{+}^{2}])^{\frac{1}{2}}|\}\\
  &\leq  c_{2}(\rho)\sup_{\eta\in [0,\frac{c_{2}(\rho)M}{c_{2}(\rho)-1}]}\{|\E_{\psa}[(-V(s')+\eta)_{+}^{2}]-\E_{\pnsa}[(-V(s')+\eta)_{+}^{2}]|^{\frac{1}{2}}\}.
\end{align}
The last step is obtained using the basic inequality $|\sqrt{a}-\sqrt{b}|\leq \sqrt{|a-b|}$.

\end{proof}

%% file: supplementary/17_cover-ch-proof.tex
\begin{lemma}{($\zeta-$cover construction)}\label{lemma: ch-cover-no-pol}
For $\mathcal{V}$ denoting the set of value functions from $\cS\to[0,1/(1-\gamma)]$ it holds with probability at least $1-\delta$,

\begin{multline}
    \max_{V\in \V}\max_{s,a}\sup_{\eta\in [0,\frac{c_{2}(\rho)M}{c_{2}(\rho)-1}]}\{|\E_{\psa}[(-V(s')+\eta)_{+}^{2}]-\E_{\pnsa}[(-V(s')+\eta)_{+}^{2}]|^{\frac{1}{2}}\}\leq \\
 \mathcal{O}\Bigg(\Big(\frac{c_{2}(\rho)M}{c_{2}(\rho)-1}\Big)\Big(\frac{ \beta_n(\delta)\sqrt{2ed^{2}\gamma_{nd}}}{\sigma\sqrt{n}}\Big)^{\frac{1}{2}}\Bigg),
\end{multline}
where $c_{2}(\rho)=\sqrt{1+2\rho}$, $M=1/(1-\gamma)$.

\end{lemma}\looseness=-1
\begin{proof}
  
 Let $\mathcal{N}_{\mathcal{V}}(\zeta)$ be the $\zeta-$ cover of the set $\mathcal{V}$. By definition, there exists $V'\in \mathcal{N}_{\V}(\zeta)$ such that $\|V'-V\|\leq \zeta$ for every $V\in \V$. \looseness=-1
\begin{align}
\begin{split}
&|\E_{\psa}[(-V(s')+\eta)_{+}^{2}]-\E_{\pnsa}[(-V(s')+\eta)_{+}^{2}]|\\&\leq
|\E_{\psa}[(-V(s')+\eta)_{+}^{2}]-\E_{\psa}[(-V'(s')+\eta)_{+}^{2}]|\\&\quad+|\E_{\psa}[(-V'(s')+\eta)_{+}^{2}]-\E_{\pnsa}[(-V'(s')+\eta)_{+}^{2}]|\\&\quad+|\E_{\pnsa}[(-V'(s')+\eta)_{+}^{2}]-\E_{\pnsa}[(-V(s')+\eta)_{+}^{2}]|.
\end{split}\\\label{eq: ch-close-cover}
&\stackrel{(i)}{\leq} 4\|V'-V\|^{2}+4\eta \|V'-V\|+|\E_{\psa}[(-V'(s')+\eta)_{+}^{2}]-\E_{\pnsa}[(-V'(s')+\eta)_{+}^{2}]|,
\end{align}
where (i) follows from \Cref{lemma: function-lipschitz-ch}. Using \Cref{eq: ch-close-cover} we bound uniformly over all $V\in\V$,
\begin{align}
     &\max_{V\in \V}\max_{s,a}\sup_{\eta\in [0,\frac{c_{2}(\rho)M}{c_{2}(\rho)-1}]}\{|\E_{\psa}[(-V(s')+\eta)_{+}^{2}]-\E_{\pnsa}[(-V(s')+\eta)_{+}^{2}]|^{\frac{1}{2}}\}\\
     \begin{split}
          &\leq  \max_{V'\in \mathcal{N}_{\V}(\zeta)}\max_{s,a}\sup_{\eta\in [0,\frac{c_{2}(\rho)M}{c_{2}(\rho)-1}]}\Biggl\{\Bigl(4\|V'-V\|^{2}+4\eta \|V'-V\|+|\E_{\psa}[(-V'(s')+\eta)_{+}^{2}]\\
     &\qquad-\E_{\pnsa}[(-V'(s')+\eta)_{+}^{2}]|\Big)^{\frac{1}{2}}\Bigg\}\nonumber
     \end{split}\\
     \begin{split}
         &\stackrel{(ii)}{\leq}  \max_{V'\in \mathcal{N}_{\V}(\zeta)}\max_{s,a}\sup_{\eta\in [0,\frac{c_{2}(\rho)M}{c_{2}(\rho)-1}]}\Biggl\{\Big(\E_{\psa}[(-V'(s')+\eta)_{+}^{2}]-\E_{\pnsa}[(-V'(s')+\eta)_{+}^{2}]\Big)^{\frac{1}{2}}\Biggr\}\\&\qquad+\sqrt{4\zeta^{2}+4\zeta \tfrac{c_{2}(\rho)M}{c_{2}(\rho)-1}}\nonumber
     \end{split} \\
     \begin{split}
         &\stackrel{(iii)}{\leq}  \max_{V'\in \mathcal{N}_{\V}(\zeta)}\max_{s,a}\sup_{\eta\in [0,\frac{c_{2}(\rho)M}{c_{2}(\rho)-1}]}\Bigg\{ \Big(\tfrac{c_{2}(\rho)M}{c_{2}(\rho)-1}\Big)\sqrt{2\sigma^{-1}\|f(s,a)-\hat{f}_{n}(s,a)\|}\Bigg\}+\sqrt{4\zeta^{2}+4\zeta \tfrac{c_{2}(\rho)M}{c_{2}(\rho)-1}}\nonumber
     \end{split}\\
      &\stackrel{(iv)}{\leq} \mathcal{O}\Bigg(\Big(\tfrac{c_{2}(\rho)M}{c_{2}(\rho)-1}\Big)\Big(\tfrac{ \beta_n(\delta)\sqrt{2ed^{2}\gamma_{nd}}}{\sigma\sqrt{n}}\Big)^{\frac{1}{2}}\Bigg)+\sqrt{4\zeta^{2}+4\zeta \tfrac{c_{2}(\rho)M}{c_{2}(\rho)-1}}\\
      \label{eq: learn-error-usage-chi-cover}
      &\stackrel{(v)}{\leq} \mathcal{O}\Bigg(\Big(\tfrac{c_{2}(\rho)M}{c_{2}(\rho)-1}\Big)\Big(\tfrac{ \beta_n(\delta)\sqrt{2ed^{2}\gamma_{nd}}}{\sigma\sqrt{n}}\Big)^{\frac{1}{2}}\Bigg),
\end{align}
where (ii) follows from $\|V'-V\|\leq \zeta$ and $\eta\leq \frac{c_{2}(\rho)M}{c_{2}(\rho)-1}$, (iii) follows from \Cref{lemma: error bound-ch}, (iv) follows from \Cref{eq: learning error bound}, and (v) follows from substituing $\zeta=1$ (or any constant).

\end{proof}

%% file: supplementary/18_function_lipschitz_ch.tex
\begin{lemma}\label{lemma: function-lipschitz-ch}
 For any two value functions $V,V'$  from $\cS\to[0,1/(1-\gamma)]$, it holds that
   \begin{equation}
       \Big|\E_{\psa}[(-V'(s')+\eta)_{+}^{2}]-\E_{\psa}[(-V(s')+\eta)_{+}^{2}]\Big|\leq 2\|V'-V\|^{2}+2\eta \|V'-V\|.
   \end{equation} 
\end{lemma}
\begin{proof}
Let $p_{\psa}(\cdot)$ denote the probability density function of $\psa$. Then,
\begin{align}
    &\E_{\psa}[(-V'(s')+\eta)_{+}^{2}]-\E_{\psa}[(-V(s')+\eta)_{+}^{2}]\nonumber\\ &\leq \int\limits_{s'\sim\psa}\Big(\mathbbm{1}(V'(s')<\eta)(-V'(s')+\eta)^{2}-\mathbbm{1}(V(s')<\eta)(-V(s')+\eta)^{2}\Big)p_{\psa}(s')ds'.\nonumber\\
    \begin{split}  \label{eq: ch-lip-func-bound-1}
&\leq\underbrace{\int_{s'\sim\psa}\Big(\mathbbm{1}(V'(s')<\eta)-\mathbbm{1}(V(s')<\eta)\Big)(-V'(s')+\eta)^{2}p_{\psa}(s')ds'}_{(i)}\\
    &\quad+\underbrace{\int_{s'\sim\psa}\mathbbm{1}(V(s')<\eta)\Big((-V'(s')+\eta)^{2}-(-V(s')+\eta)^{2}\Big)p_{\psa}(s')ds'}_{(ii)}.
    \end{split}
\end{align}
where the last inequality is obtained by adding and subtracting $\mathbbm{1}(V(s')<\eta)(-V'(s')+\eta)^{2}$.

We begin by bounding (ii). We have,\looseness=-1
\begin{align}
     (ii)&=\int_{s'\sim\psa}\limits\mathbbm{1}(V(s')<\eta)\Big((-V'(s')+\eta)^{2}-(-V(s')+\eta)^{2}\Big)p_{\psa}(s')ds'\nonumber\\
    &=\int_{s'\sim\psa}\limits\mathbbm{1}(V(s')<\eta)\Big(-V'(s')+V(s')\Big)\Big(-V'(s')-V(s')+2\eta\Big)p_{\psa}(s')ds'\nonumber\\
    \begin{split}
    &\leq \int_{s'\sim\psa}\limits\mathbbm{1}(V(s')<\eta)\Big(\mathbbm{1}(V'(s')<\eta)+\mathbbm{1}(V'(s')\geq\eta)\Big)\Big(-V'(s')+V(s')\Big)\\&\qquad\Big(-V'(s')-V(s')+2\eta\Big)p_{\psa}(s')ds'
    \end{split}\nonumber\\
    \begin{split}\label{eq: ch-lip-func-bound-2}
    &\leq \underbrace{\int\mathbbm{1}(V(s'),V'(s')<\eta)(-V'(s')+V(s'))(-V'(s')-V(s')+2\eta)p_{\psa}(s')ds'}_{(ii-a)}\\&
    \quad+\underbrace{\int\limits\mathbbm{1}(V(s')<\eta\leq V'(s'))(-V'(s')+V(s'))(-V'(s')-V(s')+2\eta)p_{\psa}(s')ds'}_{(ii-b)}.
    \end{split}
\end{align}
Bounding $(ii-a)$ first, we have,\looseness=-1
\begin{align}
    (ii-a)&=\int\mathbbm{1}(V(s'),V'(s')<\eta)(-V'(s')+V(s'))(-V'(s')-V(s')+2\eta)p_{\psa}(s')ds'\nonumber\\\stackrel{(a)}{\leq} &
    \int\mathbbm{1}(V(s'),V'(s')<\eta)\Big|-V'(s')+V(s')\Big|(-V'(s')-V(s')+2\eta)p_{\psa}(s')ds'\nonumber\\
    \stackrel{(b)}{\leq} & 
    \int_{s'\sim\psa}\limits\mathbbm{1}(V(s'),V'(s')<\eta)\Big|-V'(s')+V(s')\Big|\Big(2\eta\Big)p_{\psa}(s')ds'\nonumber\\\label{eq: ch-lip-func-bound-3}
    \leq & 2\eta\|V'-V\|,
\end{align}
where (a) and (b) follows from $(-V'(s')-V(s')+2\eta)>0$ as $V(s'),V'(s')<\eta$. And $(ii-b)$ can be bounded as,
\begin{align}
    (ii-b)&=\int\limits\mathbbm{1}(V(s')<\eta\leq V'(s'))(-V'(s')+V(s'))(-V'(s')-V(s')+2\eta)p_{\psa}(s')ds'\nonumber\\
    &\leq    \int\mathbbm{1}(V(s')<\eta\leq V'(s'))\Big|-V'(s')+V(s')\Big|\Big|-V'(s')-V(s')+2\eta\Big|p_{\psa}(s')ds'\nonumber\\
    &\stackrel{(c)}{\leq}    \int\mathbbm{1}(V(s')<\eta\leq V'(s'))\Big|-V'(s')+V(s')\Big|\Big|-V(s')+V'(s')\Big|p_{\psa}(s')ds'\nonumber\\
    \label{eq: ch-lip-func-bound-4}
    &\leq    \int_{s'\sim\psa}\limits\mathbbm{1}(V(s')<\eta\leq V'(s'))\Big|-V'(s')+V(s')\Big|^{2}p_{\psa}(s')ds'\nonumber\\
    &\leq  \|V'-V\|^{2},
\end{align}
where (c) follows from $\eta\leq V'(s')$.
Bounding (i) similarly,

\begin{align}
    i&=\int_{s'\sim\psa}\Big(\mathbbm{1}(V'(s')<\eta)-\mathbbm{1}(V(s')<\eta)\Big)(-V'(s')+\eta)^{2}p_{\psa}(s')ds'\nonumber\\
&\leq\int_{s'\sim\psa}\Big(\mathbbm{1}(V'(s')<\eta\leq V(s'))\Big)(-V'(s')+\eta)^{2}p_{\psa}(s')ds'\nonumber\\
&\leq\int_{s'\sim\psa}\Big(\mathbbm{1}(V'(s')<\eta\leq V(s'))\Big)(-V'(s')+V(s'))^{2}p_{\psa}(s')ds'\nonumber\\
\label{eq: ch-lip-func-bound-5}
&\leq \|V'-V\|^{2}.
\end{align}
Using \Cref{eq: ch-lip-func-bound-1,eq: ch-lip-func-bound-2,eq: ch-lip-func-bound-3,eq: ch-lip-func-bound-4,eq: ch-lip-func-bound-5} we get the desired result.
\end{proof}

%% file: supplementary/19_error_bound_ch.tex
\begin{lemma}\label{lemma: error bound-ch}(Bound by difference between estimated model $\hat{f}_n$ and true $f$)
For any value function $V(s'): \cS\to [0,1/(1-\gamma)] $ and any $\alpha>0$, it holds that
\begin{equation*}
    |\E_{\psa}[(-V(s')+\eta)_{+}^{2}]-\E_{\pnsa}[(-V(s')+\eta)_{+}^{2}]|\leq 2\sigma^{-1}\Big(\frac{c_{2}(\rho)M}{c_{2}(\rho)-1}\Big)^{2}\|f(s,a)-\hat{f}_{n}(s,a)\|,
\end{equation*} 
where $\pnsa=\mathcal{N}(\hat{f}_{n}(s,a),\sigma^{2} I)$ and $\psa=\mathcal{N}(f(s,a),\sigma^{2} I)$, $\eta\in [0,\frac{c_{2}(\rho)M}{c_{2}(\rho)-1}]$, $c_{2}(\rho)=\sqrt{1+2\rho}$ and $M=1/(1-\gamma)$.\looseness=-1
\end{lemma}
\begin{proof}
\begin{align*}
    &\Big|\E_{\psa}[(-V(s')+\eta)_{+}^{2}]-\E_{\pnsa}[(-V(s')+\eta)_{+}^{2}]\Big|\\&=\Big|\int_{\mathbb{R}^{d}}\frac{1}{\sqrt{(2\pi\sigma^{2})^{d}}}(-V(s')+\eta)_{+}^{2}(e^{-\frac{\|x-f(s,a)\|^{2}}{2\sigma^{2}}}-e^{-\frac{\|x-\hat{f}_{n}(s,a)\|^{2}}{2\sigma^{2}}})\Big|\\
    &\leq  \int_{\mathbb{R}^{d}}\frac{1}{\sqrt{(2\pi\sigma^{2})^{d}}}(-V(s')+\eta)_{+}^{2}\Big|e^{-\frac{\|x-f(s,a)\|^{2}}{2\sigma^{2}}}-e^{-\frac{\|x-\hat{f}_{n}(s,a)\|^{2}}{2\sigma^{2}}}\Big|\\
    &\stackrel{(i)}{\leq}  \Big(\frac{c_{2}(\rho)M}{c_{2}(\rho)-1}\Big)^{2}\int_{\mathbb{R}^{d}}\frac{1}{\sqrt{(2\pi\sigma^{2})^{d}}}\Big|e^{-\frac{\|x-f(s,a)\|^{2}}{2\sigma^{2}}}-e^{-\frac{\|x-\hat{f}_{n}(s,a)\|^{2}}{2\sigma^{2}}}\Big|\\
     &\stackrel{\textrm{(ii)}}{\leq} 2\Big(\frac{c_{2}(\rho)M}{c_{2}(\rho)-1}\Big)^{2}\cdot  \textrm{TV}(\pnsa,\psa)\\
    &\stackrel{\textrm{(iii)}}{\leq} 2\Big(\frac{c_{2}(\rho)M}{c_{2}(\rho)-1}\Big)^{2} \sqrt{\textrm{KL}(\pnsa,\psa)/2}\\
    &\stackrel{\textrm{(iv)}}{\leq} 2 \Big(\frac{c_{2}(\rho)M}{c_{2}(\rho)-1}\Big)^{2}\sqrt{\|f(s,a)-\hat{f}_n(s,a)\|^{2}/4\sigma^{2}}\\
    &\leq\Big(\frac{c_{2}(\rho)M}{c_{2}(\rho)-1}\Big)^{2} \|f(s,a)-\hat{f}_n(s,a)\|/\sigma,
\end{align*}

where (i) follows from $(-V(s')+\eta)_{+}^{2}\leq \Big(\frac{c_{2}(\rho)M}{c_{2}(\rho)-1}\Big)^{2}$ as $\eta\leq \Big(\frac{c_{2}(\rho)M}{c_{2}(\rho)-1}\Big)$, (ii) follows from the definition of Total Variation (TV) distance between any two multivariate Gaussians, (iii) uses the Pinsker's inequality, and (iv) uses the formula for KL-divergence between multivariate Gaussian distributions.  \\

\end{proof}

%% file: supplementary/tv-formulation-proof.tex
\subsection{Total Variation Distance}\label{sec: tv distance proof}
Similar to \cref{lemma: f-div reform}, we want a similar convex reformulation for the variation distance. We derive such a reformulation starting from the dual reformulation from \cite{shapiro2017distributionally} and \cite{ben2013robust} stated as Proposition-1 in \cite{duchi2018learning}.

\begin{lemma}
For $X\sim P_{0}$ where $P_{0}$ is any probability distribution over $\mathcal{X}$ with $H:\mathcal{X}\to\mathbb{R}$ , $\rho>0$ and, $D_{f}(P||P_{0})$ defined as in \Cref{eq:f-divergence} , it holds that 
\begin{equation}\label{eq: general divergence transformation}
    \sup_{P:D_{f}(P||P_{0})\leq \rho}\E_{P}[H(X)]=\inf_{\lambda \geq 0, \eta\in \rp }\Big\{\E_{P_{0}}\Big[\lambda f^{*}\Big(\frac{H(X)-\eta}{\lambda}\Big)\Big]+\lambda \rho+\eta\Big\}.
\end{equation}
\end{lemma}
Note that the total variation distance between two probability distributions $P$ and $P_{0}$ is attained by substituting $f_{\mathrm{TV}}(t)=|t-1|$ in $ D_{f}(P||P_{0})=\int f\Big(\frac{dP}{dP_{0}}\Big)dP_{0}$. The corresponding Fenchel conjugate $f_{\mathrm{TV}}^{*}(s)$ for $f_{\mathrm{TV}}(t)=|t-1|$ would be    
\begin{equation}\label{eq: var-dis conjugate}
    f_{\mathrm{TV}}^{*}(s)=\begin{cases}
    -1,& s\leq -1\\
    s,&    s\in[-1,1]\\
    \infty,              & s>1
\end{cases}
\end{equation}
As we require $\inf_{P:\mathrm{TV}(P||P_{0})\leq \rho}\E_{P}[H(X)]$, using \Cref{eq: general divergence transformation} and replacing $\eta$ with $-\eta$, we have 
\begin{equation}\label{eq: inf general divergence transformation}
    \inf_{P:\mathrm{TV}(P||P_{0})\leq \rho}\E_{P}[H(X)]=\sup_{\lambda \geq 0, \eta\in \rp }\{-\E_{P_{0}}\Big[\lambda f_{\mathrm{TV}}^{*}\Big(\frac{-H(X)+\eta}{\lambda}\Big)\Big]-\lambda \rho+\eta\}.
\end{equation}
Using \Cref{eq: inf general divergence transformation}, we derive a convex reformulation in \Cref{lemma: var-dis-reform}
\begin{lemma}(Reformulation for total variation distance based on \citet{yang2021towards})\label{lemma: var-dis-reform}
For  $\rho>0$ and $X\sim P_{0}$ where $P_{0}$ is any probability distribution over $\mathcal{X}$ with $H:\mathcal{X}\to\mathbb{R}$,  for $0\leq H(x)\leq \frac{1}{1-\gamma}$ and $ESI(Y)=\sup\{t\in \mathbb{R}: \pr\{Y<t\}=0\}$ (essential infimum), it holds that
\begin{equation}
       \inf_{P:\mathrm{TV}(P||P_{0})\leq \rho}\E_{P}[H(X)]=
    \sup_{\eta\in [0,\frac{(2+\rho)}{\rho(1-\gamma)}]}\Big\{-\E_{P_{0}}[-H(X)+\eta]_{+}-  \frac{(-ESI(H(x))+\eta)_{+}}{2}\rho+\eta\Big\}.
\end{equation}
\end{lemma}
where $\mathrm{TV}$ denotes the total variation distance. 
\begin{proof}

Substituting \Cref{eq: var-dis conjugate} in \Cref{eq: inf general divergence transformation} to obtain the reformulation for total variation distance, we have
\begin{align}
    &\inf_{P:\mathrm{TV}(P||P_{0})\leq \rho}\E_{P}[H(X)]\\
    &=\sup_{\lambda \geq 0, \eta\in \rp,\frac{-H(x)+\eta}{\lambda}\leq 1 }\{-\E_{P_{0}}\Big[\lambda \max\Big\{\frac{-H(X)+\eta}{\lambda},-1\Big\}\Big]-\lambda \rho+\eta\}\\
    &=\sup_{\lambda \geq 0, \eta\in \rp,\frac{-H(x)+\eta}{\lambda}\leq 1 }\{-\E_{P_{0}}\Big[ \max\Big\{-H(X)+\eta,-\lambda\Big\}\Big]-\lambda \rho+\eta\}\\
    \label{eq: var-dis-reform-1}
    &=\sup_{\lambda \geq 0, \eta\in \rp,\frac{-H(x)+\eta}{\lambda}\leq 2 }\{-\E_{P_{0}}\Big[ \max\Big\{-H(X)+\eta-\lambda,-\lambda\Big\}\Big]-\lambda \rho+\eta-\lambda\}\\
    &=\sup_{\lambda \geq 0, \eta\in \rp,\frac{-H(x)+\eta}{\lambda}\leq 2 }\{-\E_{P_{0}}\Big[ \max\Big\{-H(X)+\eta,0\Big\}\Big]-\lambda \rho+\eta\}\\
    &=\sup_{\lambda \geq 0, \eta\in \rp,\frac{-H(x)+\eta}{\lambda}\leq 2 }\{-\E_{P_{0}}\Big[-H(X)+\eta\Big]_{+}-\lambda \rho+\eta\}.
\end{align}
Here \Cref{eq: var-dis-reform-1} is obtained by substituting $\eta$ with $\eta-\lambda$. In order to optimize over $\lambda$, we need to choose the minimum $\lambda$ satisfying the constraints. We require $\lambda \geq \frac{-H(x)+\eta}{2}$ which translates to  $\lambda \geq \frac{-ESI(H(x))+\eta}{2}$ (as this constraint originates inside the expectation, points with zero mass, $\{t\in \mathbb{R}: \pr\{Y<t\}=0\}$, will have no effect).

Substituting this, we have
\begin{align}\label{eq: var-dis reformulation}
    \inf_{P:\mathrm{TV}(P||P_{0})\leq \rho}\E_{P}[H(X)]&=
    \sup_{\eta\in \rp}\{-\E_{P_{0}}\Big[-H(X)+\eta\Big]_{+}-  \frac{(-ESI(H(x))+\eta)_{+}}{2}\rho+\eta\}.
\end{align}
Denote the inner function in \Cref{eq: var-dis reformulation}, as 
\begin{equation}\label{eq: g-var-dis}
    g_{\mathrm{TV}}(\eta,P_{0})=-\E_{P_{0}}\Big[-H(X)+\eta\Big]_{+}-  \frac{(-ESI(H(x))+\eta)_{+}}{2}\rho+\eta.
\end{equation}
Note that for $\eta\leq0$, the first two terms in $g_{\mathrm{TV}}(\eta,P_{0})$ will be $0$ if $H(x)>0$ for all $x$. This implies 
\begin{equation}\label{eq: g-var lower bound}
g_{\mathrm{TV}}(\eta,P_{0})=\eta\leq 0 \quad \forall\quad  \eta\leq 0.     
\end{equation}
Also, as $H(x)\leq\frac{1}{1-\gamma}$, we substitute $\eta=\frac{2+\rho}{\rho(1-\gamma)}$ in $g_{\mathrm{TV}}(\eta,P_0)$, and bound it as follows:
\begin{align}
    g_{\mathrm{TV}}\Big(\frac{(2+\rho)}{\rho(1-\gamma)},P_{0}\Big)&=-\E_{P_{0}}\Big[-H(X)+\frac{(2+\rho)}{\rho(1-\gamma)}\Big]_{+}-  \frac{(-ESI(H(x))+\frac{(2+\rho)}{\rho(1-\gamma)})_{+}}{2}\rho+\frac{(2+\rho)}{\rho(1-\gamma)}\\
    \label{eq: eta constraint-1}
    &=\E_{P_{0}}\Big[H(X)\Big]-\frac{(2+\rho)}{\rho(1-\gamma)}-  \frac{(-ESI(H(x))+\frac{(2+\rho)}{\rho(1-\gamma)})_{+}}{2}\rho+\frac{(2+\rho)}{\rho(1-\gamma)}\\
    &=\E_{P_{0}}\Big[H(X)\Big]-  \frac{(-ESI(H(x))+\frac{(2+\rho)}{\rho(1-\gamma)})_{+}}{2}\rho\\
    \label{eq: eta constraint-2}
    &=\E_{P_{0}}\Big[H(X)\Big]-  \frac{(-ESI(H(x))+\frac{(2+\rho)}{\rho(1-\gamma)})}{2}\rho\\
    &=\E_{P_{0}}\Big[H(X)-\frac{1}{1-\gamma}\Big]+ \frac{\rho ESI(H(x))}{2}-\frac{\rho}{2(1-\gamma)}\\
    &=\E_{P_{0}}\Big[H(X)-\frac{1}{1-\gamma}\Big]+ \frac{\rho}{2}( ESI(H(x))-\frac{1}{(1-\gamma)})\\
    \label{eq: g-var upper bound}
    &\leq 0.
\end{align}
Here \Cref{eq: eta constraint-1}, \Cref{eq: eta constraint-2} and \Cref{eq: g-var upper bound} are obtained from the fact that that $ H(x)\leq \frac{1}{1-\gamma}$ ($-H(x)+\frac{(2+\rho)}{\rho(1-\gamma)}>0$) and $ESI(H(x))\leq \frac{1}{1-\gamma}$ ($-ESI(H(x))+\frac{(2+\rho)}{\rho(1-\gamma)}>0$). Along with the convexity of $ g_{\mathrm{TV}}(\eta,P_{0})$, \Cref{eq: g-var lower bound} and \Cref{eq: g-var upper bound} imply that the $\sup_{\eta\in\rp}\{g_{\mathrm{TV}}(\eta,P_{0})\}$ is attained in the $\eta$ range $[0,\frac{(2+\rho)}{\rho(1-\gamma)}]$.
\end{proof}
\begin{theorem}(Sample Complexity under TV uncertainty set)\label{thm: tv sample supp} Consider a robust MDP (see \Cref{sec: Problem Formulation}) with nominal transition dynamics $f$ and uncertainty set defined as in \Cref{eq: uncertainty set} w.r.t.~TV distance. For $\pi^{*}$ denoting the robust optimal policy w.r.t.~nominal transition dynamics $f$ and  $\pi_{N}^{*}$ denoting the robust optimal policy w.r.t.~learned nominal transition dynamics $\hat{f}_N$ via \Cref{alg: mgpbo}, and $\delta\in(0, 1)$, $\epsilon\in (0,\frac{1}{1-\gamma})$,it holds that $\max_{s}|V^{\text{R}}_{\pi_N^{*},f}(s)-V^{\text{R}}_{\pi^{*},f}(s)|\leq \epsilon$ with probability
at least $1 - \delta$ for any $N \geq  N_{\mathrm{TV}}$, where

\begin{equation}
    N_{\mathrm{TV}}= \mathcal{O}\Big(\frac{(2+\rho)^{2}\gamma^{2}}{\rho^{2}(1-\gamma)^{4}}\frac{ \beta_n(\delta)^{2}d^{2}\gamma_{nd}}{\epsilon^{2}}\Big).
\end{equation} 
\end{theorem}
\begin{proof}
\textbf{Step (i):}
As detailed in the proof outline of \Cref{sec:sample_complexity}, in order to bound $V^{\text{R}}_{\hpis,f}(s)-V^{\text{R}}_{\pi^* ,f}(s)$, we begin by adding and subtracting $V^{\text{R}}_{\hpis,\hat{f}_{n}}(s)$ which is the robust value function w.r.t.~the nominal transition dynamics $\hat{f}_{n}$ and its corresponding optimal policy $\hpis$. Then, we split the difference into two terms as follows:
\begin{equation}\label{eq: bound separation tv}
   V^{\text{R}}_{\hpis,f}(s)-V^{\text{R}}_{\pi^*,f}(s)= \underbrace{V^{\text{R}}_{\hpis,f}(s)-V^{\text{R}}_{\hpis,\hat{f}_{n}}(s)}_{(i)}+\underbrace{V^{\text{R}}_{\hpis,\hat{f}_{n}}(s)-V^{\text{R}}_{\pi^{*},f}(s)}_{(ii)}.
\end{equation}
In order to not disturb the flow of the proof we bound (i) and (ii) separately \Cref{lemma: (i) bound lemma} and \Cref{lemma: (ii) bound lemma} respectively. From \Cref{lemma: (i) bound lemma}, we obtain that 
\begin{align}\label{eq: (i) reusage tv}
    (i)&\leq\max_{s}\Big|V^{\text{R}}_{\hpis,f}(s)-V^{\text{R}}_{\hpis,\hat{f}_{n}}(s)\Big|\nonumber\\&\leq \frac{\gamma}{1-\gamma}\max_{s}\Big|\inf_{\mathrm{TV}(p||P_{f}(s,\hpis(s)))\leq \rho}\E_{s'\sim p}\Big[V^{\text{R}}_{\hpis,f}(s')\Big]-\inf_{\mathrm{TV}(p||P_{\hat{f}_{n}}(s,\hpis(s)))\leq \rho}\E_{s'\sim p}\Big[V^{\text{R}}_{\hpis,f}(s')\Big]\Big|.
\end{align}
And from \Cref{lemma: (ii) bound lemma}, we obtain that 
\begin{align}\label{eq: (ii) reusage tv}
    (ii)&\leq\max_{s}\Big|V^{\text{R}}_{\hpis,\hat{f}_n}(s)-V^{\text{R}}_{\pi^{*},f}(s)\Big|\nonumber \\&\leq \frac{\gamma}{1-\gamma}\max_{s}\Big|\inf_{\mathrm{TV}(p||P_{\hat{f}_{n}}(s,\hpis(s)))\leq \rho}\E_{s'\sim p}\Big[V^{\text{R}}_{\pi^{*},f}(s')\Big]-\inf_{\mathrm{TV}(p||P_{f}(s,\hpis(s)))\leq \rho}\E_{s'\sim p}\Big[V^{\text{R}}_{\pi^{*},f}(s')\Big]\Big|.
\end{align}
Note that both these terms in \Cref{eq: (i) reusage tv,eq: (ii) reusage tv} are of the form mentioned in the \textbf{Step (i)} of \Cref{sec:sample_complexity}.

\textbf{Step (ii):} Next, corresponding to \textbf{step (ii)} of the proof outline in \Cref{sec:sample_complexity}, we use \Cref{lemma: var-dis-reform} to bound \Cref{eq: (i) reusage tv,eq: (ii) reusage tv}.
Denote $M:= \frac{1}{1-\gamma}\geq \max_{s}V^{\text{R}}_{\pi}(s)$ for convenience. 
Using \Cref{eq: (i) reusage tv} and \Cref{lemma: diff-opt-tv} (internally using \Cref{lemma: var-dis-reform}), it holds that

\begin{align}
    (i)&\leq\max_{s}\Big|V^{\text{R}}_{\hpis,f}(s)-V^{\text{R}}_{\hpis,\hat{f}_{n}}(s)\Big|\nonumber\\&\leq\frac{1}{1-\gamma}\max_{s}\Big|\gamma\inf_{\mathrm{TV}(p||P_{f}(s,\hpis(s)))\leq \rho}\E_{s'\sim p}\Big[V^{\text{R}}_{\hpis,f}(s')\Big]-\gamma\inf_{\mathrm{TV}(p||P_{\hat{f}_{n}}(s,\hpis(s)))\leq \rho}\E_{s'\sim p}\Big[V^{\text{R}}_{\hpis,f}(s')\Big]\Big| \nonumber \\
    \label{eq: diff-opt lemma usage tv}
    &\leq   \frac{\gamma}{1-\gamma}\max_{s,a}\Bigg(\sup_{\eta\in [0,\frac{(2+\rho)}{\rho(1-\gamma)}]}\Big\{\Big|(\E_{\psa}[(-V^{\text{R}}_{\hpis,\hat{f}_{n}}(s')+\eta)_{+}])-\E_{\pnsa}[(-V^{\text{R}}_{\hpis,\hat{f}_{n}}(s')+\eta)_{+}])\Big|\Big\}\Bigg)\\\label{eq: diff-opt lemma usage overall-tv-1}
    &\leq \frac{\gamma}{1-\gamma}\max_{V(\cdot)\in\mathcal{V}}\max_{s,a}\Bigg(\sup_{\eta\in [0,\frac{(2+\rho)}{\rho(1-\gamma)}]}\Big\{\Big|(\E_{\psa}[(-V(s')+\eta)_{+}])-\E_{\pnsa}[(-V(s')+\eta)_{+}])\Big|\Big\}\Bigg).
\end{align}

We can bound (ii) similarly.
\begin{align}
    (ii)&\leq\max_{s}\Big|V^{\text{R}}_{\hpis,\hat{f}_n}(s)-V^{\text{R}}_{\pi^{*},f}(s)\Big|\\\label{eq: diff-opt lemma usage overall-tv-2}
    &\leq  \frac{\gamma}{1-\gamma}\max_{V(\cdot)\in\mathcal{V}}\max_{s,a}\Bigg(\sup_{\eta\in [0,\frac{(2+\rho)}{\rho(1-\gamma)}]}\Big\{\Big|(\E_{\psa}[(-V(s')+\eta)_{+}])-\E_{\pnsa}[(-V(s')+\eta)_{+}])\Big|\Big\}\Bigg).
\end{align}
\textbf{Step (iii):} Next, we want to utilize the learning error bound (\Cref{eq: learning error bound}) that bounds the difference between the means of true nominal transition dynamics $P_{f}$ and learned nominal transition dynamics $P_{\hat{f}_n}$ to bound \Cref{eq: diff-opt lemma usage overall-tv-1,eq: diff-opt lemma usage overall-tv-2}.

We begin by bounding the difference $\Big|\E_{\psa}[(-V(s')+\eta)_{+}]-\E_{\pnsa}[(-V(s')+\eta)_{+}]\Big|$, by the difference in means of $P_{f}$ and $P_{\hat{f}_n}$ in \Cref{lemma: error bound-tv}. Since \Cref{eq: diff-opt lemma usage overall-tv-1} has a $\max$ over all value functions, we introduce a covering number argument in \Cref{lemma: tv-cover-no-pol} to reform it to a $\max$ over the functions in the $\zeta-$covering set.  We then use \Cref{lemma: error bound-tv}  to obtain bounds in terms of maximum information gain $\Gamma_{Nd}$ (\Cref{eq: max info for gaus}) and $\zeta$. Further details regarding the covering number argument are deferred to \Cref{lemma: tv-cover-no-pol}. Then, we apply the result of \Cref{lemma: tv-cover-no-pol} with $\zeta=1$ (defined in \Cref{lemma: tv-cover-no-pol}) on \Cref{eq: diff-opt lemma usage overall-tv-1}. Then, it holds that
\begin{align}\label{eq: (i) final bound tv}
  (i) \leq \max_{s}\Big|V^{\text{R}}_{\hpis,f}(s)-V^{\text{R}}_{\hpis,\hat{f}_{n}}(s)\Big|&
= \mathcal{O}\Bigg(\Big(\tfrac{(2+\rho)\gamma}{\rho(1-\gamma)^2}\Big)\Big(\tfrac{ \beta_n(\delta)\sqrt{2ed^{2}\gamma_{nd}}}{\sigma\sqrt{n}}\Big)\Bigg).
\end{align}

Note that $\beta_n$, which appears in \Cref{lemma: vakili beta}, has a logarithmic dependence on $n$. Similarly, from \Cref{eq: diff-opt lemma usage overall-tv-2}, and \Cref{lemma: error bound-tv,lemma: tv-cover-no-pol}, we obtain 
\begin{align}\label{eq: (ii) final bound tv}
    (ii)\leq\max_{s}\Big|V^{\text{R}}_{\hpis,\hat{f}_n}(s)-V^{\text{R}}_{\pi^{*},f}(s)\Big|
    &=\mathcal{O}\Bigg(\Big(\tfrac{(2+\rho)\gamma}{\rho(1-\gamma)^2}\Big)\Big(\tfrac{ \beta_n(\delta)\sqrt{2ed^{2}\gamma_{nd}}}{\sigma\sqrt{n}}\Big)\Bigg).  
\end{align}
Note that we want to bound $V^{\text{R}}_{\hpis,f}(s)-V^{\text{R}}_{\pi^{*},f}(s)=(i)+(ii)$ over all $s\in \cS$. Using $ \max_{s}\Big|V^{\text{R}}_{\hpis,f}(s)-V^{\text{R}}_{\pi^{*},f}(s)\Big|\leq  \max_{s}\Big|V^{\text{R}}_{\hpis,\hat{f}_n}(s)-V^{\text{R}}_{\pi^{*},f}(s)\Big|+  \max_{s}\Big|V^{\text{R}}_{\hpis,\hat{f}_n}(s)-V^{\text{R}}_{\pi_n^{*},f}(s)\Big|$ and substituting $M$ by $1/(1-\gamma)$, we obtain from \Cref{eq: (i) final bound tv} and \Cref{eq: (ii) final bound tv}
\begin{align}
    \max_{s}\Big|V^{\text{R}}_{\hpis,f}(s)-V^{\text{R}}_{\pi^{*},f}(s)\Big| 
     &={ \mathcal{O}\Bigg(\Big(\tfrac{(2+\rho)\gamma}{\rho(1-\gamma)^2}\Big)\Big(\tfrac{ \beta_n(\delta)\sqrt{2ed^{2}\gamma_{nd}}}{\sigma\sqrt{n}}\Big)\Bigg)}.  \nonumber
\end{align}

Finally,  to ensure that $\max_{s}|V^{\text{R}}_{\hpis,f}(s)-V^{\text{R}}_{\pi^{*},f}(s)|\leq \epsilon$ , it suffices to have
\begin{align}
     \max_{s}\Big|V^{\text{R}}_{\hpis,f}(s)-V^{\text{R}}_{\pi^{*},f}(s)\Big| 
     =\mathcal{O}\Bigg(\Big(\tfrac{(2+\rho)\gamma}{\rho(1-\gamma)^2}\Big)\Big(\tfrac{ \beta_n(\delta)\sqrt{2ed^{2}\gamma_{nd}}}{\sigma\sqrt{n}}\Big)\Bigg) 
     &= \epsilon\nonumber.
\end{align}
Moving $\sqrt{n}$ and $\epsilon$ to opposite sides and squaring both sides, we obtain
\begin{equation}
     n= \mathcal{O}\Bigg(\Big(\frac{(2+\rho)^2\gamma^2}{\rho^2(1-\gamma)^4}\Big)\Big(\frac{ \beta_n(\delta)^2 2ed^{2}\gamma_{nd}}{\sigma^2\epsilon^2}\Big)\Bigg).\nonumber\\
\end{equation}

\end{proof}

%% file: supplementary/diff-opt-tv.tex
\begin{lemma}(Simplification using \Cref{lemma: var-dis-reform} reformulation)\label{lemma: diff-opt-tv}
Let $V$ be a value function from $ \cS\to [0,1/(1-\gamma)] $. Then, it holds that
\begin{multline}
    \max_{s}|\inf_{\mathrm{TV}(p||P_{\hat{f}_{n}}(s,\hat{\pi}_{n}(s)))\leq \rho}\E_{s'\sim p}\Big[V(s')\Big]-\inf_{\mathrm{TV}(p||P_{f}(s,\hat{\pi}_{n}(s)))\leq \rho}\E_{s'\sim p}\Big[V(s')\Big]|\leq\\ \max_{s,a}\sup_{\eta\in [0,\frac{(2+\rho)}{\rho(1-\gamma)}]}\{|(\E_{\psa}[(-V(s')+\eta)_{+}])-\E_{\pnsa}[(-V(s')+\eta)_{+}])|\}.\nonumber
\end{multline}
\end{lemma}
\begin{proof}
First note that,
\begin{multline}\label{eq: stosa-var}
    \max_{s}\Big|\inf_{\mathrm{TV}(p||P_{\hat{f}_{n}}(s,\hat{\pi}_{n}(s)))\leq \rho}\E_{s'\sim p}\Big[V(s')\Big]-\inf_{\mathrm{TV}(p||P_{f}(s,\hat{\pi}_{n}(s)))\leq \rho}\E_{s'\sim p}\Big[V(s')\Big]\Big|\leq\\ \max_{s,a}\Big|\inf_{\mathrm{TV}(p||P_{\hat{f}_{n}}(s,a))\leq \rho}\E_{s'\sim p}\Big[V(s')\Big]-\inf_{\mathrm{TV}(p||P_{f}(s,a))\leq \rho}\E_{s'\sim p}\Big[V(s')\Big]\Big|
\end{multline}
Using \Cref{lemma: var-dis-reform} and focusing to bound right side of \Cref{eq: stosa-var} for one particular $(s,a)$ state action pair, we obtain   
\begin{align}\label{eq: var-dis reformulation-2}
    &\Big|\inf_{\mathrm{TV}(p||P_{\hat{f}_{n}}(s,a))\leq \rho}\E_{s'\sim p}\Big[V(s')\Big]-\inf_{\mathrm{TV}(p||P_{f}(s,a))\leq \rho}\E_{s'\sim p}\Big[V(s')\Big]\Big|\nonumber\\
    \begin{split}
    &=\Big|\sup_{\eta\in [0,\frac{(2+\rho)}{\rho(1-\gamma)}]}\{-\E_{P_{f}(s,a)}\Big[-V(s')+\eta\Big]_{+}-  \frac{(-ESI_{\psa}(V(s'))+\eta)_{+}}{2}\rho+\eta\}-\\& \sup_{\eta\in [0,\frac{(2+\rho)}{\rho(1-\gamma)}]}\{-\E_{\pnsa}\Big[-V(s')+\eta\Big]_{+}-  \frac{(-ESI_{\pnsa}(V(s'))+\eta)_{+}}{2}\rho+\eta\}\Big| 
    \end{split}\\
    \label{eq: var-dis-reform-bound}
    &\leq \sup_{\eta\in [0,\frac{(2+\rho)}{\rho(1-\gamma)}]}\{|(\E_{\psa}[(-V(s')+\eta)_{+}])-\E_{\pnsa}[(-V(s')+\eta)_{+}])|\}.
\end{align}
Here, \Cref{eq: var-dis-reform-bound} is obtained using $ESI_{\psa}(V(s'))=ESI_{\pnsa}(V(s'))$ as shown in proof of \Cref{lemma: diff-opt} (Case-1).

\end{proof}

%% file: supplementary/error-bound-tv.tex
\begin{lemma}\label{lemma: error bound-tv}(Bound by difference between estimated model $\hat{f}_n$ and true $f$)
Let $V$ be a value function from $ \cS\to [0,1/(1-\gamma)] $. Then, it holds that 
\begin{equation}
    |\E_{\psa}[(-V(s')+\eta)_{+}]-\E_{\pnsa}[(-V(s')+\eta)_{+}]|\leq \Big(\frac{(2+\rho)}{\rho(1-\gamma)}\Big)\sigma^{-1}\|f(s,a)-\hat{f}_{n}(s,a)\|,\\
\end{equation} 
where $\pnsa=\mathcal{N}(\hat{f}_{n}(s,a),\sigma^{2} I)$ and $\psa=\mathcal{N}(f(s,a),\sigma^{2} I)$ and $\eta\in [0,\frac{(2+\rho)}{\rho(1-\gamma)}]$.
\end{lemma}
\begin{proof}
\begin{align*}
    &\Big|\E_{\psa}[(-V(s')+\eta)_{+}]-\E_{\pnsa}[(-V(s')+\eta)_{+}]\Big|\\&=\Big|\int_{\mathbb{R}^{d}}\frac{1}{\sqrt{(2\pi\sigma^{2})^{d}}}(-V(s')+\eta)_{+}(e^{-\frac{\|x-f(s,a)\|^{2}}{2\sigma^{2}}}-e^{-\frac{\|x-\hat{f}_{n}(s,a)\|^{2}}{2\sigma^{2}}})\Big|\\
    &\leq  \int_{\mathbb{R}^{d}}\frac{1}{\sqrt{(2\pi\sigma^{2})^{d}}}(-V(s')+\eta)_{+}\Big|e^{-\frac{\|x-f(s,a)\|^{2}}{2\sigma^{2}}}-e^{-\frac{\|x-\hat{f}_{n}(s,a)\|^{2}}{2\sigma^{2}}}\Big|\\
    &\stackrel{(i)}{\leq}  \frac{(2+\rho)}{\rho(1-\gamma)}\int_{\mathbb{R}^{d}}\frac{1}{\sqrt{(2\pi\sigma^{2})^{d}}}\Big|e^{-\frac{\|x-f(s,a)\|^{2}}{2\sigma^{2}}}-e^{-\frac{\|x-\hat{f}_{n}(s,a)\|^{2}}{2\sigma^{2}}}\Big|\\
     &\stackrel{\textrm{(ii)}}{\leq}2\frac{(2+\rho)}{\rho(1-\gamma)}\cdot  \textrm{TV}(\pnsa,\psa)\\
    &\stackrel{\textrm{(iii)}}{\leq} 2\frac{(2+\rho)}{\rho(1-\gamma)}\sqrt{\textrm{KL}(\pnsa,\psa)/2}\\
    &\stackrel{\textrm{(iv)}}{\leq} 2 \frac{(2+\rho)}{\rho(1-\gamma)}\sqrt{\|f(s,a)-\hat{f}_n(s,a)\|^{2}/4\sigma^{2}}\\
    &\leq \frac{(2+\rho)}{\rho(1-\gamma)}\|f(s,a)-\hat{f}_n(s,a)\|/\sigma,
\end{align*}

where (i) follows from $(-V(s')+\eta)_{+}^{2}\leq \frac{(2+\rho)}{\rho(1-\gamma)}$ as $\eta\leq \frac{(2+\rho)}{\rho(1-\gamma)}$, (ii) follows from the definition of Total Variation (TV) distance between any two multivariate Gaussians, (iii) uses the Pinsker's inequality, and (iv) uses the formula for KL-divergence between multivariate Gaussian distributions.  \\

\end{proof}

%% file: supplementary/cover_tv_proof.tex
\begin{lemma}{($\zeta-$cover construction)}\label{lemma: tv-cover-no-pol}
For $\V$ denoting the set of value functions from $\cS\to[0,1/(1-\gamma)]$, with probability at least $1-\delta$ it holds that
\begin{multline}
    \max_{V\in \V}\max_{s,a}\sup_{\eta\in [0,\frac{(2+\rho)}{\rho(1-\gamma)}]}\{|\E_{\psa}[(-V(s')+\eta)_{+}]-\E_{\pnsa}[(-V(s')+\eta)_{+}]|\} \\
    \leq \mathcal{O}\Bigg(\Big(\tfrac{(2+\rho)}{\rho(1-\gamma)}\Big)\Big(\tfrac{ \beta_n(\delta)\sqrt{2ed^{2}\gamma_{nd}}}{\sigma\sqrt{n}}\Big)\Bigg),
\end{multline}
where $\mathcal{N}_{\V}(\zeta)$ is the $\zeta-$cover for  $\V$.
\end{lemma}
\begin{proof}
Let $\mathcal{N}_{\V}(\zeta)$ be the $\zeta-$ cover of the set $\V$. By definition, there exists $V'\in \mathcal{N}_{\V}(\zeta)$ such that $\|V'-V\|\leq \zeta$ for every $V\in \V$.
\begin{align}
\begin{split}
&|\E_{\psa}[(-V(s')+\eta)_{+}]-\E_{\pnsa}[(-V(s')+\eta)_{+}]|\\&\leq
|\E_{\psa}[(-V(s')+\eta)_{+}]-\E_{\psa}[(-V'(s')+\eta)_{+}]|\\&\quad+|\E_{\psa}[(-V'(s')+\eta)_{+}]-\E_{\pnsa}[(-V'(s')+\eta)_{+}]|\\&\quad+|\E_{\pnsa}[(-V'(s')+\eta)_{+}]-\E_{\pnsa}[(-V(s')+\eta)_{+}]|.
\end{split}\\\label{eq: tv-close-cover}
&\stackrel{(i)}{\leq} 2 \|V'-V\|+|\E_{\psa}[(-V'(s')+\eta)_{+}]-\E_{\pnsa}[(-V'(s')+\eta)_{+}]|,
\end{align}
where (i) follows from \Cref{lemma: function-lipschitz-tv}. Using \Cref{eq: tv-close-cover}, we bound uniformly over all $V\in\V$. Using \Cref{eq: tv-close-cover} we bound uniformly over all $V\in\V$,
\begin{align}
     &\max_{V\in \V}\max_{s,a}\sup_{\eta\in [0,\frac{(2+\rho)}{\rho(1-\gamma)}]}\{|\E_{\psa}[(-V(s')+\eta)_{+}]-\E_{\pnsa}[(-V(s')+\eta)_{+}]|\}\\
     \begin{split}
          &\leq  \max_{V'\in \mathcal{N}_{\V}(\zeta)}\max_{s,a}\sup_{\eta\in [0,\frac{(2+\rho)}{\rho(1-\gamma)}]}\Biggl\{\Bigl|2 \|V'-V\|+|\E_{\psa}[(-V'(s')+\eta)_{+}]-\E_{\pnsa}[(-V'(s')+\eta)_{+}]|\Bigr|\Biggr\}\nonumber
     \end{split}\\
     \begin{split}
         &\stackrel{(ii)}{\leq}  \max_{V'\in \mathcal{N}_{\V}(\zeta)}\max_{s,a}\sup_{\eta\in [0,\frac{(2+\rho)}{\rho(1-\gamma)}]}\Biggl\{\Big|\E_{\psa}[(-V'(s')+\eta)_{+}]-\E_{\pnsa}[(-V'(s')+\eta)_{+}]\Big|\Biggr\}+2\zeta\nonumber
     \end{split} \\
     \begin{split}
         &\stackrel{(iii)}{\leq}  \max_{V'\in \mathcal{N}_{\V}(\zeta)}\max_{s,a}\sup_{\eta\in [0,\frac{(2+\rho)}{\rho(1-\gamma)}]}\Bigg\{ \Big(\tfrac{(2+\rho)}{\rho(1-\gamma)}\Big)\sigma^{-1}\|f(s,a)-\hat{f}_{n}(s,a)\|\Bigg\}+2\zeta\nonumber
     \end{split}\\
      &\stackrel{(iv)}{\leq} \mathcal{O}\Bigg(\Big(\tfrac{(2+\rho)}{\rho(1-\gamma)}\Big)\Big(\tfrac{ \beta_n(\delta)\sqrt{2ed^{2}\gamma_{nd}}}{\sigma\sqrt{n}}\Big)\Bigg)+2\zeta\\
      \label{eq: learn-error-usage-tv-cover}
      &\stackrel{(v)}{\leq}  \mathcal{O}\Bigg(\Big(\tfrac{(2+\rho)}{\rho(1-\gamma)}\Big)\Big(\tfrac{ \beta_n(\delta)\sqrt{2ed^{2}\gamma_{nd}}}{\sigma\sqrt{n}}\Big)\Bigg),
\end{align}
where (ii) follows from $\|V'-V\|\leq \zeta$ , (iii) follows from \Cref{lemma: error bound-tv}, (iv) follows from \Cref{eq: learning error bound}, and (v) follows from substituing $\zeta=1$ (or any constant).

\end{proof}

%% file: supplementary/function_lipschitz_tv.tex
\begin{lemma}\label{lemma: function-lipschitz-tv}
For any two value functions $V,V':\cS\to[0,\frac{1}{1-\gamma}]$, it holds that
   \begin{equation}
       |\E_{\psa}[(-V'(s')+\eta)_{+}]-\E_{\psa}[(-V(s')+\eta)_{+}]|\leq \|V'-V\|.
   \end{equation} 
\end{lemma}
\begin{proof}
Noting that both the distributions are w.r.t. the same distribution $\psa$ we have,
\begin{align}
    &\E_{\psa}[(-V'(s')+\eta)_{+}]-\E_{\psa}[(-V(s')+\eta)_{+}]\nonumber\\\label{eq: tv cover diff bound} &\leq \int\limits_{s'\sim\psa}\Big(\mathbbm{1}(V'(s')<\eta)(-V'(s')+\eta)-\mathbbm{1}(V(s')<\eta)(-V(s')+\eta)\Big)p_{\psa}(s')ds'.
\end{align}

Adding and subtracting $\mathbbm{1}(V(s')<\eta)(-V'(s')+\eta)$ to \Cref{eq: tv cover diff bound}, we obtain 2 terms,
\begin{align}
    i&=\int_{s'\sim\psa}\Big(\mathbbm{1}(V'(s')<\eta)-\mathbbm{1}(V(s')<\eta)\Big)(-V'(s')+\eta)p_{\psa}(s')ds'\\
    ii&=\int_{s'\sim\psa}\mathbbm{1}(V(s')<\eta)\Big((-V'(s')+\eta)-(-V(s')+\eta)\Big)p_{\psa}(s')ds'.
\end{align}
Bounding i first,
\begin{align}
    i=&\int_{s'\sim\psa}\Big(\mathbbm{1}(V'(s')<\eta)-\mathbbm{1}(V(s')<\eta)\Big)(-V'(s')+\eta)p_{\psa}(s')ds'\\
\begin{split}
=&\int_{s'\sim\psa}\Big(\mathbbm{1}(V'(s')<\eta\leq V(s'))\Big)(-V'(s')+\eta)p_{\psa}(s')ds'\\&-\int_{s'\sim\psa}\Big(\mathbbm{1}(V(s')<\eta< V'(s'))\Big)(-V'(s')+\eta)p_{\psa}(s')ds'
\end{split}\\
\begin{split}
\leq&\int_{s'\sim\psa}\Big(\mathbbm{1}(V'(s')<\eta\leq V(s'))\Big)(-V'(s')+V(s'))p_{\psa}(s')ds'\\&-\int_{s'\sim\psa}\Big(\mathbbm{1}(V(s')<\eta< V'(s'))\Big)(-V'(s')+V(s'))p_{\psa}(s')ds'
\end{split}\\
\begin{split}
\leq&\int_{s'\sim\psa}\Big(\mathbbm{1}(V'(s')<\eta\leq V(s'))\Big)(-V'(s')+V(s'))p_{\psa}(s')ds'\\&+\int_{s'\sim\psa}\Big(\mathbbm{1}(V(s')<\eta< V'(s'))\Big)(V'(s')-V(s'))p_{\psa}(s')ds'
\end{split}\\
\label{eq: vd-lip-func-bound-1}
&\leq \|V'-V\|.
\end{align}
Similarly bounding ii,
\begin{align}
     ii&=\int_{s'\sim\psa}\mathbbm{1}(V(s')<\eta)\Big((-V'(s')+\eta)-(-V(s')+\eta)\Big)p_{\psa}(s')ds'\\
    &=\int_{s'\sim\psa}\limits\mathbbm{1}(V(s')<\eta)\Big(-V'(s')+V(s')\Big)p_{\psa}(s')ds'\\
    &\leq\int_{s'\sim\psa}\limits\mathbbm{1}(V(s')<\eta)\Big|-V'(s')+V(s')\Big|p_{\psa}(s')ds'\\
    \label{eq: vd-lip-func-bound-2}
    &\leq \|V'-V\|.
\end{align}
Using \Cref{eq: vd-lip-func-bound-1,eq: vd-lip-func-bound-2} we get the desired result.
\end{proof}

%% file: supplementary/20_additional_experiments.tex
\section{Additional Experiments and Details}
\label{app:additional_exps}

In this section, we report additional experiments and discuss further details of our experimental setup. All experiments were run with GPU clusters: 10xNVidia 32Gb Tesla V100 with Intel(R) processors (2 cores, 2.50 GHz) and 256Gb RAM.

For all the experiments, we use the environment implementations of \citet{mehta2021experimental} as done in \href{https://github.com/fusion-ml/trajectory-information-rl/tree/main}{https://github.com/fusion-ml/trajectory-information-rl/tree/main}. Also, to learn the environment transition model, we use the same corresponding GP hyperparameters proposed by~\citet{mehta2021experimental}.
For the offline \textsc{RFQI}/\textsc{FQI} algorithms we follow the implementation of  \citet{panaganti2022robust,chen2019information} in \href{https://github.com/zaiyan-x/RFQI}{https://github.com/zaiyan-x/RFQI}. We use the same default hyperparameters as used in their code except for training steps, batch size and robustness radius $\rho$ (for \textsc{RFQI}) which we tune depending on the environment as outlined next. For \textsc{SAC} in Pendulum experiments, we use the implementation and hyperparameters of \href{https://github.com/DLR-RM/rl-baselines3-zoo}{https://github.com/DLR-RM/rl-baselines3-zoo}. Whereas, for \textsc{SAC} in Reacher experiments, we use the implementation and hyperparameters of \href{https://github.com/fusion-ml/bac-baselines}{https://github.com/fusion-ml/bac-baselines}, \href{https://github.com/IanChar/rlkit2}{https://github.com/IanChar/rlkit2} (as done in \citep{mehta2021experimental}).

\textbf{Pendulum:} In Pendulum experiments, we construct the learned model using 60 samples from the true environment. Then, we train a \textsc{SAC} policy on such a model for $2*10^4$ steps and use it (with the probability of choosing a random action being $0.3$ or $0.5$) to generate $10^6$ offline data (these are used both for \textsc{MVR+RFQI} and \textsc{MVR+FQI}). For training steps and batch size we consider the following combinations:
$\{'2000-100','5000-100','10000-100','20000-100','35000-100','50000-100','5000-500','5000-1000'\}$. We combine all these combinations with the following values of $\rho$ -- $\{0.1,0.2,0.3,0.5,0.6,0.7,0.8,0.9\}$. For each algorithm, we pick the best-performing combination in terms of average reward over 20 episodes for all (or most) perturbation values. We do this separately for length perturbations and action perturbations. In the length perturbation, the pendulum's length is changed from its nominal value to a new value depending on the perturbation percentage. In the action perturbation, a random action is chosen instead of the action chosen by the policy with various probabilities ranging from $[0,1]$. We detail the optimal hyperparameters we realized for each algorithm in \Cref{tab:pendulum_hyperparams} for the length and action perturbation, respectively. Moreover, we plot the average performance (over 20 episodes) of the different baselines w.r.t. length and action perturbations in \Cref{fig:pendulum_supp}. We notice that in the case of length perturbation, the robust algorithms (RFQI and MVR+RFQI) outperform the corresponding non-robust baselines. In the case of action perturbations, we observe all algorithms except for SAC achieve similar performance.

\begin{table}[h]
    \centering
    \small
    \begin{tabular}{r c c c c c c}
    \toprule
         &  \textsc{Training Steps}  & \textsc{Batch-Size} & \textsc{$\rho$} & \textsc{Random Action Probability (Dataset)}  \\
         \midrule
        MVR+RFQI & $5000$ & $100$  & $0.3$ & $0.5$  \\
         MVR+FQI & $2000$ & $100$ & - & $0.5$  \\
        RFQI & $2000$ & $100$ &   $0.9$  & $0.5$ \\
        FQI & $5000$ & $500$ &   -  & $0.5$ \\
        \bottomrule\\
    \end{tabular}
     \begin{tabular}{r c c c c c c}
    \toprule
         &  \textsc{Training Steps}  & \textsc{Batch-Size} & \textsc{$\rho$} & \textsc{Random Action Probability (Dataset)}  \\
         \midrule
        MVR+RFQI & $20000$ & $100$  & $0.5$ & $0.3$  \\
         MVR+FQI & $50000$ & $100$ & - & $0.3$  \\
        RFQI & $50000$ & $100$ &   $0.1$  & $0.5$ \\
        FQI & $5000$ & $500$ &   -  & $0.5$ \\
        \bottomrule\\
    \end{tabular}
    \vspace{1mm}
    \caption{Hyperparameters for Pendulum - length perturbation (top) and action perturbation (bottom).}
    \label{tab:pendulum_hyperparams}
\end{table}

\begin{figure}[h]
    \centering
    \includegraphics[width= 0.4\textwidth]{pendulum_sac_plots/paper_plots_pendulum_0pt3_length_neg.png}
    \hspace{2em}
    \includegraphics[width= 0.4\textwidth]{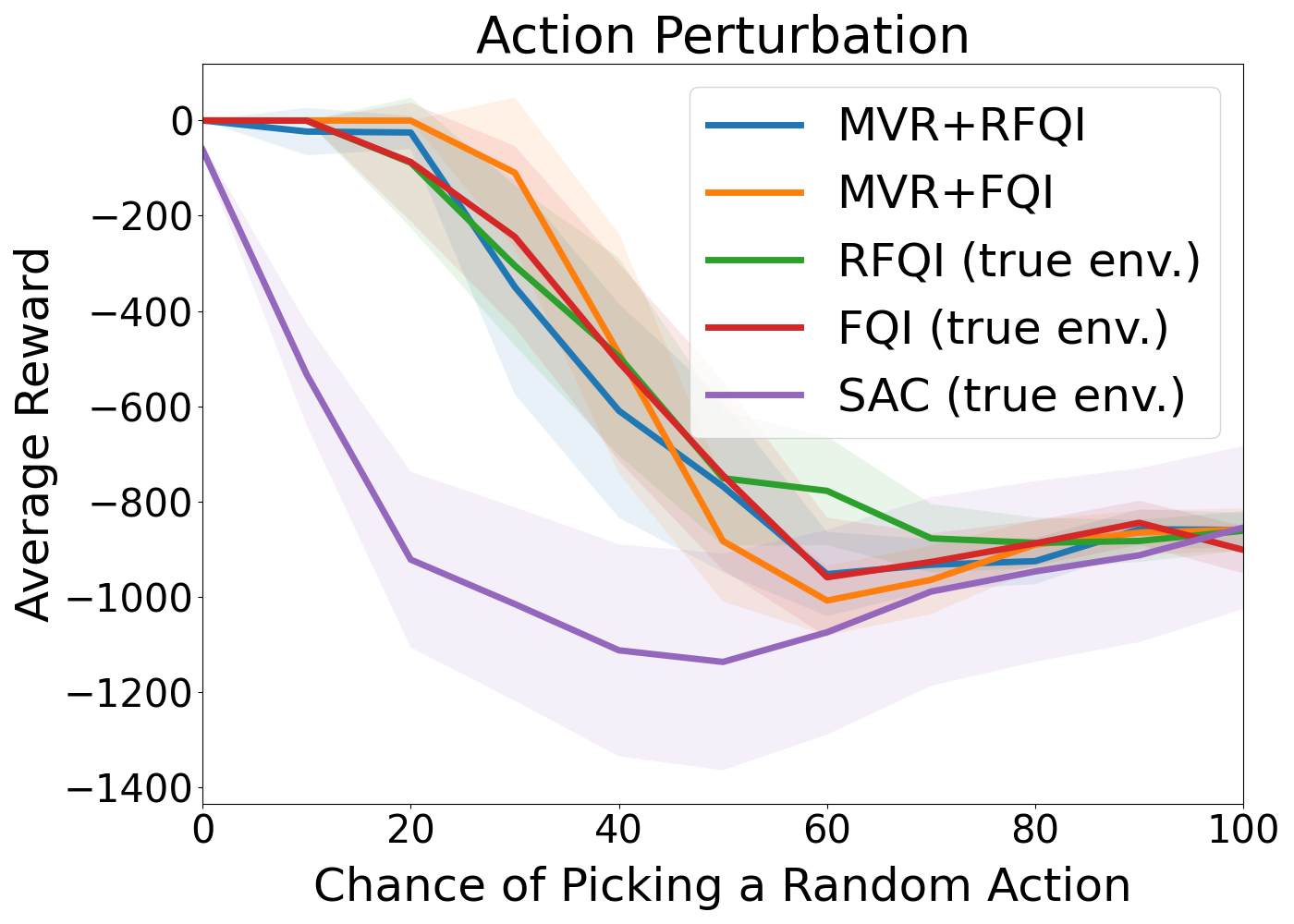}
    \caption{Pendulum experiments.}
    \label{fig:pendulum_supp}
\end{figure}

\newpage
\textbf{Cartpole:} In Cartpole experiments, we construct the learned model using 150 samples from the true environment. Then, we run MPC on such a model following the implementation and hyperparameters of \citep{mehta2021experimental,pinneri2020sample} requiring 2250 samples to calculate the optimal action at each step and use it (with the probability of choosing a random action being $0.3$) to generate $10^6$ offline data for \textsc{MVR+RFQI} and \textsc{MVR+FQI}. For training steps and batch size, we test the following combinations:
$\{'2000-100','5000-100','10000-100','20000-100','35000-100','50000-100','5000-500','5000-1000'\}$, and consider radii $\rho$ in $\{0.1,0.2,0.3,0.5,0.6,0.7,0.8,0.9\}$. We consider perturbations of the force magnitude and the gravity, whereby the actuation force/gravity is changed from its nominal value to a new value depending on the perturbation percentage. We report the best-performing (average over 20 episodes) hyperparameters for each algorithm in \Cref{tab:cartpole_hyperparams}. Such parameters were observed to be a good choice for both perturbation types. Finally, we plot the average performance (over 20 episodes) of the different baselines w.r.t. force magnitude and gravity perturbations in \Cref{fig:cartpole_supp}. We notice that in both  perturbations, the robust algorithms (RFQI and MVR+RFQI) outperform the corresponding non-robust baselines.

\begin{table}[h]
    \centering
        \small
    \begin{tabular}{r c c c c c c}
    \toprule
         &  \textsc{Training Steps}  & \textsc{Batch-Size} & \textsc{$\rho$} & \textsc{Random Action Probability (Dataset)}  \\
         \midrule
        MVR+RFQI & $5000$ & $500$  & $0.5$ & $0.3$  \\
         MVR+FQI & $50000$ & $100$ & - & $0.3$  \\
        RFQI & $5000$ & $100$ &   $0.3$  & $0.3$ \\
        FQI & $10000$ & $100$ &   -  & $0.3$ \\
        \bottomrule\\
    \end{tabular}
    \vspace{1mm}
    \caption{Hyperparameters for Cartpole.}
    \label{tab:cartpole_hyperparams}
\end{table}

\begin{figure}[h]
    \centering
    \includegraphics[width= 0.4\textwidth]{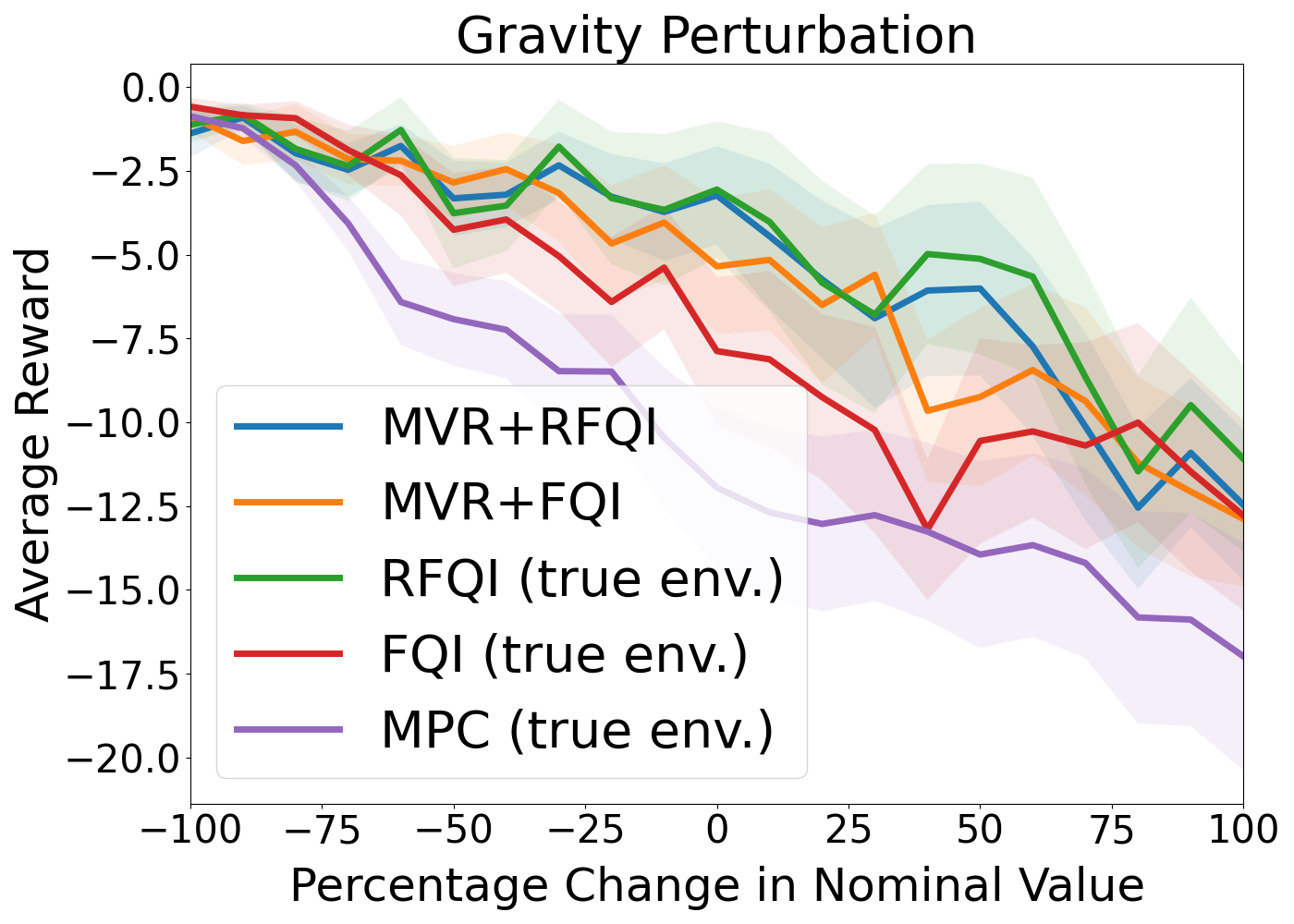}
    \hspace{2em}
    \includegraphics[width= 0.4\textwidth]{cartpole_mpc_plots/paper_plots_pilcocartpole_0pt3_force_mag.png}
    \caption{Cartpole experiments.}
    \label{fig:cartpole_supp}
\end{figure}

\textbf{Reacher:} In Reacher experiments, we construct the learned model using 2000 samples from the true environment. Then, we train a SAC policy on such a model for $10^6$ steps and use it (with the probability of choosing a random action being $0.3$) to generate $10^6$ offline data for \textsc{MVR+RFQI} and \textsc{MVR+FQI}. For training steps and batch size, we consider the following combinations: $\{'10000-500','20000-500','40000-500','80000-500','160000-1000'\}$, while we consider radii $\rho$ in $\{0.1,0.3,0.5,0.7,0.9\}$. We consider perturbations of the joint stiffness subject to different equilibrium positions, the latter represented by the 'Springref' parameter which we take to be $50$ or $100$. In both perturbation types, the joint stiffness is changed from its nominal value of $0$ to a new value depending on the perturbation magnitude. Best-performing hyperparameters' configurations are reported in \Cref{tab:reacher_hyperparams}. We plot the average performance (over 20 episodes) of the different baselines in \Cref{fig:reacher_supp}. Similar to the other environments, we observe the robust algorithms (RFQI and MVR+RFQI) outperform the corresponding non-robust baselines.

\begin{table}[h]
    \centering
        \small
    \begin{tabular}{r c c c c c c}
    \toprule
         &  \textsc{Training Steps}  & \textsc{Batch-Size} & \textsc{$\rho$} & \textsc{Random Action Probability (Dataset)}  \\
         \midrule
        MVR+RFQI & $10000$ & $500$  & $0.5$ & $0.3$  \\
         MVR+FQI & $20000$ & $500$ & - & $0.3$  \\
        RFQI & $40000$ & $500$ &   $0.1$  & $0.3$ \\
        FQI & $20000$ & $500$ &   -  & $0.3$ \\
        \bottomrule\\
    \end{tabular}
    \vspace{1mm}
    \caption{Hyperparameters for Reacher.}
    \label{tab:reacher_hyperparams}
\end{table}

\begin{figure}[h]
    \centering
    \includegraphics[width= 0.4\textwidth]{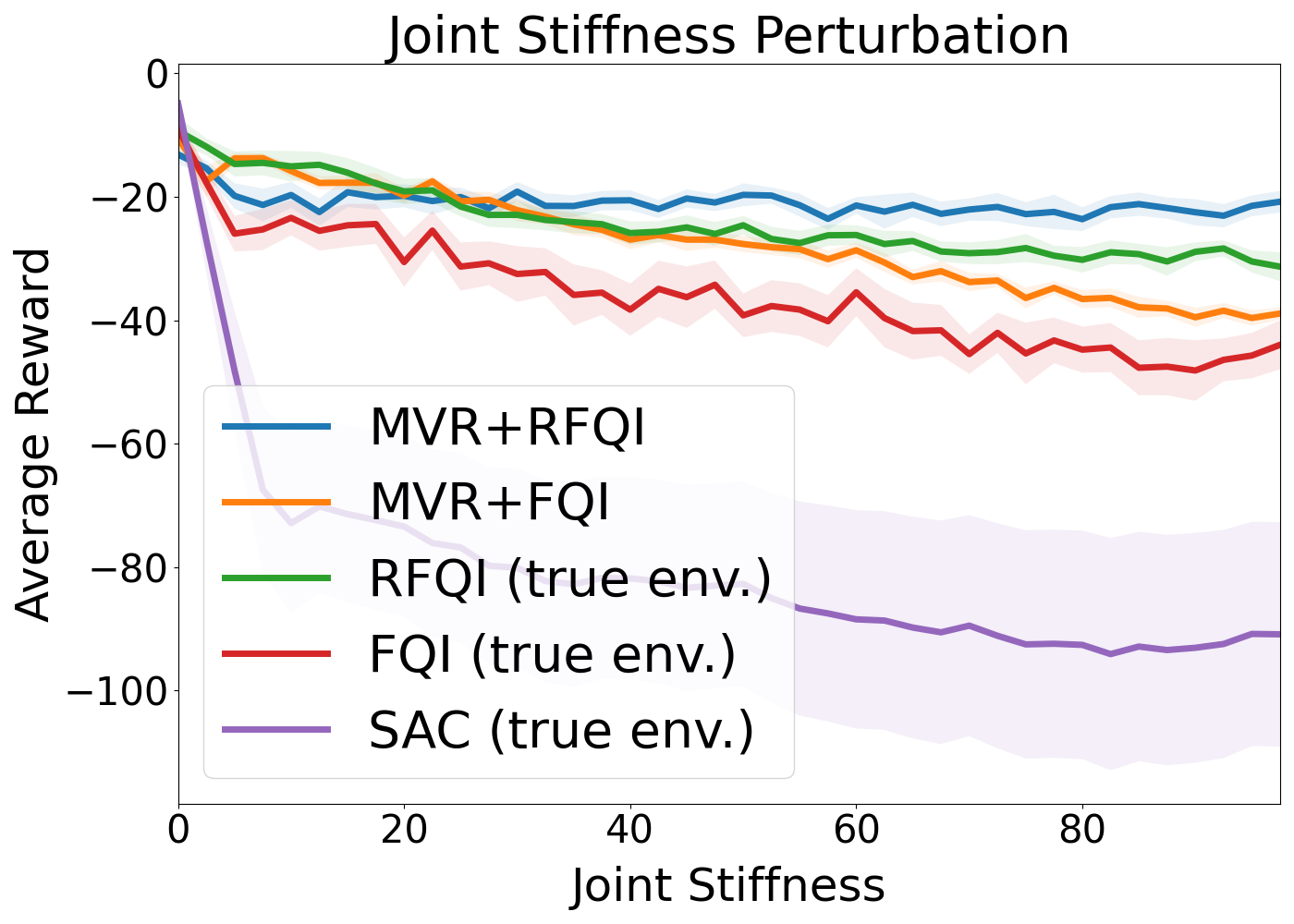}
    \hspace{2em}
    \includegraphics[width= 0.4\textwidth]{reacher_sac_plots/paper_plots_reacher_0pt3_springref100_joint_stiffness_x.png}
    \caption{Reacher experiments with 'Springref' parameter set to 50 (left) or 100 (right).}
    \label{fig:reacher_supp}
\end{figure}

\newpage

\newpage